\documentclass[twocolumn]{svjour3}          %
\smartqed  %
\usepackage{graphicx}
\usepackage{times}
\usepackage{amsmath}
\usepackage{amsfonts}
\usepackage{amssymb}
\usepackage{subfigure}
\usepackage{multirow}
\usepackage{color}
\usepackage{tabu}
\usepackage[authoryear,sort&compress,]{natbib}
\usepackage[colorlinks=true,urlcolor=blue,citecolor=blue,linkcolor=blue,bookmarks=true]{hyperref}
\usepackage[title]{appendix}
\usepackage[english]{babel}
\usepackage{bm}
\usepackage{colortbl}
\usepackage{enumitem}
\usepackage{array}
\usepackage[ruled,vlined]{algorithm2e}
\usepackage{threeparttable}

\newcommand{\blue}[1]{\color{blue}{#1}}

\newcommand{\sz}{\fontsize{9pt}{\lineskip}\selectfont}

\newcommand*{\QEDB}{\hfill\ensuremath{\square}}  %

 \journalname{    }
\begin{document}

\title{Adaptive Importance Learning for Improving Lightweight Image Super-resolution Network}

\titlerunning{Adaptive Importance Learning for Improving Lightweight Image Super-resolution Network}        %
\author{
Lei Zhang$^{\textbf{1},\textbf{2}}$\and
Peng Wang$^\textbf{1}$\and
Chunhua Shen$^{\textbf{1}*}$ \and
Lingqiao Liu$^\textbf{1}$\and
Wei Wei$^\textbf{2}$\and
Yanning Zhang$^{\textbf{2}}$\and
Anton van den Hengel$^{\textbf{1}}$       
}

\authorrunning{Zhang et al.} %

\institute{
$^*$ Corresponding author.\\
$^{\textbf{1}}$ School of Computer Science, The University of Adelaide, Australia\\
$^{\textbf{2}}$ School of Computer Science, Northwestern Polytechnical University, Xi’an, China\\
}

\date{Received: date / Accepted: date}

\maketitle

\begin{abstract}
Deep neural networks have achieved remarkable success in single image super-resolution (SISR). The computing and memory requirements of these methods have hindered their application to broad classes of real devices with limited computing power, however. One approach to this problem has been lightweight network architectures that balance the super-resolution performance and the computation burden. In this study, we revisit this problem from an orthogonal view, and propose a novel learning strategy to maximize the pixel-wise fitting capacity of a given lightweight network architecture. Considering that the initial capacity of the lightweight network is very limited, we present an adaptive importance learning scheme for SISR that trains the network with an easy-to-complex paradigm by dynamically updating the importance of image pixels on the basis of the training loss. Specifically, we formulate the network training and the importance learning into a joint optimization problem. With a carefully designed importance penalty function, the importance of individual pixels can be gradually increased through solving a convex optimization problem. The training process thus begins with pixels that are easy to reconstruct, and gradually proceeds to more complex pixels as fitting improves. Furthermore, the proposed learning scheme is able to seamlessly assimilate knowledge from a more powerful teacher network in the form of importance initialization, thus obtaining better initial capacity in the network. Through learning the network parameters, and updating pixel importance, the proposed learning scheme enables smaller, lightweight, networks to achieve better performance than has previously been possible. Extensive experiments on four benchmark datasets demonstrate the potential benefits of the proposed learning strategy in lightweight SISR network enhancement. In some cases, our learned network with only $25\%$ of the parameters and computational complexity can produce comparable or even better results than the corresponding full-parameter network.

\end{abstract}

\section{Introduction}
There are a wide variety of applications where the ability to increase the resolution of an image adds to the user experience, from from surveillance and public security~\cite{zhang2017beyond}, business and entertainment~\cite{liu2017compositional} to remote sensing~\cite{wei2017structured}. Single-image super resolution (SISR), the process of increasing the resolution of an image without additional information, has received significant attention~\citep{yang2014single,huang2015single,kim2016accurate} as a result.

Most early SISR methods focus on exploiting pixel statistics~\citep{kim2010single,efrat2013accurate} or the internal patch recurrence~\citep{huang2015single,glasner2009super} of HR images as priors. These methods typically do not generalise well, because even a small divergence between the properties of the real low-resolution image and the prior embodied in the heuristic causes visible artifacts in the reconstructed HR image. Recently, deep convolution neural network (DCNN) based learning methods~\citep{wang2015deep,kim2016accurate,kim2016deeply,ledig2017photo,tai2017image}, have shown remarkable success in SISR, especially on some specific scaling factors (e.g., 2-4). Nevertheless, due to their very deep structures, these methods often exhibit significant memory and computing requirements, which necessitates powerful computational units (e.g., GPUs) thus limiting their application to the many real devices with limited computing power (and particularly hand-held devices including phones).

To address this problem, some efforts~\citep{dong2016accelerating,shi2016real} dedicate to customize specific lightweight network architectures. In this study, we revisit this problem in an orthogonal view and propose to develop an novel learning strategy to maximize the pixel-wise fitting capacity of a given lightweight architecture. To this end, we revisit the traditional training procedure for a SISR network, which seeks the optimal network parameters to minimize the average loss over all pixels in training images. Moreover, pixels of diffident reconstruction difficulty are mixed together to fed into the network for training. However, by doing this, complex pixels that are difficult to reconstruct will mislead the training procedure, which renders the network even failing to handle pixels that are easy to reconstruct, since the initial capacity of the lightweight network is very limited and vulnerable. This is similar to the cognitive process of human which is prone to be confused when starts with a compound of complex and easy tasks and considers them equally. For example, when receiving a compound of easy and hard words one time, a pupil may fail to remember those easy ones that should be well mastered. Alternatively, if he starts with some easy words and gradually attempts to remember more and more hard ones when these easy words have been well mastered, more words will be remembered. Therefore, the basic pattern of human cognitive process is to learn from easy to complex and gradually enhance the capacity of human. Recently, it has been empirically demonstrated that learning as such a paradigm can avoid bad local minima and generalize better~\citep{khan2011humans,basu2013teaching}. Therefore, it is promising to enhance the capacity of the lightweight SISR network with an appropriate easy-to-complex learning paradigm.

Inspired by this, we present an adaptive importance learning scheme for SISR, which assigns importance (i.e., the probability of participating training and zero importance denotes removing the pixel during training) to each image pixel and dynamically updates the importance to control the network training following an easy-to-complex paradigm. To this end, we formulate the network training as well as the pixel-wise importance learning into a bi-convex optimization problem. With introducing a carefully designed importance penalty function, the importance of image pixels can be adaptively updated by solving a convex optimization problem. As a result, the importance is gradually increased according to the network reconstruction error on these pixels. By doing this, the network will start with pixels that are easy to reconstruct for training, and gradually be exposed to more and more complex pixels when its fitting capacity is enhanced. Furthermore, with the proposed importance learning scheme, the network can seamlessly assimilate the knowledge from a more powerful teacher network in the form of pixel importance initialization, which enables the network to generalize better. Through learning the network parameters and updating the pixel importance in an alternative way until convergence, the proposed learning scheme can obviously enhance the network capacity. With extensive experiments on four benchmark datasets and two seminal DCNN architectures for SISR, we demonstrate that the proposed adaptive importance learning scheme is able to enhance the performance of different scales of lightweight networks obviously. Moreover, due to not designing specific lightweight network architecture, it can be conveniently applied to any lightweight SISR networks for enhancement.

In summary, this study mainly contributes in the following four aspects.
\begin{itemize}[noitemsep]
\item We propose to develop an easy-to-complex learning para\-digm to maximize the fitting capacity of a given lightwei\-ght network architecture for SISR. To the best of our knowledge, this is the first attempt to do this in SISR.
\item We present an adaptive importance learning scheme to train the lightweight SISR network for enhancement.
\item We propose to distil knowledge from a more powerful teacher network for better importance initialization.
\item We demonstrate the pleasing potential of the proposed learning scheme in extensive experiments.
\end{itemize}

\section{Related work}
In this section, we briefly review the following three aspects of works related to this study.

{\textbf{Single image super-resolution}}. In early stage, SISR are addressed by exploiting the statistical characteristics of HR image as priors. For example, Sun et al. in~\citep{sun2008image} learn a gradient profile prior from extensive natural images and then apply it for SISR. In~\citep{kim2010single}, Kim et al. employ a modification of the natural image prior to refine the detailed structure along edges. Different from these methods, Glasner et al.~\citep{glasner2009super} propose to exploit the internal patch recurrence for super-resolution. Huang et al.~\citep{huang2015single} further introduce the geometric variation in searching recurrent patches. Recently, inspired by the success of deep neural networks, especially DCNN, some literatures commence at learning more powerful SISR models with DCNN from extensive LR-HR paris. For example, Dong et al.~\citep{dong2016image} construct a 3-layer DCNN for SISR which outperforms most of previous non-learning methods. With introducing residual learning, Kim et al.~\citep{kim2016accurate} develop a much deeper (e.g., 20 layers) DCNN based SISR model. Tai et al.~\citep{tai2017image} further introduce a recursive block into the global residual structure and gains the state-of-the-art performance. In~\citep{ledig2017photo}, Ledig et al. present a generative adversarial network to obtain photo-realistic HR images. Although those deep models achieve satisfactory SISR results, most of them are computational expensive to deploy on real devices. Currently, a few literatures have commenced at handling this problem by developing lightweight network architecture. For example, a compact hourglass-shape DCNN structure and a subpixel convolution structure are designed in~\citep{dong2016accelerating,shi2016real}, respectively. In this study, we solve this problem in an orthogonal view and propose to maximize the capacity of a given lightweight network with a new learning strategy. In addition, due to not involving network architecture, the proposed scheme can be directly integrated into any lightweight SISR networks for enhancement.

{\textbf{Knowledge distillation}}. This line of research aims at distilling knowledge from a complicated (or an ensemble of models) teacher model into a compact (or single) alternative without performance drop. Hinton et al.~\citep{hinton2015distilling} propose to distil knowledge by matching the soften output (e.g., logits) of teacher models. Romero et al.~\citep{romero2014fitnets} further match the intermediate features (e.g., hints) of teacher models. Zhang et al.~\citep{zhang2017deep} integrate the knowledge distillation into a mutual learning framework. Different from matching the output of teacher models, we propose to learn the pixel-wise importance of each example to training loss from a teacher model. 
 
{\textbf{Curriculum and self-paced learning}}. Similar as this study, these two paradigms learn a model gradually including from easy to complex examples in training phase. In curriculum learning~\citep{bengio2009curriculum}, the curriculum (i.e., learning sequence) is often derived by predetermined heuristics. For example, in~\citep{bengio2009curriculum}, the curriculum is derived based on the variability in shape to enable shapes with less variability being learned earlier. In~\citep{khan2011humans}, the common sense of participants are employed to determine the learning sequence of graspability to object. In self-paced learning, the curriculum design is often integrated into the learning objective as a regularization. For example, Jiang et al.~\citep{jiang2014easy} jointly optimize the learning objective as well as a binary weight vector which controls the learning pace. In contrast, the proposed adaptive importance learning scheme learns a pixel-wise curriculum based on the reconstruction error of the network and aims at enhancing the capacity of a given lightweight SISR network. Moreover, it enables the network to seamlessly assimilate the knowledge from a more powerful teacher network.

\section{The proposed learning paradigm}\label{sec:proposed}
In general, with $n$ LR-HR image pairs $\{\mathbf{x}_i, \mathbf{y}_i\}^n_{i=1}$, we can learn a lightweight network $\mathcal{S}(\cdot,\mathbf{\theta})$ as follows
\begin{equation}\label{eq:eq1}
\begin{aligned}
\min\limits_{\mathbf{\theta}} \mathbb{E}(\mathbf{\theta})=\frac{1}{n}\sum\limits^n_{i=1}l\left(\mathbf{y}_i, \mathcal{S}(\mathbf{x}_i, \mathbf{\theta})\right)
\end{aligned}
\end{equation}
where $\mathbf{\theta}$ denotes the network parameters and $l$ indicates the loss function (e.g., MSE loss or $\ell_1$ loss). In the training phase, the optimal $\mathbf{\theta}$ seeks to minimize the expectation $\mathbb{E}(\mathbf{\theta})$ where all pixel with different reconstruction difficulties are fed together into $\mathcal{S}$ for training. To maximize the pixel-wise fitting capacity of $\mathcal{S}$, we propose to train $\mathcal{S}$ with an adaptive importance learning scheme as
\begin{equation}\label{eq:eq2}
\begin{aligned}
&\min\limits_{\mathbf{\theta}, W} \mathbb{E}(\mathbf{\theta}, W)=\frac{1}{n}\sum\limits^n_{i=1}\left[l\left(\mathbf{y}_i \odot \mathbf{w}_i, \mathcal{S}(\mathbf{x}_i, \mathbf{\theta}) \odot \mathbf{w}_i\right) + h(\mathbf{w}_i)\right],\\
&{\rm{s.t.}}{\kern 4pt}\forall i, 0 \preceq \mathbf{w}_i \preceq 1,\\
\end{aligned}
\end{equation}
where $\mathbf{w}_i$ indicates the pixel-wise importance vector for each training pair and $W = \{\mathbf{w}_i\}^n_{i=1}$ collects all importance vectors. Since $0 \preceq \mathbf{w}_i \preceq 1$, the pixel-wise importance can be viewed as the probability of each pixel participating the training procedure as Eq.~\eqref{eq:eq2}, e.g., when the importance is zero, the corresponding pixel will removed from training the network. $\odot$ denotes point-wise multiplication. $h(\mathbf{w}_i)$ represents a penalty function over $\mathbf{w}_i$, which controls the importance learning strategy as well as avoiding trivial solutions of $\mathbf{w}_i$ (e.g., $\mathbf{w}_i=\mathbf{0}$).

In the adaptive importance learning scheme, the network parameter $\mathbf{\theta}$ and the importance $W$ are jointly optimized. To solve this problem, we can adopt the alternative minimization scheme~\citep{zhang2018cluster}, which reduces this problem into a $\mathbf{\theta}$-subproblem and a $W$-subproblem, and then alternatively optimizes each subproblem until convergence. Different from the traditional learning scheme in Eq.~\eqref{eq:eq1} which only trains the network once, the proposed learning scheme will train the network in several rounds. More importantly, with an appropriate $h(\mathbf{w}_i)$, the importance of image pixels can be assigned to any value expected, with which a specific group of pixels can be picked out from all training examples to optimize for the network parameter $\mathbf{\theta}$ in the next iteration. Through optimizing the network parameter $\mathbf{\theta}$ and dynamically updating the importance in an alternative way, the proposed learning scheme is able to train the network $\mathcal{S}$ with a specific learning paradigms. In addition, when $h(\mathbf{w}_i)$ is given as the following indicator function,
\begin{equation}\label{eq:eq4.1}
\begin{aligned}
h(\mathbf{w}_i) = \mathcal{I}(\mathbf{w}_i, \mathbf{1}) = \left\{  
             \begin{array}{lr}  
             \infty, & \mathbf{w}_i \neq \mathbf{1}\\  
             0, & \mathbf{w}_i = \mathbf{1}\\  
             \end{array},  
\right.
\end{aligned}
\end{equation} 
the proposed learning scheme will degenerate to the traditional learning scheme in Eq.~\eqref{eq:eq1}. Therefore, the proposed adaptive importance learning scheme is a general learning framework for SISR.

In this study, we employ the proposed learning scheme in Eq.~\eqref{eq:eq2} with a carefully designed $h(\mathbf{w}_i)$ to train a given lightweight SISR network $\mathcal{S}$ with an easy-to-complex paradigm for capacity enhancement. To this end, the importance produced by the designed $h(\mathbf{w}_i)$ are required to conform with the following requirements. At beginning, the importance of complex pixels that are difficult to reconstruct will be suppressed (i.e., assigned to a small value close to zero) while the importance of pixels that are easy to reconstruct will be highlighted (i.e., assigned to a large value close to one). By doing this, $\mathcal{S}$ is encouraged to focus on learning to reconstruct easy pixels when its initial capacity is limited. Given the learned $\mathcal{S}$, importance $W$ will be gradually increased to expose $\mathcal{S}$ to more complex pixels for the next round of training, and thus the capacity of $\mathcal{S}$ will be enhanced. When the alternative minimization converges, the capacity of $\mathcal{S}$ can be maximized. In the following, we will introduce a carefully designed $h$ to update the importance $W$ as expected.

\begin{figure}[htp]
\setlength{\abovecaptionskip}{0pt}
\begin{center}
\includegraphics[height=2.2in,width=2.6in,angle=0]{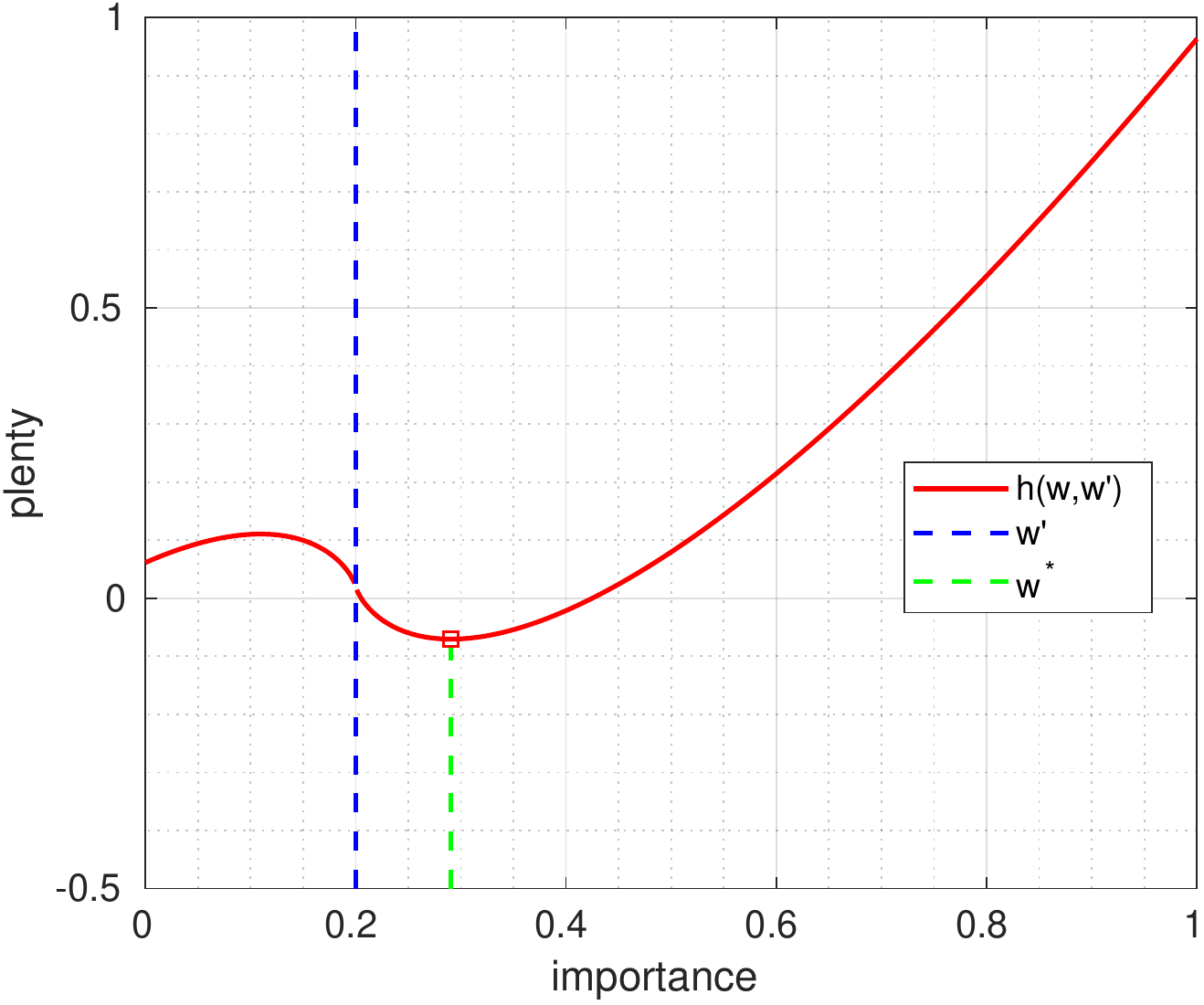}
\end{center}
\caption{The designed importance penalty function $h(\mathbf{w},\mathbf{w}')$ (e.g., $d=0.1$, $\lambda=0.1$).}
\label{fig:penalty}
\end{figure}

\begin{figure*}[htp]
\setlength{\abovecaptionskip}{0pt}
\begin{center}
\subfigure[Importance function] {\includegraphics[height=1.4in,width=1.75in,angle=0]{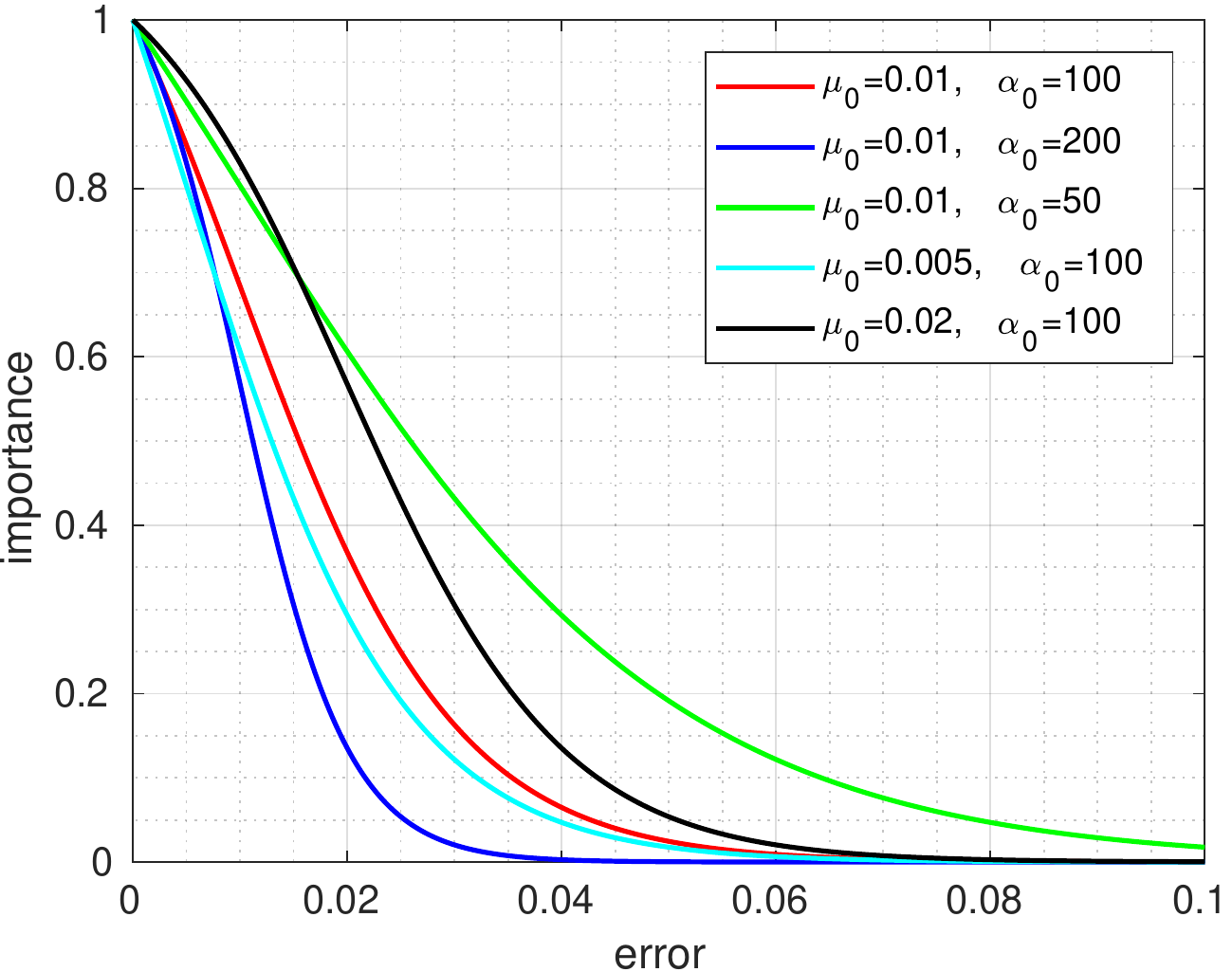}}
\hspace{-0.1cm}
\subfigure[Image] {\includegraphics[height=1.4in,width=1.4in,angle=0]{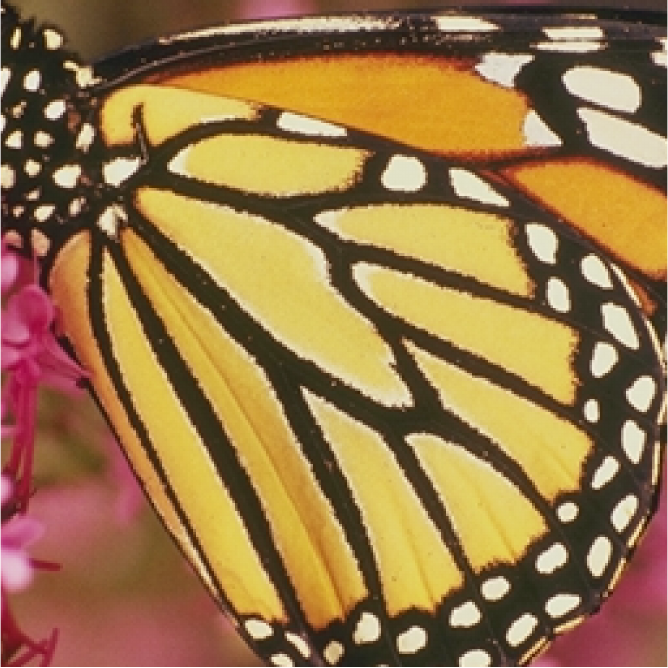}}
\hspace{-0.1cm}
\subfigure[Prediction error] {\includegraphics[height=1.4in,width=1.4in,angle=0]{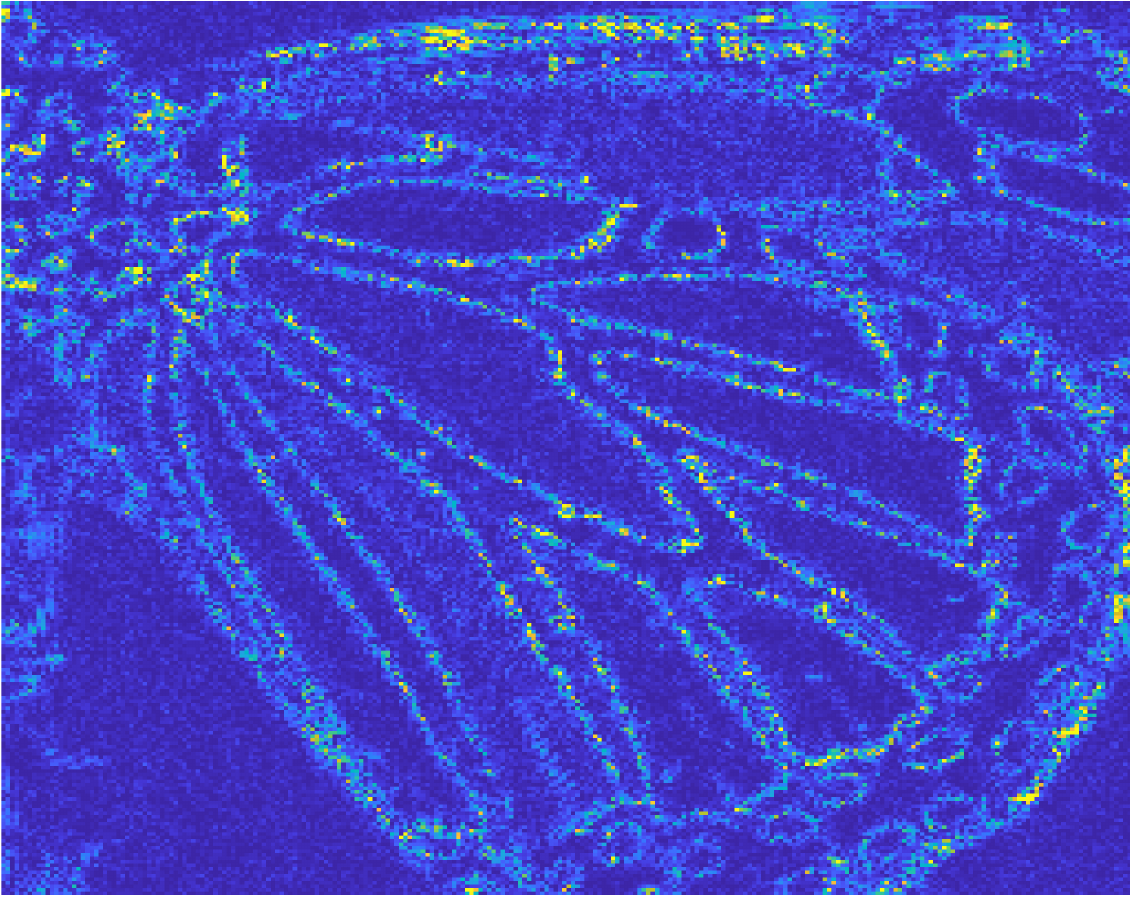}\includegraphics[height=1.4in,width=0.2in,angle=0]{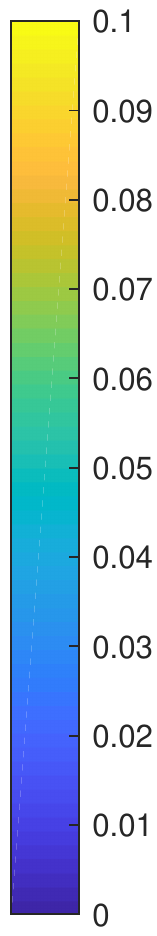}}
\hspace{-0.1cm}
\subfigure[Importance map] {\includegraphics[height=1.4in,width=1.4in,angle=0]{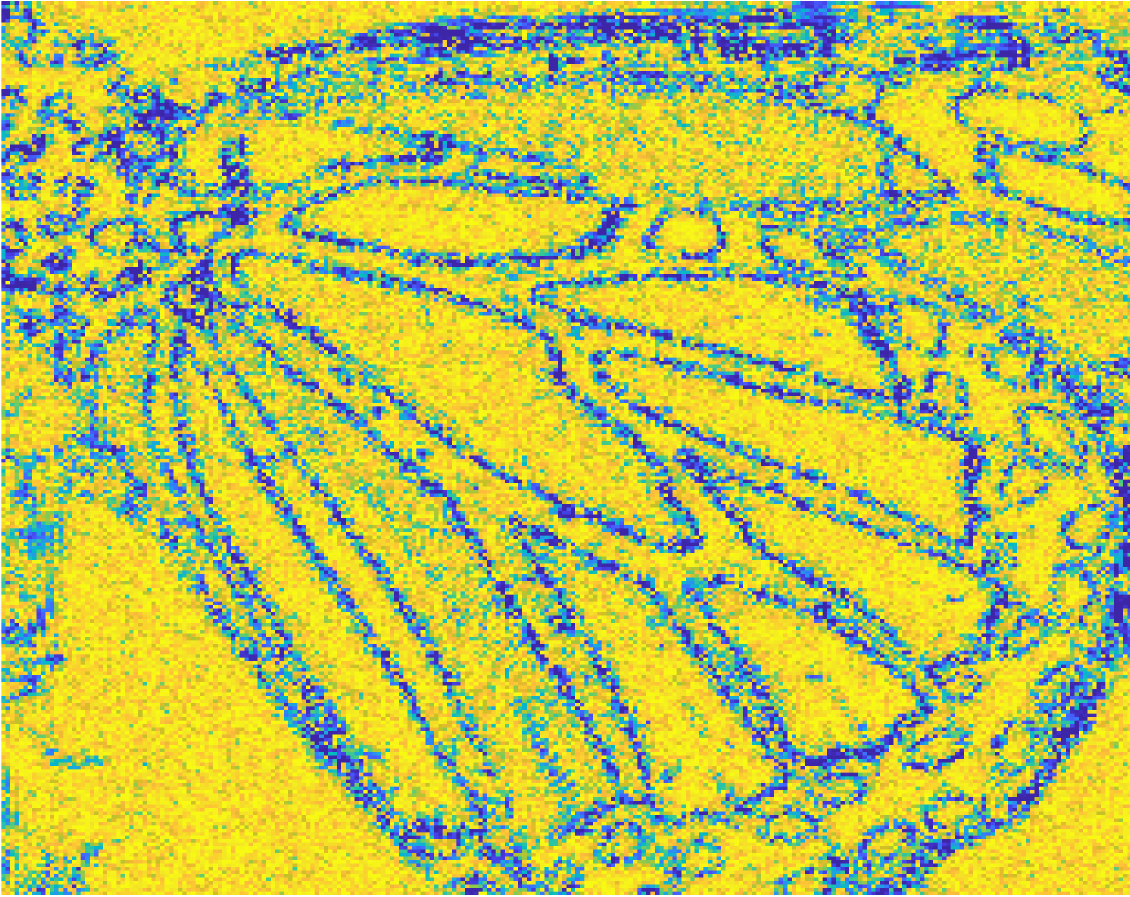}\includegraphics[height=1.4in,width=0.2in,angle=0]{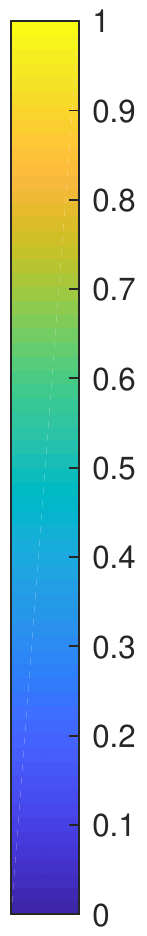}}
\end{center}
\caption{Importance initialization from a teacher network (e.g., VDSR~\citep{kim2016accurate}). (a) Importance function in Eq.~\eqref{eq:eq3} with different parameters. (b) Example image. (c) Prediction error (e.g., $\ell_2$ norm) from the teacher network. (d) Importance map for the example image with parameter $\mu_0=0.01,\alpha_0=100$.}
\label{fig:importance}
\end{figure*}

\subsection{Adaptive importance learning}\label{subsec:AIL}
According to the discussion above, we find that a basic principle for importance updating is to gradually increase the importance to feed $\mathcal{S}$ with more complex pixels in the next round of training. Moreover, the increment to importance should be determined by a decreasing function over the reconstruction difficulty of image pixels to guarantee the easy-to-complex learning paradigm. However, it is difficult to determine the reconstruction difficulty of pixels given an image. Intuitively, pixels lying on image details or within complex structures often are more difficult to reconstruct than those on flat areas. To quantitatively indicate the reconstruction difficult, we adopt the reconstruction error of the learned network $\mathcal{S}$ on pixels as a rough measure. This is inspired by the observation that most SISR methods can better reconstruct pixels on flat areas than those on image details. In addition, the reconstruction error of network $\mathcal{S}$ on all pixels can be directly indicated by the loss $l$ in Eq.~\eqref{eq:eq2}. Thus, the key for importance learning is to design an appropriate importance penalty function $h(\mathbf{w}_i)$. 

To comply with the importance learning principle mentioned above, we carefully design a penalty function $h$ and reformulate the learning scheme in Eq.~\eqref{eq:eq2} as follows
\begin{equation}\label{eq:eq5}
\begin{aligned}
&\min\limits_{\mathbf{\theta},W} \mathbb{E}(\mathbf{\theta}, W)=\frac{1}{n}\sum\limits^n_{i=1}\left[l(\mathbf{y}_i \odot \mathbf{w}_i, \mathcal{S}(\mathbf{x}_i, \mathbf{\theta}) \odot \mathbf{w}_i) + h(\mathbf{w}_i, \mathbf{w}'_i)\right],\\
&{\rm{s.t.}}{\kern 4pt}\forall i, \mathbf{w}'_i\leq \mathbf{w}_i \leq 1,
\end{aligned}
\end{equation}
where $\mathbf{w}'_i$ denotes the importance vector in previous iteration and $h(\mathbf{w}_i, \mathbf{w}'_i)$ is given as
\begin{equation}\label{eq:eq6}
\begin{aligned}
h(\mathbf{w}_i, \mathbf{w}'_i) = \sum\limits_{j}\left(w_{ji} - w'_{ji}\right)\left(\ln\frac{w_{ji} - w'_{ji}}{\lambda} - 1\right).
\end{aligned}
\end{equation}
In Eq.~\eqref{eq:eq6}, $w_{ji}$ and $w'_{ji}$ denote the $j$-th element in $\mathbf{w}_i$ and $\mathbf{w}'_i$, respectively. $\lambda$ is a predefined positive scalar. In the following, we will discuss the benefits of $h(\mathbf{w}_i, \mathbf{w}'_i)$ in details.

Similar as solving Eq.~\eqref{eq:eq2}, we adopt the alternative minimizing scheme to alternatively optimize $\theta$ and $W$ in Eq.~\eqref{eq:eq6}. Specifically, when the importance vectors $W$ are given, the learning problem for $\mathbf{\theta}$ can be well addressed by the back-propagation algorithm. When $\mathbf{\theta}$ is fixed, the learning problem for $W$ can be simplified as
\begin{equation}\label{eq:eq7}
\begin{aligned}
&\min\limits_{w} dw + (w - w')\left(\ln\frac{w - w'}{\lambda} - 1\right), \\
&{\rm s.t.}{\kern 4pt} w'\leq w \leq 1,
\end{aligned}
\end{equation}
where $w$ denotes the importance of a specific pixel in training samples (e.g., an element from $\mathbf{w}_i$) and $w'$ denotes the corresponding importance value in previous iteration (e.g., the corresponding element from $\mathbf{w}'$). $d$ denotes the reconstruction loss of the learned network $\mathcal{S}$ on the considered pixel. To solve the problem in Eq.~\eqref{eq:eq7}, we introduce the following result.

\begin{theorem}\label{theorem:theorem1}
Considering the constraint $w'\leq w \leq 1$, function $f(w)= dw + (w - w')\left(\ln\frac{w - w'}{\lambda} - 1\right)$ is a convex function and $f(w^*)$ reaches the minima when 
\begin{equation}\label{eq:eq8}
\begin{aligned}
w^* = w' + \lambda \cdot e^{-d}.
\end{aligned}
\end{equation}
\end{theorem}

\begin{proof}
Given $f(w)$ and the constraint $w'\leq w \leq 1$, we have $\frac{\partial^2 f(w)}{\partial w^2} = \frac{\partial}{\partial w}\left(d + \ln\frac{w - w'}{\lambda}\right) = \frac{\lambda}{w-w'}>0$. Thus, with the constraint $w'\leq w \leq 1$, $f(w)$ is a convex function, and the minima is reached when $ \frac{\partial f(w)}{\partial w}|_{w=w^*}=0$. We have

$$d + \ln\frac{w^* - w'}{\lambda} =0 \Rightarrow w^* = w' + \lambda \cdot e^{-d}$$. 

To further illustrate this point, a visual example can be found in Figure~\ref{fig:penalty}.
\QEDB
\end{proof}

According to Theorem~\ref{theorem:theorem1}, the problem in Eq.~\eqref{eq:eq7} has a closed-form solution as Eq.~\eqref{eq:eq8}. In Eq.~\eqref{eq:eq8}, the importance $w^*$ is updated by adding an increment to importance value $w'$ in the previous iteration. Since $\lambda \cdot e^{-d}\geq 0$, such a update rule enables to gradually increase the importance in each iteration. Moreover, the increment is determined by an decreasing function over the reconstruction loss of the pre-learned model $\mathcal{S}$ on the corresponding pixel, viz., a small increment is given when the reconstruction loss is large. Both aspects of principle for importance learning mentioned at the beginning of this subsection are satisfied. Therefore, the learning scheme in Eq.~\eqref{eq:eq5} with the penalty function $h(\mathbf{w},\mathbf{w}')$ is able to feed more and more complex pixels into $\mathcal{S}$ for training with an easy-to-complex paradigm through adaptively updating the importance vector as Eq.~\eqref{eq:eq8}. Furthermore, the proposed learning scheme enables the network to seamlessly assimilate the knowledge from a more powerful teacher network in the form of pixel importance initialization. This will be introduced in details in the following subsection.

\subsection{Importance initialization from the teacher}
In Eq.~\eqref{eq:eq5}, the proposed adaptive importance learning scheme depends on the the importance vectors $w'$ in previous iteration. This brings an intuitive problem in initializing the importance at beginning. According to the discussion at the beginning of Section~\ref{sec:proposed}, it is necessary to determine the importance of image pixels according to their reconstruction difficulty and complex pixels are expected to be assigned to smaller importance than that to easy pixels. Since $\mathcal{S}$ is unknown at beginning, it is infeasible to indicate the pixel importance according to the reconstruction error of $\mathcal{S}$ as Section~\ref{subsec:AIL}. To address this problem, we propose to learn important $W$ from a given more powerful teacher network $\mathcal{T}$. Similar as the learned $\mathcal{S}$, $\mathcal{T}$ will produce larger reconstruction error on complex pixels than those easy ones. Then, a decreasing function over the reconstruction error is employed to produce the importance. To well suppressing the complex pixel as well as highlight the easy ones at the beginning, we establish the following importance function
\begin{equation}\label{eq:eq3}
\begin{aligned}
g(x) = \frac{z}{1 + e^{(x - \mu_0)\alpha_0}},
\end{aligned}
\end{equation}
where $x$ denotes the reconstruction error (e.g, $\ell_2$ norm) of the teacher network $\mathcal{T}$ on a specific pixel and $g(x)$ is the corresponding importance value. $\mu_0$ and $\alpha_0$ denote the bias and scale parameters in this function. $z = (1 + e^{- \mu_0\alpha_0})$ is a normalization factor which scales the importance into $[0,1]$. To demonstrate the effectiveness of the importance function in Eq.~\eqref{eq:eq3}, we plot the profiles of $g(x)$ with different parameters as well as the estimated importance map on an example image in Figure~\ref{fig:importance}. It can be seen that $g(x)$ will produce a small importance when the reconstruction error is large, vice versa. On the example image, we can find that pixels lying on image details (i.e., exhibiting complex structures) are assigned to low importance, while pixels on flat areas are assigned to high importance. This complies with the intuition that pixels on image details are more difficult to reconstruct than those on flat areas.

Given the teacher network $\mathcal{T}$ and the importance function $g$, we can train the network $\mathcal{S}$ by solving the following problem
\begin{equation}\label{eq:eq4}
\begin{aligned}
&\min\limits_{\mathbf{\theta}} \mathbb{E}(\mathbf{\theta};W)=\frac{1}{n}\sum\limits^n_{i=1}l\left(\mathbf{y}_i \odot \mathbf{w}_i, \mathcal{S}(\mathbf{x}_i, \mathbf{\theta}) \odot \mathbf{w}_i\right),\\
&{\rm{s.t.}}{\kern 4pt}\forall i, \mathbf{w}_i = g(\mathcal{T}(\mathbf{x}_i)),\\
\end{aligned}
\end{equation}
where, for a concise formulation, we employ $g(\mathcal{T}(\mathbf{x}_i))$ to denote applying $g$ to the reconstruction error of $\mathcal{T}$ on each pixel in $\mathbf{x}_i$. In this learning scheme, the knowledge from the teacher network is distilled to guide training the network with the easy-complex paradigm.

\textbf{Relation to focal loss}
The proposed learning scheme in Eq.~\eqref{eq:eq4} is similar to the focal loss based learning scheme~\citep{lin2017focal}. Both of them dynamically reweight samples during the training procedure to enhance the capacity of network. However, they totally differ in the following three aspects. 1) With the proposed scheme, the learned model is forced to focus on easy cases, whereas focal loss encourages the network to focus on complex cases. 2) In Eq.~\eqref{eq:eq4}, the weights to training examples are determined by the prediction error of the given teacher model, while focal loss determines those weights based on the training error of the learned model. 3) Focal loss is proposed for training a more robust classifier or detector, while the proposed scheme aims at learning a more powerful compact SISR model.

\subsection{Algorithm}\label{sec:opt}
With the alternative minimizing scheme, the overall optimization procedure for the prosed adaptive importance learning scheme in Eq.~\eqref{eq:eq2} can be summarized into Algorithm~\ref{alg:bpl}. At the beginning, the network $\mathcal{S}$ is trained with the importance vectors $W$ initialized by the given teacher network $\mathcal{T}$ as Eq.~\eqref{eq:eq4}. Then, the learning scheme in Eq.~\eqref{eq:eq5} is carried out in $T$ iterations to gradually enhance the capacity of $\mathcal{S}$.

\begin{algorithm}
\caption{Adaptive importance learning (AIL)}
\label{alg:bpl}
\KwIn{Input HR-LR training pairs $\{\mathbf{x}_i, \mathbf{y}_i\}^n_{i=1}$, pre-trained teacher model $\mathcal{T}$, importance function $g$, penalty function $h$ and $\lambda$.}
\textit{1. Importance initialization from teacher}:\\
${\kern 4pt}$ (1) Learn importance $W$ as Eq.~\eqref{eq:eq4};\\
${\kern 4pt}$ (2) Update model parameter \\
${\kern 25pt}$ $\mathbf{\theta}^{*} = \arg\min_{\mathbf{\theta}} \mathbb{E}(\mathbf{\theta};W)$ as Eq.~\eqref{eq:eq5};\\
\textit{2. Adaptive importance learning}:\\
\textbf{For} {$t\leftarrow 1$ {\textbf{to}} $T$}
{\\
${\kern 6pt}$(1) Update $W^* = \arg\min_{W}\mathbb{E}(\mathbf{\theta}^*, W)$ as Eq.~\eqref{eq:eq8};\\
${\kern 6pt}$(2) Update $\mathbf{\theta}^* = \arg\min_{\mathbf{\theta}}\mathbb{E}(\mathbf{\theta}, W^*)$ as Eq.~\eqref{eq:eq6};\\
}
\textbf{End for}\\
\KwOut{$\mathbf{\theta}$-parameterized model $\mathcal{S}$.}
\end{algorithm}

It is noticeable that in theory Algorithm~\ref{alg:bpl} can well converge. Specifically, according to Eq.~\eqref{eq:eq8}, the importance vectors $W$ are gradually increased with the proceeding of iterations. When all elements in $W$ increase to $1$, the importance $W$ will be unchanged in the following iterations and Algorithm~\ref{alg:bpl} will converge, since no novel information will be provided by the training examples. More experimental evidence will be provided in Section~\ref{subsec:ablation}. 

In addition, different from previous methods~\citep{dong2016accelerating,shi2016real} that design new lightweight network architectures to deploy deep SISR methods onto real devices, the proposed adaptive importance learning scheme only focuses on how to enhance the capacity of network with a new training paradigm, and thus it can be directly applied to any given lightweight SISR network architecture. Experimental evidence will provided in Section~\ref{sec:exper}.

\section{Customizing lightweight SISR model}\label{sec:light}
Most of state-of-the-art SISR models~\citep{dong2016image,kim2016accurate,tai2017image,ledig2017photo} are inspired by the DCNN framework where the basic modules are convolution layer. To obtain a lightweight network, previous literatures~\citep{dong2016accelerating,shi2016real} propose to design new architectures (e.g., introducing a hourglass-shape structure or a sub-pixel convolution structure), which, however, cannot be conveniently applied to other DCNNs for SISR, especially when different scales of lightweight networks are required to fit various real devices. In this study, given a teacher network, we customize the lightweight network by directly reducing filters in each convolution layer to reduce the amount of output feature maps by a fixed ratio (e.g., $0<\rho<1$). By doing this, we can obtain different scales of lightweight networks with different $\rho$s. Given a fixed $\rho$, each convolution layer (i.e., except the input and output layer) in the obtained lightweight network reduces $1 - (1 -\rho)^2\%$ parameters as well as computational complexity, compared with that in the teacher network. The parameters and computational complexity of some lightweight networks are provided in Table~\ref{table:AIL_Complexity}. 

It is noticeable that the comparison between different ways of customizing lightweight network architecture is beyond the scope of this study. Our aim of adopting the way of reducing filters is to make it convenient to verify the effectiveness of the proposed learning scheme in enhancing different scales of lightweight networks.  

\section{Experimental results and analysis}\label{sec:exper}
In this section, we conduct extensive experiments to demonstrate the effectiveness of the proposed learning scheme in enhancing a given lightweight SISR network architecture.
\subsection{Dataset}
{\textbf{Training datasets}} Current SISR methods often adopt different training datasets. For example, the very large ImageNet dataset is adopted by~\citep{dong2016image}, while literatures~\citep{kim2016accurate,tai2017image} aggregate $91$ images from~\citep{yang2010image} and another $200$ images from the Berkeley Segmentation Dataset~\citep{martin2001database} together for training. In this study, we adopt the dataset utilized in~\citep{kim2016accurate} with $291$ images as benchmark to train all networks for fair comparison. In addition, rotation (e.g., with angle $90^\circ$, $180^\circ$, $270^\circ$), flip and downsampling (e.g., with ratio $0.5$, $0.7$, $1.0$) are further employed for data augmentation.

{\noindent\textbf{Test datasets}} Similar as ~\citep{huang2015single,kim2016accurate,tai2017image}, we adopt four benchmark datasets for performance evaluation, namely Set5~\citep{bevilacqua2012low}, Set14~\citep{zeyde2010single}, BSD100~\citep{timofte2014a+} and Urban100~\citep{huang2015single}, which contain $5$, $14$, $100$ and $100$ indoor and outdoor natural images, respectively.

\subsection{Teacher SISR networks}
In this study, we adopt two seminal DCNN architectures for SISR to customize the lightweight network as well as initializing importance for Algorithm~\ref{alg:bpl}, including {\texttt{VDSR}}~\citep{kim2016accurate} and {\texttt{DRRN}}~\citep{tai2017image}. Currently, the network architectures of most state-of-the-art SISR methods~\citep{mao2016image,kim2016deeply,lai2017deep} are inspired by these two models. In VDSR, $20$ fully convolution layers with global residual structure are employed to learn a deep mapping from a given LR input to an HR output. This is the first attempt to introduce the global residual structure into SISR, which enables a much deeper model than previous works~\citep{dong2016image} and improves the SISR performance obviously. According to~\citep{kim2016accurate}, $64$ feature maps are adopted for {\texttt{VDSR}} in this study. Recently, {\texttt{DRRN}} advances replacing the convolution layers in {\texttt{VDSR}} with a recursive block, which further improves the SISR performance as well as reducing the model parameters. As suggested in~\citep{tai2017image}, the recursive number and amount of feature maps are set $25$ and $128$, respectively.

\subsection{Training and testing setup}
For network training, we follow the standard protocol utilized in~\citep{kim2016accurate}. Specifically, we implement these two teacher networks mentioned above as well as the corresponding lightweight networks based on the codes released online~\footnote{VDSR: https://github.com/twtygqyy/pytorch-vdsr\\ DRRN: https://github.com/jt827859032/DRRN-pytorch}. With introducing the mean squared error (MSE) loss as $l$ into Eq.~\eqref{eq:eq1} and Eq.~\eqref{eq:eq2}, we train each network in $50$ epochs with batch size $128$ in the Pytorch framework~\citep{paszke2017pytorch}. Learning rate is initially set as $0.1$ and then decayed by a factor $10$ every $10$ epochs. Model parameters are learned by the SGD optimizer with momentum parameter $0.9$, weight decay parameter $1e^{-4}$ and gradient clip parameter $0.4$. In Algorithm~\ref{alg:bpl}, we set the pre-defined parameter $\lambda=0.15$ and maximum iterations $T = 10$. For the importance function $g$, the parameter $\alpha_0 = 0.01$ and $\mu_0=100$ are fixed in the following experiments.

In testing phase, we employ each learned network to improve the resolution of a given LR image with three different scaling factors $2,3,4$. To quantitatively evaluate the performance of each network, we adopt two standard criteria, namely peak signal-to-noise ratio (PSNR) and structured similarity (SSIM) to measure their super-resolution results.

\begin{figure*}[htp]
\setlength{\abovecaptionskip}{0pt}
\begin{center}
\includegraphics[height=1.6in,width=2in,angle=0]{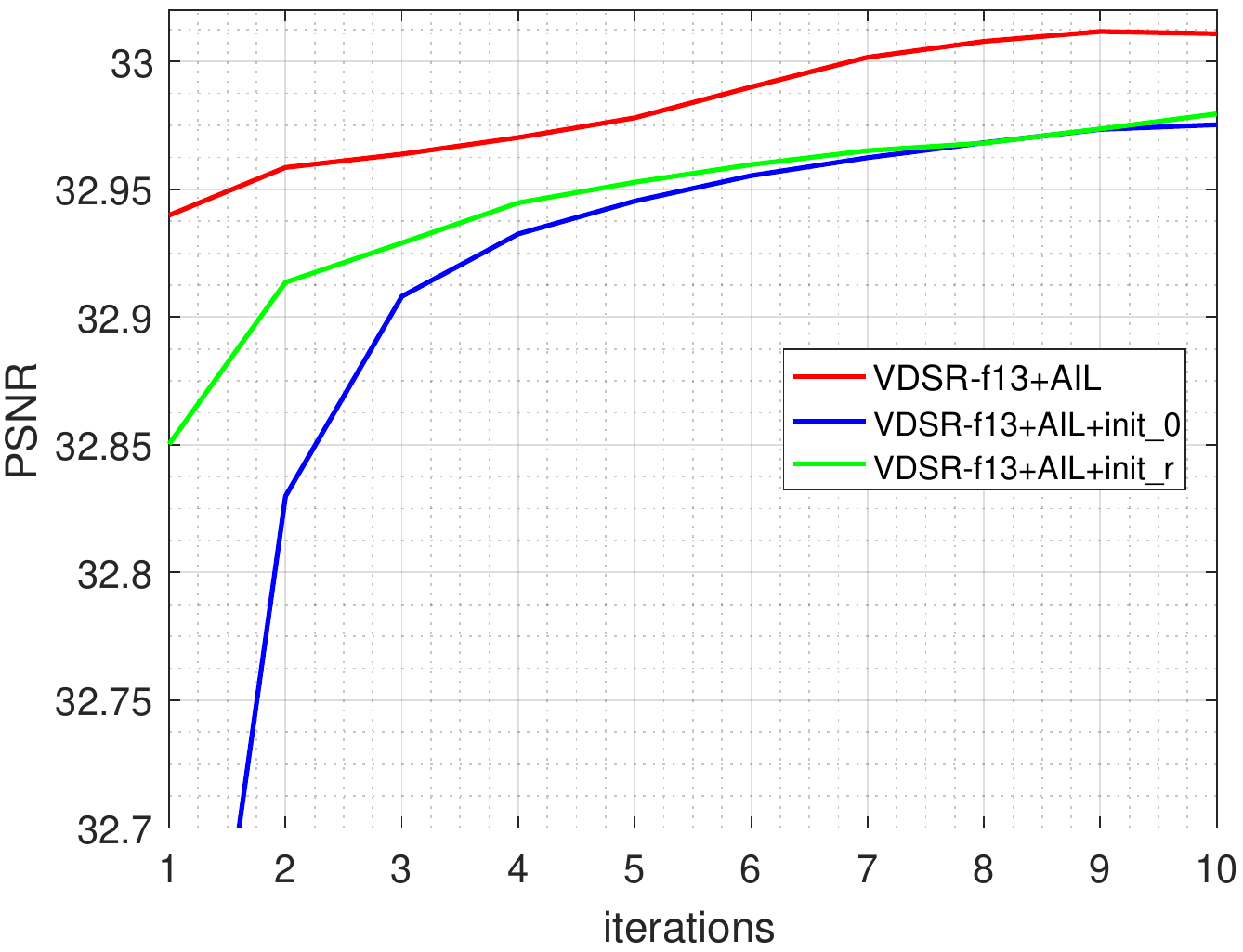}{\hspace{-0.05cm}}
\includegraphics[height=1.6in,width=2in,angle=0]{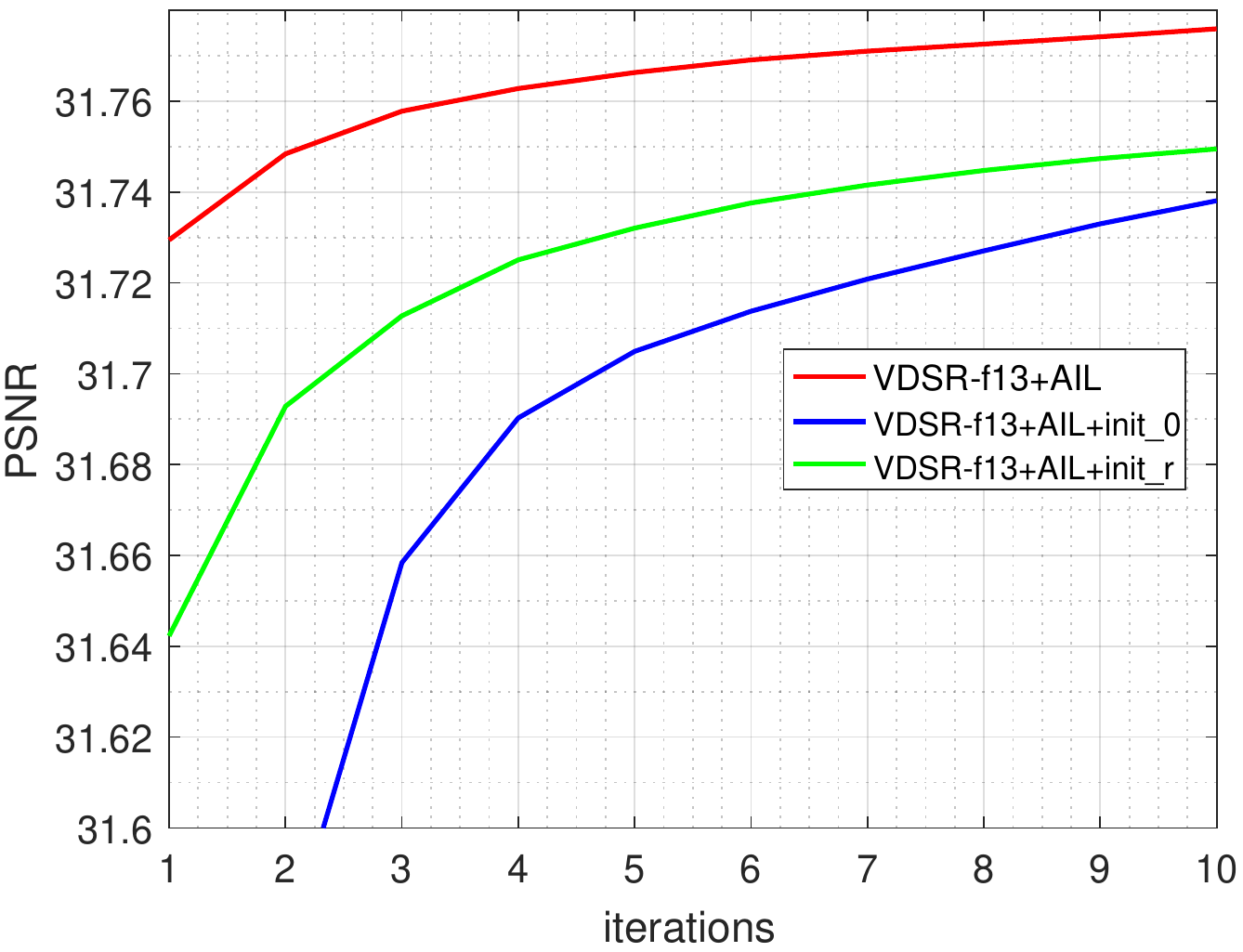}{\hspace{-0.05cm}}
\includegraphics[height=1.6in,width=2in,angle=0]{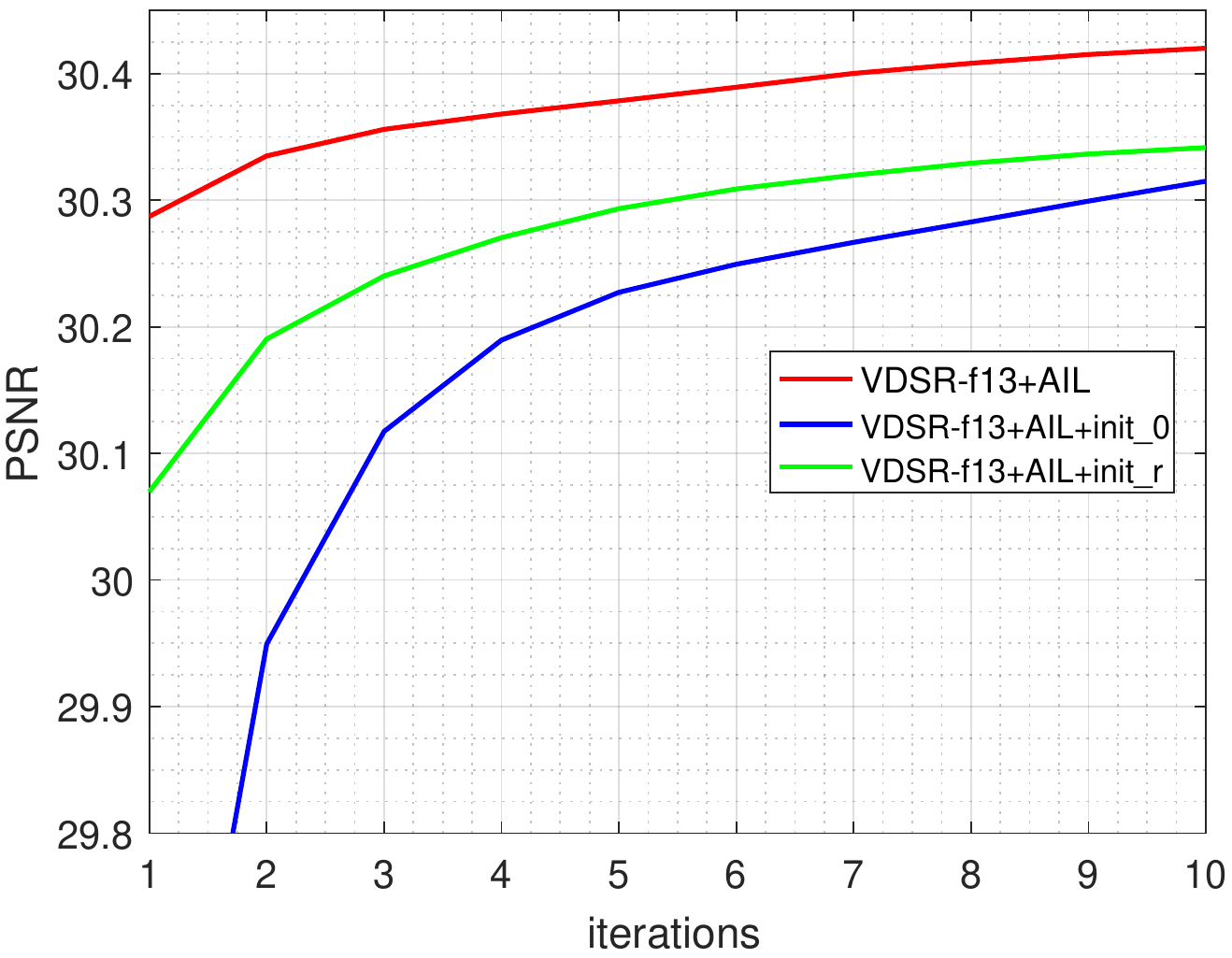}
\\
\subfigure[Set14]{\includegraphics[height=1.6in,width=2in,angle=0]{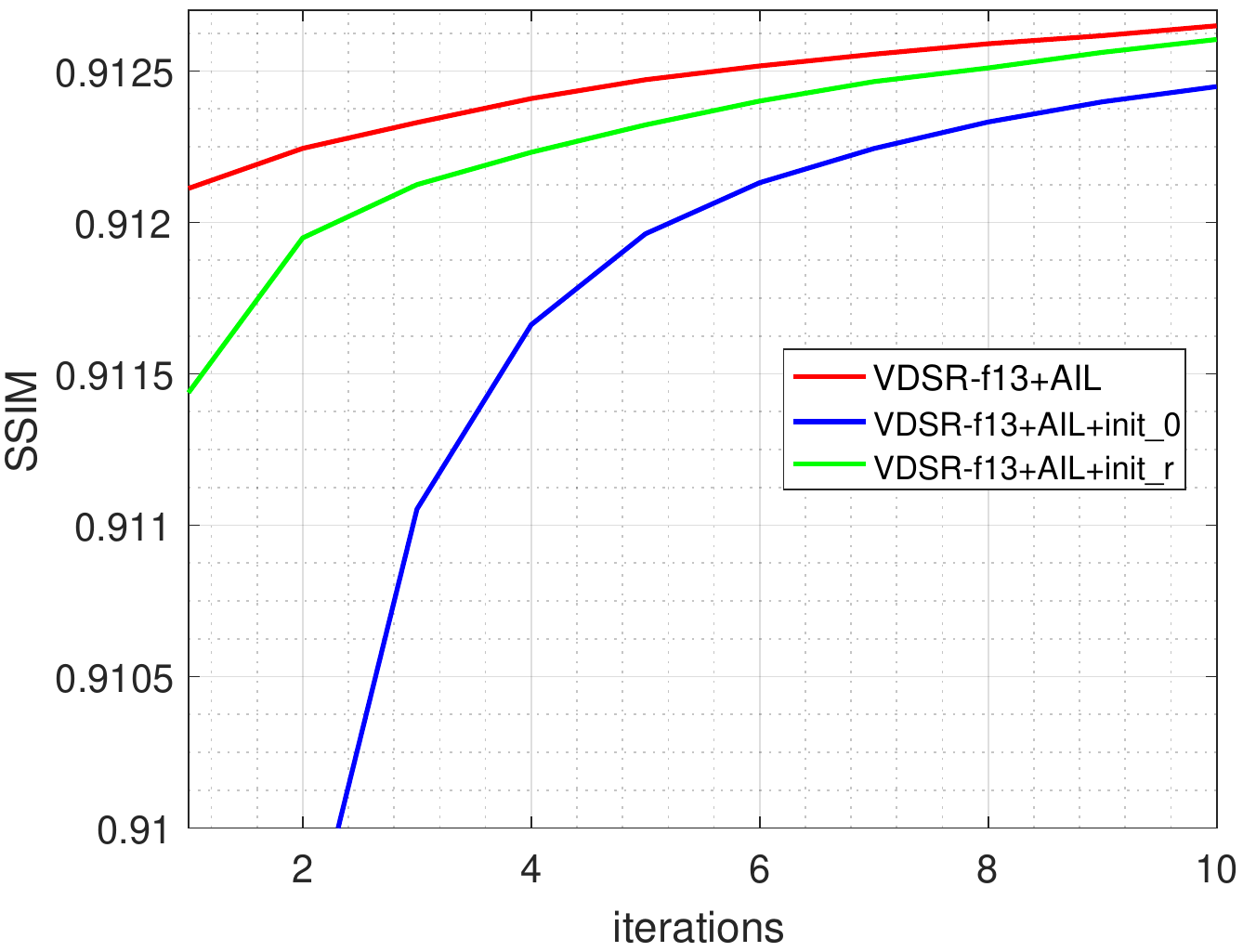}}{\hspace{-0.05cm}}
\subfigure[BSD100]{\includegraphics[height=1.6in,width=2in,angle=0]{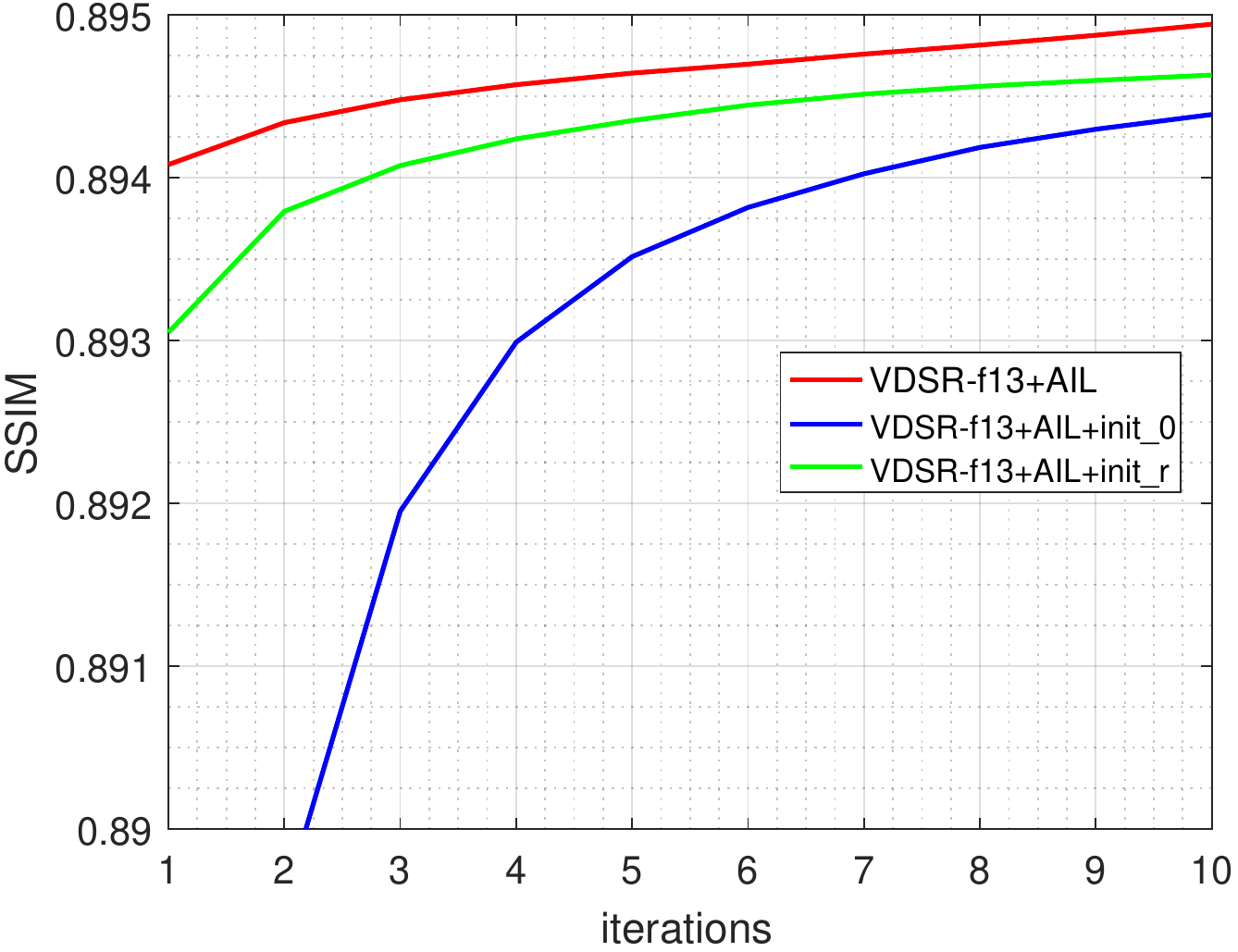}}{\hspace{-0.05cm}}
\subfigure[Urban100]{\includegraphics[height=1.6in,width=2in,angle=0]{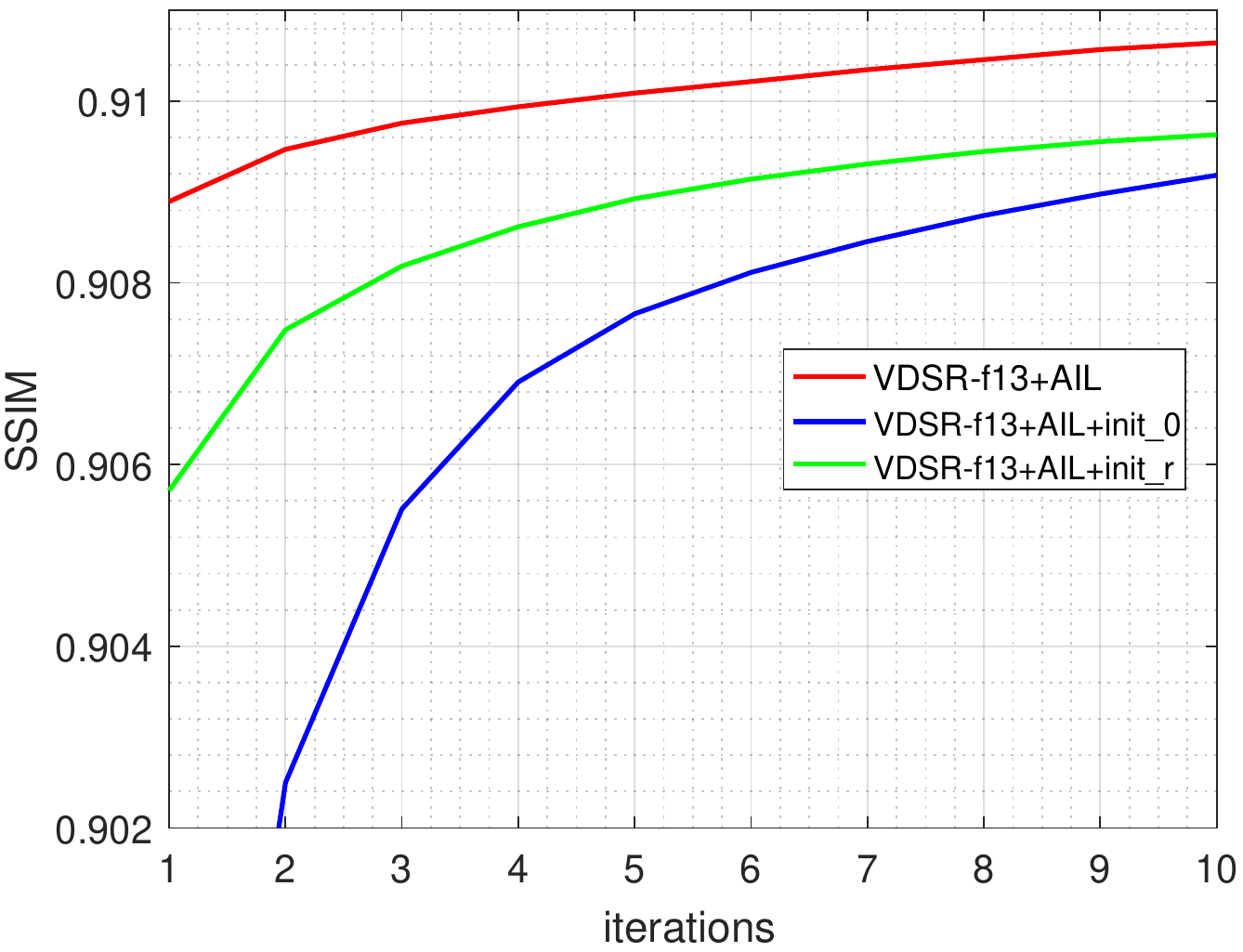}}
\end{center}
\caption{The performance (e.g., PSNR and SSIM) curves of {\texttt{VDSR-f13+AIL}}, {\texttt{VDSR-f13+AIL+init\_0}} and {\texttt{VDSR-f13+AIL+init\_r}}} within $10$ iterations on three test datasets.
\label{fig:convergence}
\end{figure*}

\begin{table*}\footnotesize%
\caption{Average PSNR/SSIM of {\texttt{VDSR-f13}}, {\texttt{VDSR-f13+ILT}} and {\texttt{VDSR-f13+AIL}} on four test datasets. The best results are in bold. {\blue{$\uparrow$PSNR/SSIM}} and {{$\downarrow$PSNR/SSIM}} denote the performance increase and decrease over {\texttt{VDSR-f13}}, respectively.}
\renewcommand{\arraystretch}{1.2}
\begin{center}
\begin{tabular}{l|c|c|cc|cc|cc}
\hline
Dataset & scale & VDSR-f13 & \multicolumn{2}{c|}{VDSR-f13+AIL+init\_0} & \multicolumn{2}{c|}{VDSR-f13+AIL+init\_r} & \multicolumn{2}{c}{VDSR-f13+AIL}\\
\hline
\multirow{3}{*}{Set5} & $\times$2 & 37.18/0.9580 & 37.39/0.9589 & {\blue{$\uparrow$0.21}}/{\blue{$\uparrow$0.0009}} & 37.41/0.9590 & {\blue{$\uparrow$0.23}}/{\blue{$\uparrow$0.0010}} & {\textbf{37.43/0.9591}} & {\blue{$\uparrow$0.25}}/{\blue{$\uparrow$0.0011}}\\
& $\times$3 & 33.07/0.9155 & 33.26/0.9184 & {\blue{$\uparrow$0.19}}/{\blue{$\uparrow$0.0029}} & 33.27/0.9183 & {\blue{$\uparrow$0.21}}/{\blue{$\uparrow$0.0028}} & {\textbf{33.36/0.9196}} & {\blue{$\uparrow$0.29}}/{\blue{$\uparrow$0.0041}}\\
& $\times$4 & 30.74/0.8724 & 30.87/0.8764 & {\blue{$\uparrow$0.13}}/{\blue{$\uparrow$0.0040}} & 30.90/0.8769 & {\blue{$\uparrow$0.16}}/{\blue{$\uparrow$0.0045}} & {\textbf{30.99/0.8788}} & {\blue{$\uparrow$0.25}}/{\blue{$\uparrow$0.0065}}\\
\hline
\multirow{3}{*}{Set14} & $\times$2 & 32.84/0.9115 & 32.98/0.9124 & {\blue{$\uparrow$0.13}}/{\blue{$\uparrow$0.0009}} & 32.98/0.9126 & {\blue{$\uparrow$0.14}}/{\blue{$\uparrow$0.0011}} & {\textbf{33.01/0.9126}} & {\blue{$\uparrow$0.17}}/{\blue{$\uparrow$0.0011}}\\
& $\times$3 & 29.53/0.8269 & 29.65/0.8295 & {\blue{$\uparrow$0.11}}/{\blue{$\uparrow$0.0025}} & 29.68/0.8300 & {\blue{$\uparrow$0.15}}/{\blue{$\uparrow$0.0030}} & {\textbf{29.73/0.8309}} & {\blue{$\uparrow$0.20}}/{\blue{$\uparrow$0.0040}}\\
& $\times$4 & 27.75/0.7600 & 27.84/0.7633 & {\blue{$\uparrow$0.10}}/{\blue{$\uparrow$0.0033}} & 27.87/0.7636 & {\blue{$\uparrow$0.12}}/{\blue{$\uparrow$0.0036}} & {\textbf{27.93/0.7654}} & {\blue{$\uparrow$0.18}}/{\blue{$\uparrow$0.0054}}\\
\hline
\multirow{3}{*}{BSD100} & $\times$2 & 31.63/0.8930 & 31.74/0.8944 & {\blue{$\uparrow$0.11}}/{\blue{$\uparrow$0.0014}} & 31.75/0.8946 & {\blue{$\uparrow$0.12}}/{\blue{$\uparrow$0.0017}} & {\textbf{31.78/0.8949}} & {\blue{$\uparrow$0.15}}/{\blue{$\uparrow$0.0020}}\\
& $\times$3 & 28.54/0.7906 & 28.62/0.7932 & {\blue{$\uparrow$0.08}}/{\blue{$\uparrow$0.0026}} & 28.65/0.7935 & {\blue{$\uparrow$0.11}}/{\blue{$\uparrow$0.0030}} & {\textbf{28.69/0.7946}} & {\blue{$\uparrow$0.14}}/{\blue{$\uparrow$0.0041}}\\
& $\times$4 & 27.03/0.7169 & 27.08/0.7193 & {\blue{$\uparrow$0.05}}/{\blue{$\uparrow$0.0024}} & 27.11/0.7199 & {\blue{$\uparrow$0.08}}/{\blue{$\uparrow$0.0030}} & {\textbf{27.14/0.7212}} & {\blue{$\uparrow$0.11}}/{\blue{$\uparrow$0.0043}}\\
\hline
\multirow{3}{*}{Urban100} & $\times$2 & 30.06/0.9056 & 30.32/0.9092 & {\blue{$\uparrow$0.25}}/{\blue{$\uparrow$0.0036}} & 30.34/0.9096 & {\blue{$\uparrow$0.28}}/{\blue{$\uparrow$0.0041}} & {\textbf{30.42/0.9106}} & {\blue{$\uparrow$0.36}}/{\blue{$\uparrow$0.0051}}\\
& $\times$3 & 26.42/0.8081 & 26.60/0.8146 & {\blue{$\uparrow$0.18}}/{\blue{$\uparrow$0.0065}} & 26.66/0.8158 & {\blue{$\uparrow$0.24}}/{\blue{$\uparrow$0.0077}} & {\textbf{26.76/0.8189}} & {\blue{$\uparrow$0.34}}/{\blue{$\uparrow$0.0108}}\\
& $\times$4 & 24.64/0.7312 & 24.74/0.7369 & {\blue{$\uparrow$0.10}}/{\blue{$\uparrow$0.0057}} & 24.79/0.7379 & {\blue{$\uparrow$0.15}}/{\blue{$\uparrow$0.0067}} & {\textbf{24.87/0.7416}} & {\blue{$\uparrow$0.22}}/{\blue{$\uparrow$0.0105}}\\
\hline
\end{tabular}
\end{center}
\label{table:Effect_Initial}
\end{table*}

\begin{table*}\footnotesize%
\caption{Average PSNR/SSIM of {\texttt{VDSR-f13}}, {\texttt{VDSR-f13+Distil}} and {\texttt{VDSR-f13+AIL}} on four test datasets. The best results are in bold. {\blue{$\uparrow$PSNR/SSIM}} and {{$\downarrow$PSNR/SSIM}} denote the performance increase and decrease over {\texttt{VDSR-f13}}, respectively.}
\renewcommand{\arraystretch}{1.2}
\begin{center}
\begin{tabular}{l|c|c|cc|cc}
\hline
Dataset & scale & VDSR-f13 & \multicolumn{2}{c|}{VDSR-f13+Distil} & \multicolumn{2}{c}{VDSR-f13+AIL}\\
\hline
\multirow{3}{*}{Set5} & $\times$2 & 37.18/0.9580 & 37.21/0.9581 & {\blue{$\uparrow$0.03}}/{\blue{$\uparrow$0.0001}} & {\textbf{37.43/0.9591}} & {\blue{$\uparrow$0.25}}/{\blue{$\uparrow$0.0011}}\\
& $\times$3 & 33.07/0.9155 & 33.09/0.9158 & {\blue{$\uparrow$0.02}}/{\blue{$\uparrow$0.0003}} & {\textbf{33.36/0.9196}} & {\blue{$\uparrow$0.29}}/{\blue{$\uparrow$0.0041}}\\
& $\times$4 & 30.74/0.8724 & 30.68/0.8705 & $\downarrow$0.06/$\downarrow$0.0019 & {\textbf{30.99/0.8788}} & {\blue{$\uparrow$0.25}}/{\blue{$\uparrow$0.0065}}\\
\hline
\multirow{3}{*}{Set14} & $\times$2 & 32.84/0.9115 & 32.88/0.9115 & {\blue{$\uparrow$0.04}}/{\blue{$\uparrow$0.0000}} & {\textbf{33.01/0.9126}} & {\blue{$\uparrow$0.17}}/{\blue{$\uparrow$0.0011}}\\
& $\times$3 & 29.53/0.8269 & 29.55/0.8272 & {\blue{$\uparrow$0.01}}/{\blue{$\uparrow$0.0003}} & {\textbf{29.73/0.8309}} & {\blue{$\uparrow$0.20}}/{\blue{$\uparrow$0.0040}}\\
& $\times$4 & 27.75/0.7600 & 27.71/0.7587 & $\downarrow$0.03/$\downarrow$0.0013 & {\textbf{27.93/0.7654}} & {\blue{$\uparrow$0.18}}/{\blue{$\uparrow$0.0054}}\\
\hline
\multirow{3}{*}{BSD100} & $\times$2 & 31.63/0.8930 & 31.65/0.8932 & {\blue{$\uparrow$0.02}}/{\blue{$\uparrow$0.0002}} & {\textbf{31.78/0.8949}} & {\blue{$\uparrow$0.15}}/{\blue{$\uparrow$0.0020}}\\
& $\times$3 & 28.54/0.7906 & 28.56/0.7911 & {\blue{$\uparrow$0.02}}/{\blue{$\uparrow$0.0005}} & {\textbf{28.69/0.7946}} & {\blue{$\uparrow$0.14}}/{\blue{$\uparrow$0.0041}}\\
& $\times$4 & 27.03/0.7169 & 27.01/0.7160 & $\downarrow$0.02/$\downarrow$0.0009 & {\textbf{27.14/0.7212}} & {\blue{$\uparrow$0.11}}/{\blue{$\uparrow$0.0043}}\\
\hline
\multirow{3}{*}{Urban100} & $\times$2 & 30.06/0.9056 & 30.07/0.9058 & {\blue{$\uparrow$0.01}}/{\blue{$\uparrow$0.0003}} & {\textbf{30.42/0.9106}} & {\blue{$\uparrow$0.36}}/{\blue{$\uparrow$0.0051}}\\
& $\times$3 & 26.42/0.8081 & 26.47/0.8097 & {\blue{$\uparrow$0.04}}/{\blue{$\uparrow$0.0015}} & {\textbf{26.76/0.8189}} & {\blue{$\uparrow$0.34}}/{\blue{$\uparrow$0.0108}}\\
& $\times$4 & 24.64/0.7312 & 24.59/0.7290 & $\downarrow$0.05/$\downarrow$0.0022 & {\textbf{24.87/0.7416}} & {\blue{$\uparrow$0.22}}/{\blue{$\uparrow$0.0105}}\\
\hline
\end{tabular}
\end{center}
\label{table:Effect_Distil}
\end{table*}

\begin{figure*}[htp]
\setlength{\abovecaptionskip}{0pt}
\begin{center}
\begin{tabu} to 1\textwidth{ccccc}
Ground truth & VDSR-f13 & VDSR-f13+Distil & VDSR-f13+AIL & VDSR\\
\includegraphics[height=1.9in,width=1.4in,angle=0]{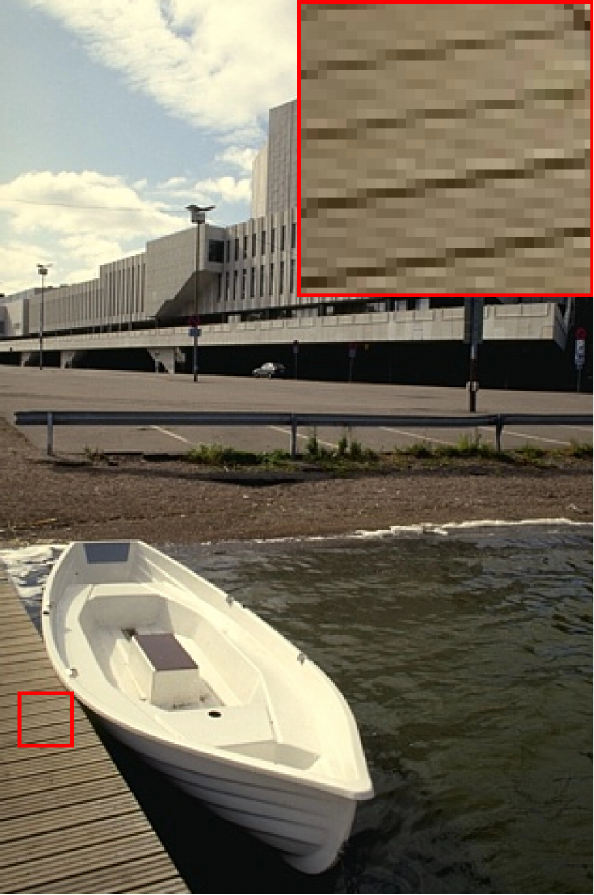}\hspace{-0.32cm} &
\includegraphics[height=1.9in,width=1.4in,angle=0]{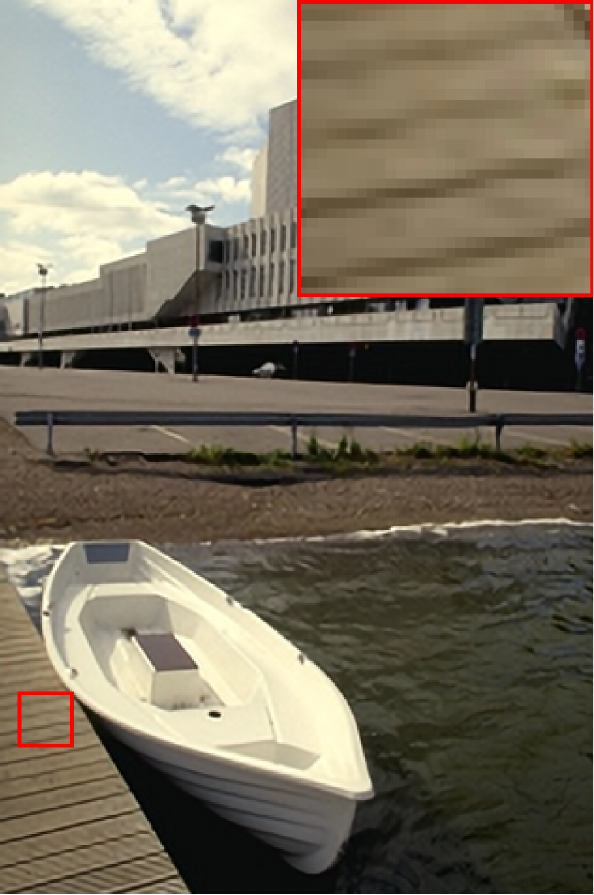} \hspace{-0.43cm} &
\includegraphics[height=1.9in,width=1.4in,angle=0]{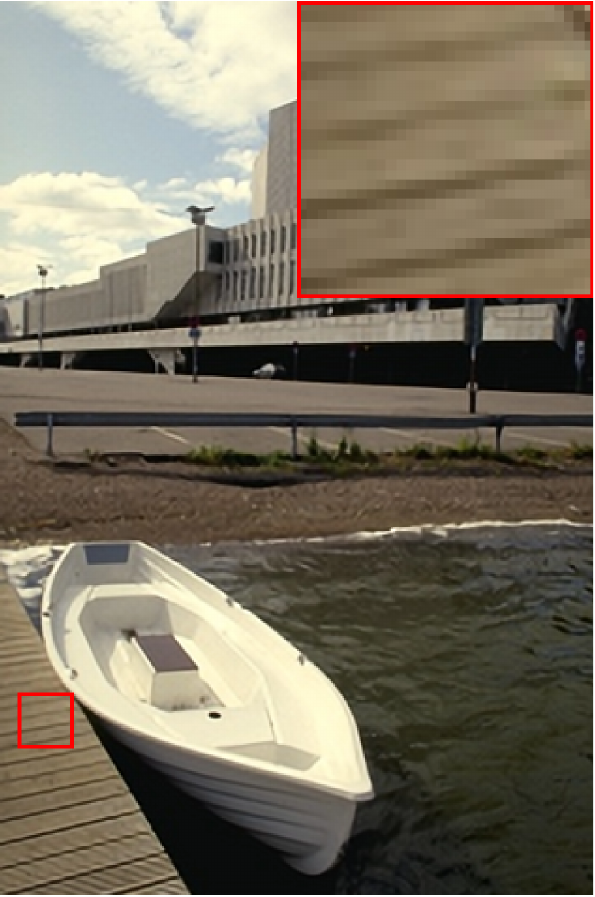}
\hspace{-0.43cm} &
\includegraphics[height=1.9in,width=1.4in,angle=0]{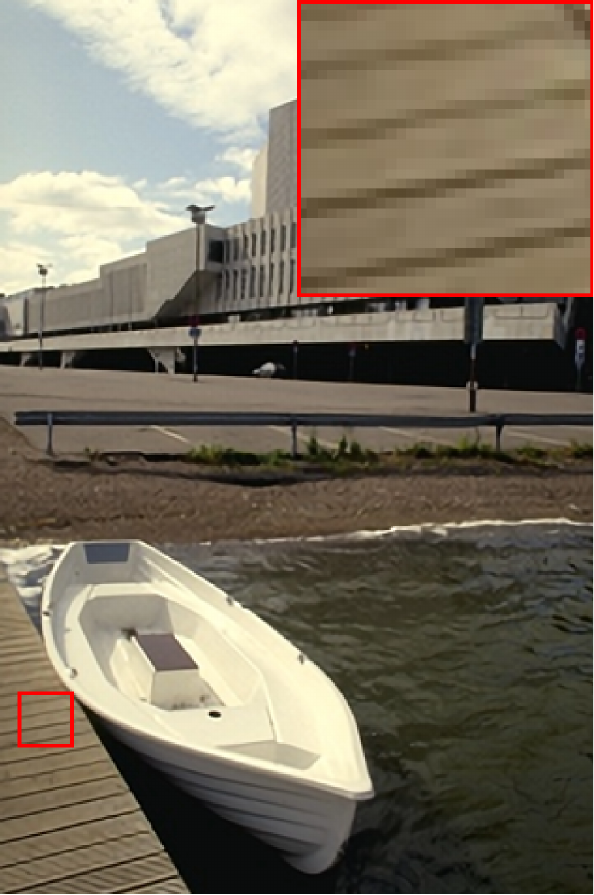}
\hspace{-0.43cm} &
\includegraphics[height=1.9in,width=1.4in,angle=0]{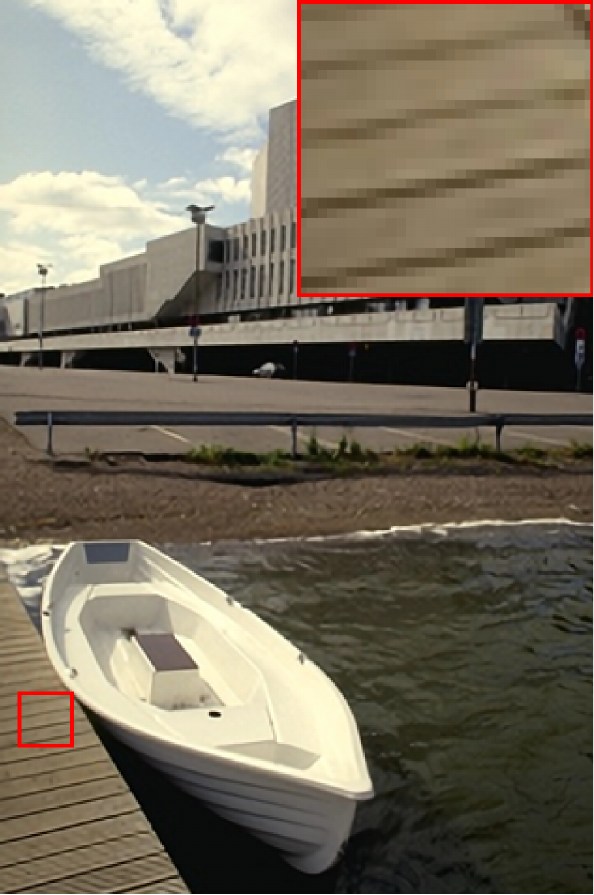}
\vspace{-0.07cm}
\\
{\sz{(PSNR/SSIM)}} & {\sz{(29.93/0.8858)}} & {\sz{(30.05/0.8874)}} & {\sz{(30.34/0.8928)}} & {\sz{(30.50/0.8958)}}\\
\includegraphics[height=1.9in,width=1.4in,angle=0]{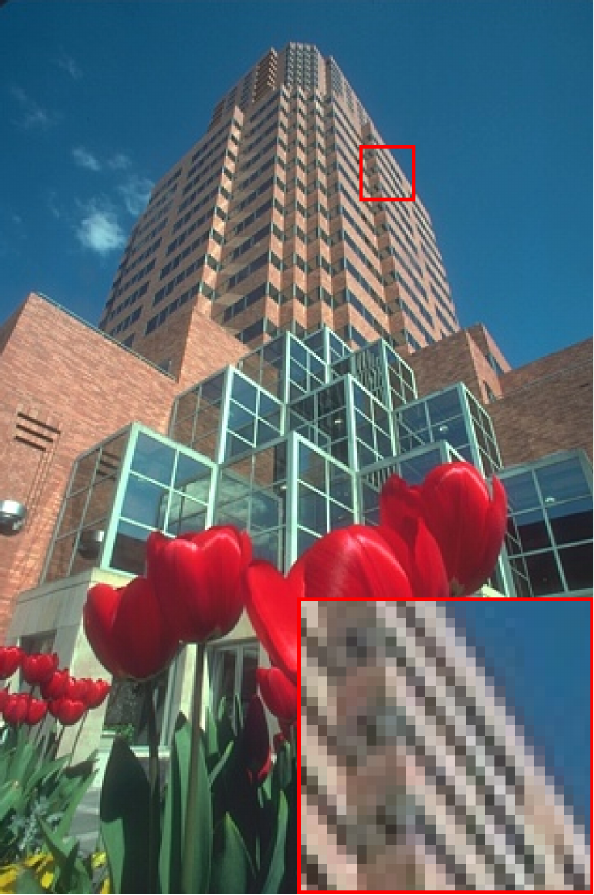}\hspace{-0.32cm} &
\includegraphics[height=1.9in,width=1.4in,angle=0]{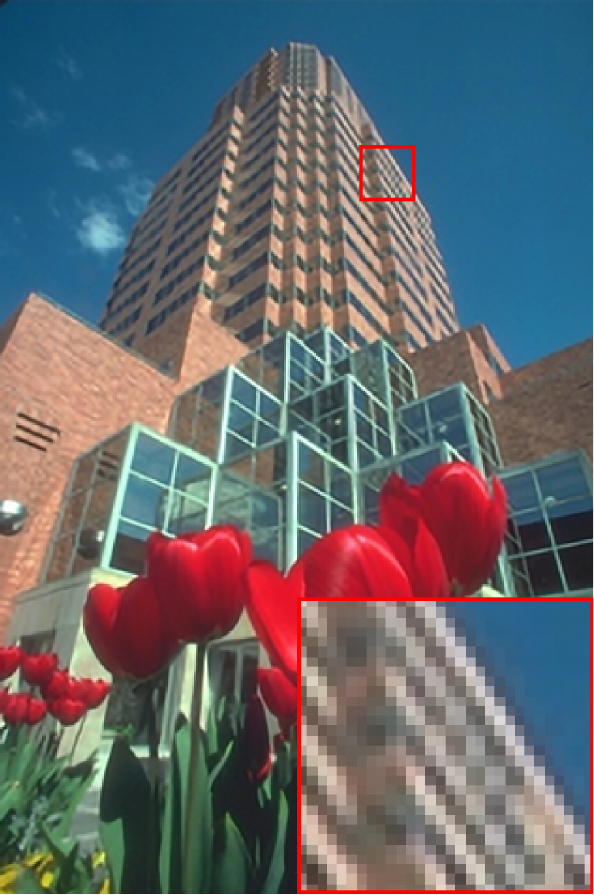} \hspace{-0.43cm} &
\includegraphics[height=1.9in,width=1.4in,angle=0]{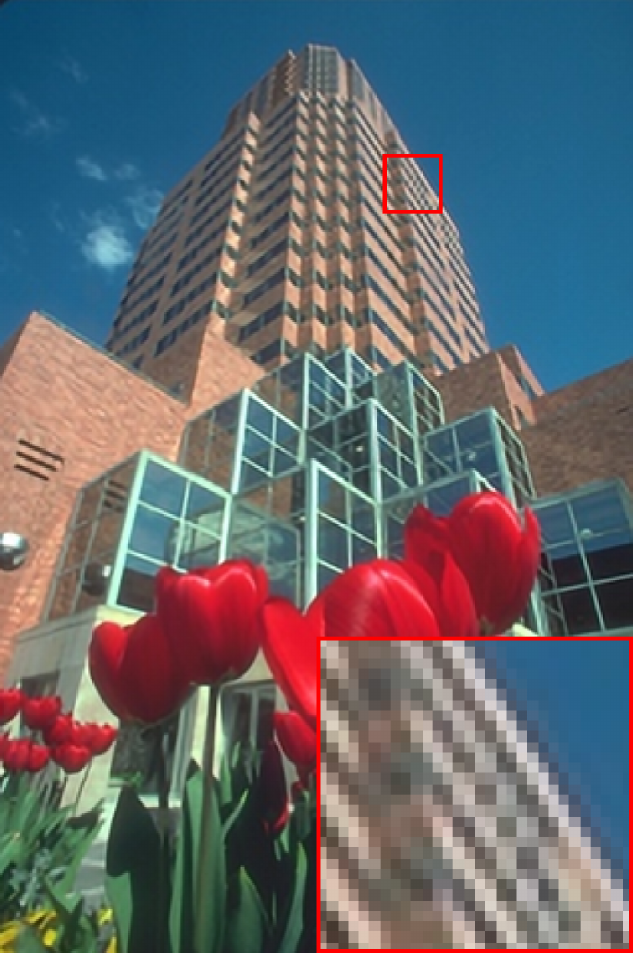}
\hspace{-0.43cm} &
\includegraphics[height=1.9in,width=1.4in,angle=0]{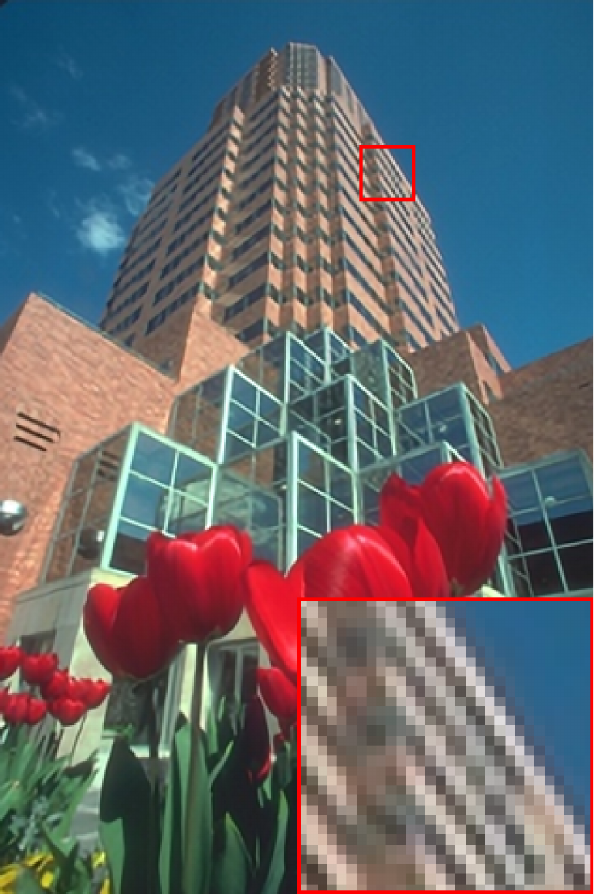}
\hspace{-0.43cm} &
\includegraphics[height=1.9in,width=1.4in,angle=0]{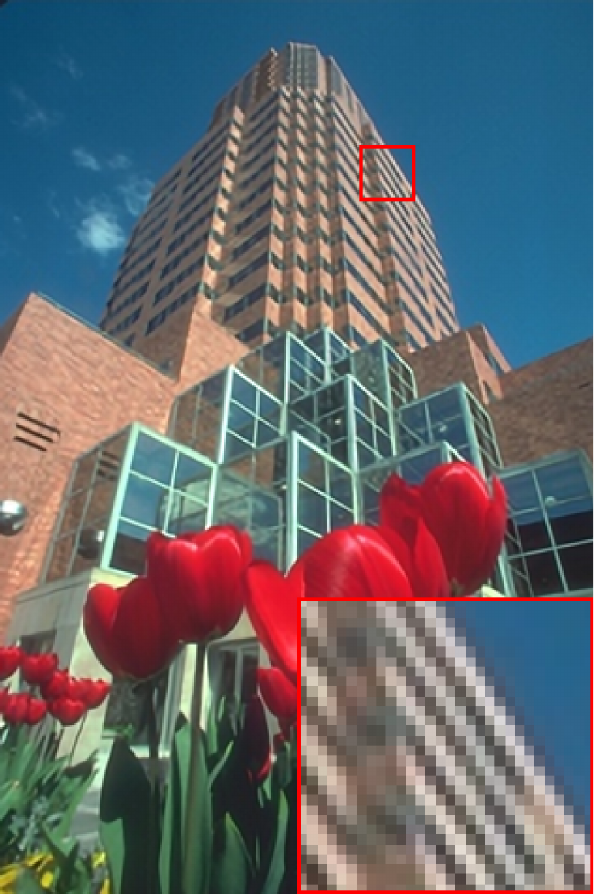}
\vspace{-0.07cm}
\\
{\sz{(PSNR/SSIM)}} & {\sz{(31.17/0.9348)}} & {\sz{(31.29/0.9358)}} & {\sz{(31.64/0.9394)}} & {\sz{(31.89/0.9441)}}\\
\end{tabu}
\end{center}
\caption{Visual super-resolution results of {\texttt{VDSR-f13}}, {\texttt{VDSR-f13+ILT}}, {\texttt{VDSR-f13+Distil}} and {\texttt{VDSR}}. First row: the super-resolution results for image '78004' from BSD100 dataset when scaling factor is $2$. Second row: the super-resolution results for image '86000' from BSD100 dataset when scaling factor is $2$.}
\label{fig:VDSR-f13}
\end{figure*}

\subsection{Ablation study}\label{subsec:ablation}
In this part, we mainly focus on demonstrating the effect of the proposed adaptive importance learning scheme and the importance initialization scheme, the difference between the proposed learning scheme and the knowledge distillation and the convergence of Algorithm~\ref{alg:bpl}. To this end, we adopt {\texttt{VDSR}} as the teacher network $\mathcal{T}$ and obtain the corresponding lightweight network by reducing the amount of feature maps in each convolution layer with a fixed ratio $\rho=0.8$ as Section~\ref{sec:light}. Concretely, the amount of feature maps in each convolution layer of the lightweight network is reduced from $64$ to $13$, viz., the parameters and the computational complexity is only $4\%$ of that in {\texttt{VDSR}}, shown as Table~\ref{table:AIL_Complexity}.

\subsubsection{Effect of adaptive importance learning}\label{subsubsec:eff_AIL}
We propose the adaptive importance learning scheme to train a given lightweight network with an easy-to-complex principle as well as gradually enhance the network generalization capacity. To demonstrate this point, we train the given lightweight network above with Algorithm~\ref{alg:bpl} and evaluate it on three test datasets (e.g., Set14, BSD100 and Urban100). For simplicity, we term the obtained network {\texttt{VDSR-f13+\-AIL}} where {\texttt{-f13}} denotes the amount of feature maps in each convolution layer of the given lightweight network. The performance (e.g., PSNR and SSIM) curves of {\texttt{VDSR-f\-13+AIL}} within $T=10$ iterations are depicted in Figure~\ref{fig:convergence}. It can be seen that on each dataset both the PSNR and SSIM measures of {\texttt{VDSR-f13+AIL}} are gradually increased with the proceeding of iterations. To further clarify this point, we implement two variants of {\texttt{VDSR-f13+AIL}} by training the same lightweight network with Algorithm~\ref{alg:bpl} but initializing the importance $W$ as zeros and random values, respectively. For simplicity, we term these two variants {\texttt{VDSR-f13+AI\-L+init\_0}} and {\texttt{VDSR-f13+AIL+init\_r}}. The corresponding performance curves for these two variants are also provided in Figure~\ref{fig:convergence}. We can find that the adaptive importance learning scheme always gradually enhance the super-resolution performance with the proceeding of iterations, which is robust to the initialization of importance. This is because that in Eq.~\eqref{eq:eq8} the importance is gradually increased based on its previous value, which enables to expose the network with more and more complex pixels. In addition, the final numerical results of these three methods on four test datasets are reported in Table~\ref{table:Effect_Initial}. To illustrate their superiority, we also implement a baseline method, {\texttt{VDSR-f13}}, which is obtained by training the given lightweight network with the traditional learning scheme in Eq.~\eqref{eq:eq1}. It can be seen that {\texttt{VDSR-f13+AIL}} and the other two variants obviously outperforms {\texttt{VDSR-f13}} in all cases. This demonstrates that the proposed easy-to-complex learning strategy can better exploit the super-resolution capacity of the given lightweight network than the traditional learning scheme in Eq.~\eqref{eq:eq1}.

In summary, we can conclude that the proposed adaptive importance learning scheme is able to gradually enhance the capacity of the given lightweight network and ultimately obviously improve the super-resolution performance, which, furthermore, is robust to the importance initialization.

\subsubsection{Effect of importance initialization from teacher}
In the proposed adaptive importance learning scheme as Algorithm~\ref{alg:bpl}, we initialize the importance $W$ by distilling knowledge from a given teacher network as Eq.~\eqref{eq:eq4}. It is noticeable that this is not the unique way for importance initialization. As mentioned in Section~\ref{subsubsec:eff_AIL}, the importance can be simply initialized as zeros or random values. To illustrate the effectiveness of the proposed importance initialization scheme, we compare {\texttt{VSDR-f13+AIL}} with its two variants, namely {\texttt{VDSR-f13+AIL+init\_0}} and {\texttt{VDSR-f13+AIL+init\_r}}. Their performance curves and the numerical comparison results can be found in Figure~\ref{fig:convergence} and Table~\ref{table:Effect_Initial}. As shown in Figure~\ref{fig:convergence}, importance initialization from a teacher network in {\texttt{VSDR-f13+AIL}} leads to much better initial capacity of network in the first iteration than that from both the zero and the random importance initialization in other two variants. For example, on Urban100 dataset, the superiority of {\texttt{VSDR-f13+AIL}} over other two variants is up to $0.2$db. Moreover, with the proceeding of iterations, {\texttt{VSDR-f13+AIL}} obviously outperforms the other two variants in all cases and {\texttt{VDSR-f13+AIL+init\_r}} often surpasses {\texttt{VDSR-f13+AI\-L+init\_0}}. Similar results also occur on the numerical results of these three methods, shown as Table~\ref{table:Effect_Initial}. The reason for their performance difference comes from the following two aspects. On one hand, when the importance $W$ is initialized as zeros, no examples will be chosen to train the network in the {\textit{Importance initialization from teacher}} step of Algorithm~\ref{alg:bpl}, and the network with randomly initialized weights will be directly fed into the {\textit{Adaptive importance learning}} step to update the importance based on its reconstruction error. Thus, the resulted importance will render the learning scheme deviating from starting with easy pixels and the following training procedures are prone to be trapped into a bad local minima. On the other hand, when $W$ is randomly initialized, the learning scheme is also prone to deviate from the principle of starting with easy pixels. In contrast to the case with zero-initialized $W$, randomly initialized $W$ enables to train the network in the {\textit{Importance learning from teacher}} step of Algorithm~\ref{alg:bpl} with some selected pixels, which leads to better initial network capacity as well as the final results, shown as the results of {\texttt{VDSR-f13+AIL+i\-nit\_r}} and {\texttt{VDSR-f13+AIL+init\_0}} in Figure~\ref{fig:convergence} and Table~\ref{table:Effect_Initial}. In this study, the proposed importance initialization from teacher enables the network to start with easy pixels, thus producing the best performance. Therefore, we can conclude that importance initialization from teacher can benefit providing better initial capacity of network as well as the ultimate super-resolution performance.

\subsubsection{Comparison with knowledge distillation}
The proposed adaptive importance learning scheme initializes the importance from a given teacher network, which is similar to the prevailing knowledge distillation scheme~\citep{hinton2015distilling}. Both of them distil specific knowledge from a given teacher model to train the student model for better generalization capacity. The difference is that the knowledge distillation scheme forces the student network to mimic the soften output of the given teacher network, whereas the proposed scheme distils the importance from the teacher network to guide the lightweight network focusing on handling easy pixels at beginning. To further clarify their difference, we implement a variant of {\texttt{VSDR-f13+AIL}} by training the same lightweight network with the knowledge distillation scheme~\citep{hinton2015distilling} as
\begin{equation}\label{eq:eq9}
\begin{aligned}
\min\limits_{\mathbf{\theta}} \frac{1}{n}\sum\nolimits^n_{i=1}\left[l\left(\mathbf{y}_i, \mathcal{S}(\mathbf{x}_i, \mathbf{\theta})\right) + \beta \cdot l\left(\mathcal{T}(\mathbf{x}_i), \mathcal{S}(\mathbf{x}_i, \mathbf{\theta})\right)\right]
\end{aligned}
\end{equation}
where $\beta$ is set as $0.1$ for the best performance. The numerical results of this variant (i.e., termed {\texttt{VSDR-f13+Distil}}), {\texttt{VSDR-f13+AIL}} and {\texttt{VSDR-f13}} on four test datasets are provided in Table~\ref{table:Effect_Distil}. It can be found that {\texttt{VSDR-f13+Distil}} only gives comparable results to that of the baseline {\texttt{VDSR-f13}} and is far inferior to {\texttt{VSDR-f13+AIL}}. To further clarify this point, we depict some visual results of these three networks in Figure~\ref{fig:VDSR-f13}. We can find that compared with {\texttt{VSDR-f13}} and {\texttt{VSDR-f13+Distil}}, {\texttt{VDSR-f13+AIL}} recovers more image details and the produced results are even close to that of {\texttt{VDSR}} with full parameters. The reason is intuitive. In~\citep{hinton2015distilling}, knowledge distillation scheme is utilized in the classification problem where the soften output of the teacher network can provide more valuable information than the discrete labels in ground truth. However, SISR is a regression problem where the ground truth is inherently continuous. Thus, the output of teacher network fails to provide more valuable information than the ground truth. In contrast, the proposed scheme enables to train the lightweight network with an easy-to-complex paradigm, which can enhance the generalization capacity of network.

\subsubsection{Convergence}
In Algorithm~\ref{alg:bpl}, the network training and the importance learning are conducted in an alternative way. Thus, it is necessary to analysis the convergence of Algorithm~\ref{alg:bpl}. In addition to the theoretical illustration in Section~\ref{sec:opt}, we further depict the PSNR and SSIM curves of {\texttt{VSDR-f13+AIL}} within $T=10$ iterations on three test datasets in Figure~\ref{fig:convergence}. It can be found that {\texttt{VSDR-f13+AIL}} gradually improves the performance and ultimately converges with the proceeding of iterations.

\begin{table}\footnotesize%
\caption{Parameters, computation complexity (e.g., FLOPS) and average running time (e.g., seconds) of different scales of lightweight networks for {\texttt{VDSR}}. The running time is evaluated over Set5 dataset with scaling factor $3$ on a single CPU.}
\renewcommand{\arraystretch}{1.2}
\begin{center}
\begin{tabular}{l|c|c|c}
\hline
Method & Parameters & Complexity & Time\\
\hline
VDSR & 665$K$ & 2491$M$ &  5.26\\
VDSR-f32/+AIL & 166$K$ & 642$M$ &  2.13\\
VDSR-f22/+AIL & 79$K$ & 311$M$ &  1.46\\
VDSR-f16/+AIL & 42$K$ & 170$M$ &  0.96\\
VDSR-f13/+AIL & 28$K$ & 115$M$ &  0.77\\
\hline
DRRN & 297$K$ & 483337$M$ & 65.98\\
DRRN-f25/+AIL & 12$K$ & 18463$M$ & 4.21\\
\hline
\end{tabular}
\end{center}
\label{table:AIL_Complexity}
\end{table}

\begin{table*}\footnotesize%
\caption{Average PSNR/SSIM of {\texttt{VDSR-f16}}, {\texttt{VDSR-f16+AIL}} and {\texttt{VDSR}} on four test datasets. {\blue{$\uparrow$PSNR/SSIM}} and {{$\downarrow$PSNR/SSIM}} denote the performance increase and decrease over VDSR-f16, respectively.}
\renewcommand{\arraystretch}{1.2}
\begin{center}
\begin{tabular}{l|c|c|cc|c}
\hline
Dataset & scale & VDSR-f16 & \multicolumn{2}{c|}{VDSR-f16+AIL} & VDSR\\
\hline
\multirow{3}{*}{Set5} & $\times$2 & 37.23/0.9581 & 37.51/0.9594 & {\blue{$\uparrow$0.28}}/{\blue{$\uparrow$0.0013}} & 37.59/0.9596\\
& $\times$3 & 33.21/0.9174 & 33.45/0.9205 & {\blue{$\uparrow$0.25}}/{\blue{$\uparrow$0.0031}} & 33.69/0.9227\\
& $\times$4 & 30.77/0.8731 & 31.05/0.8805 & {\blue{$\uparrow$0.29}}/{\blue{$\uparrow$0.0073}} & 31.34/0.8846\\
\hline
\multirow{3}{*}{Set14} & $\times$2 & 32.85/0.9115 & 33.05/0.9131 & {\blue{$\uparrow$0.19}}/{\blue{$\uparrow$0.0017}} & 33.08/0.9135\\
& $\times$3 & 29.63/0.8288 & 29.79/0.8318 & {\blue{$\uparrow$0.15}}/{\blue{$\uparrow$0.0030}} & 29.90/0.8339\\
& $\times$4 & 27.78/0.7608 & 27.96/0.7665 & {\blue{$\uparrow$0.18}}/{\blue{$\uparrow$0.0058}} & 28.10/0.7699\\
\hline
\multirow{3}{*}{BSD100} & $\times$2 & 31.65/0.8931 & 31.82/0.8955 & {\blue{$\uparrow$0.17}}/{\blue{$\uparrow$0.0024}} & 31.87/0.8961\\
& $\times$3 & 28.61/0.7926 & 28.72/0.7956 & {\blue{$\uparrow$0.10}}/{\blue{$\uparrow$0.0029}} & 28.81/0.7979\\
& $\times$4 & 27.05/0.7175 & 27.17/0.7222 & {\blue{$\uparrow$0.11}}/{\blue{$\uparrow$0.0047}} & 27.26/0.7253\\
\hline
\multirow{3}{*}{Urban100} & $\times$2 & 30.06/0.9056 & 30.53/0.9120 & {\blue{$\uparrow$0.47}}/{\blue{$\uparrow$0.0064}} & 30.65/0.9135\\
& $\times$3 & 26.58/0.8132 & 26.83/0.8210 & {\blue{$\uparrow$0.25}}/{\blue{$\uparrow$0.0077}} & 27.08/0.8273\\
& $\times$4 & 24.68/0.7327 & 24.93/0.7439 & {\blue{$\uparrow$0.24}}/{\blue{$\uparrow$0.0112}} & 25.15/0.7520\\
\hline
\end{tabular}
\end{center}
\label{table:AIL_f16}
\end{table*}

\begin{table*}\footnotesize%
\caption{Average PSNR/SSIM of {\texttt{VDSR-f22}}, {\texttt{VDSR-f22+AIL}} and {\texttt{VDSR}} on four test datasets. The results of {\texttt{VDSR-f32+AIL}} comparable to or over {\texttt{VDSR}} are in bold. {\blue{$\uparrow$PSNR/SSIM}} and {{$\downarrow$PSNR/SSIM}} denote the performance increase and decrease over {\texttt{VDSR-f22}}, respectively.}
\renewcommand{\arraystretch}{1.2}
\begin{center}
\begin{tabular}{l|c|c|cc|c}
\hline
Dataset & scale & VDSR-f22 & \multicolumn{2}{c|}{VDSR-f22+ILT} & VDSR\\
\hline
\multirow{3}{*}{Set5} & $\times$2 & 37.34/0.9586 & {\textbf{37.59/0.9597}} & {\blue{$\uparrow$0.25}}/{\blue{$\uparrow$0.0011}} & 37.59/0.9596\\
& $\times$3 & 33.32/0.9190 & 33.54/0.9211 & {\blue{$\uparrow$0.22}}/{\blue{$\uparrow$0.0022}} & 33.69/0.9227\\
& $\times$4 & 30.87/0.8759 & 31.14/0.8818 & {\blue{$\uparrow$0.27}}/{\blue{$\uparrow$0.0059}} & 31.34/0.8846\\
\hline
\multirow{3}{*}{Set14} & $\times$2 & 32.94/0.9123 & {\textbf{33.11/0.9136}} & {\blue{$\uparrow$0.17}}/{\blue{$\uparrow$0.0013}} & 33.08/0.9135\\
& $\times$3 & 29.71/0.8304 & 29.82/0.8324 & {\blue{$\uparrow$0.11}}/{\blue{$\uparrow$0.0020}} & 29.90/0.8339\\
& $\times$4 & 27.86/0.7632 & 28.00/0.7676 & {\blue{$\uparrow$0.14}}/{\blue{$\uparrow$0.0044}} & 28.10/0.7699\\
\hline
\multirow{3}{*}{BSD100} & $\times$2 & 31.73/0.8942 & {\textbf{31.87/0.8961}} & {\blue{$\uparrow$0.14}}/{\blue{$\uparrow$0.0019}} & 31.87/0.8961\\
& $\times$3 & 28.67/0.7942 & 28.75/0.7964 & {\blue{$\uparrow$0.08}}/{\blue{$\uparrow$0.0022}} & 28.81/0.7979\\
& $\times$4 & 27.10/0.7194 & 27.19/0.7231 & {\blue{$\uparrow$0.09}}/{\blue{$\uparrow$0.0037}} & 27.26/0.7253\\
\hline
\multirow{3}{*}{Urban100} & $\times$2 & 30.28/0.9086 & {\textbf{30.64/0.9133}} & {\blue{$\uparrow$0.36}}/{\blue{$\uparrow$0.0048}} & 30.65/0.9135\\
& $\times$3 & 26.72/0.8170 & 26.92/0.8233 & {\blue{$\uparrow$0.19}}/{\blue{$\uparrow$0.0063}} & 27.08/0.8273\\
& $\times$4 & 24.78/0.7368 & 24.99/0.7464 & {\blue{$\uparrow$0.22}}/{\blue{$\uparrow$0.0096}} & 25.15/0.7520\\
\hline
\end{tabular}
\end{center}
\label{table:AIL_f22}
\end{table*}

\begin{table*}\footnotesize%
\caption{Average PSNR/SSIM of {\texttt{VDSR-f32}}, {\texttt{VDSR-f32+AIL}} and {\texttt{VDSR}} on four test datasets. The results of {\texttt{VDSR-f32+AIL}} comparable to or over {\texttt{VDSR}} are in bold. {\blue{$\uparrow$PSNR/SSIM}} and {{$\downarrow$PSNR/SSIM}} denote the performance increase and decrease over {\texttt{VDSR-f32}}, respectively.}
\renewcommand{\arraystretch}{1.2}
\begin{center}
\begin{tabular}{l|c|c|cc|c}
\hline
Dataset & scale & VDSR-f32 & \multicolumn{2}{c|}{VDSR-f32+AIL} & VDSR\\
\hline
\multirow{3}{*}{Set5} & $\times$2 & 37.49/0.9593 & {\textbf{37.68/0.9601}} & {\blue{$\uparrow$0.19}}/{\blue{$\uparrow$0.0008}} & 37.59/0.9596\\
& $\times$3 & 33.35/0.9191 & {\textbf{33.68/0.9227}} & {\blue{$\uparrow$0.32}}/{\blue{$\uparrow$0.0036}} & 33.69/0.9227\\
& $\times$4 & 31.01/0.8783 & 31.26/0.8840 & {\blue{$\uparrow$0.25}}/{\blue{$\uparrow$0.0057}} & 31.34/0.8846\\
\hline
\multirow{3}{*}{Set14} & $\times$2 & 33.03/0.9130 & {\textbf{33.19/0.9144}} & {\blue{$\uparrow$0.16}}/{\blue{$\uparrow$0.0014}} & 33.08/0.9135\\
& $\times$3 & 29.71/0.8306 & {\textbf{29.89/0.8339}} & {\blue{$\uparrow$0.18}}/{\blue{$\uparrow$0.0032}} & 29.90/0.8339\\
& $\times$4 & 27.94/0.7654 & 28.08/0.7695 & {\blue{$\uparrow$0.14}}/{\blue{$\uparrow$0.0041}} & 28.10/0.7699\\
\hline
\multirow{3}{*}{BSD100} & $\times$2 & 31.81/0.8953 & 31.93/0.8970 & {\blue{$\uparrow$0.12}}/{\blue{$\uparrow$0.0017}} & 31.87/0.8961\\
& $\times$3 & 28.67/0.7943 & 28.80/0.7979 & {\blue{$\uparrow$0.13}}/{\blue{$\uparrow$0.0035}} & 28.81/0.7979\\
& $\times$4 & 27.15/0.7212 & 27.24/0.7248 & {\blue{$\uparrow$0.09}}/{\blue{$\uparrow$0.0037}} & 27.26/0.7253\\
\hline
\multirow{3}{*}{Urban100} & $\times$2 & 30.51/0.9116 & {\textbf{30.84/0.9155}} & {\blue{$\uparrow$0.32}}/{\blue{$\uparrow$0.0040}} & 30.65/0.9135\\
& $\times$3 & 26.74/0.8177 & 27.05/0.8270 & {\blue{$\uparrow$0.31}}/{\blue{$\uparrow$0.0093}} & 27.08/0.8273\\
& $\times$4 & 24.88/0.7412 & 25.10/0.7506 & {\blue{$\uparrow$0.21}}/{\blue{$\uparrow$0.0094}} & 25.15/0.7520\\
\hline
\end{tabular}
\end{center}
\label{table:AIL_f32}
\end{table*}

\begin{figure*}[htp]
\setlength{\abovecaptionskip}{0pt}
\begin{center}
\begin{tabu} to 1\textwidth{ccccc}
Ground truth & VDSR-f16 & VDSR-f16+AIL & VDSR\\
\includegraphics[height=1.2in,width=1.7in,angle=0]{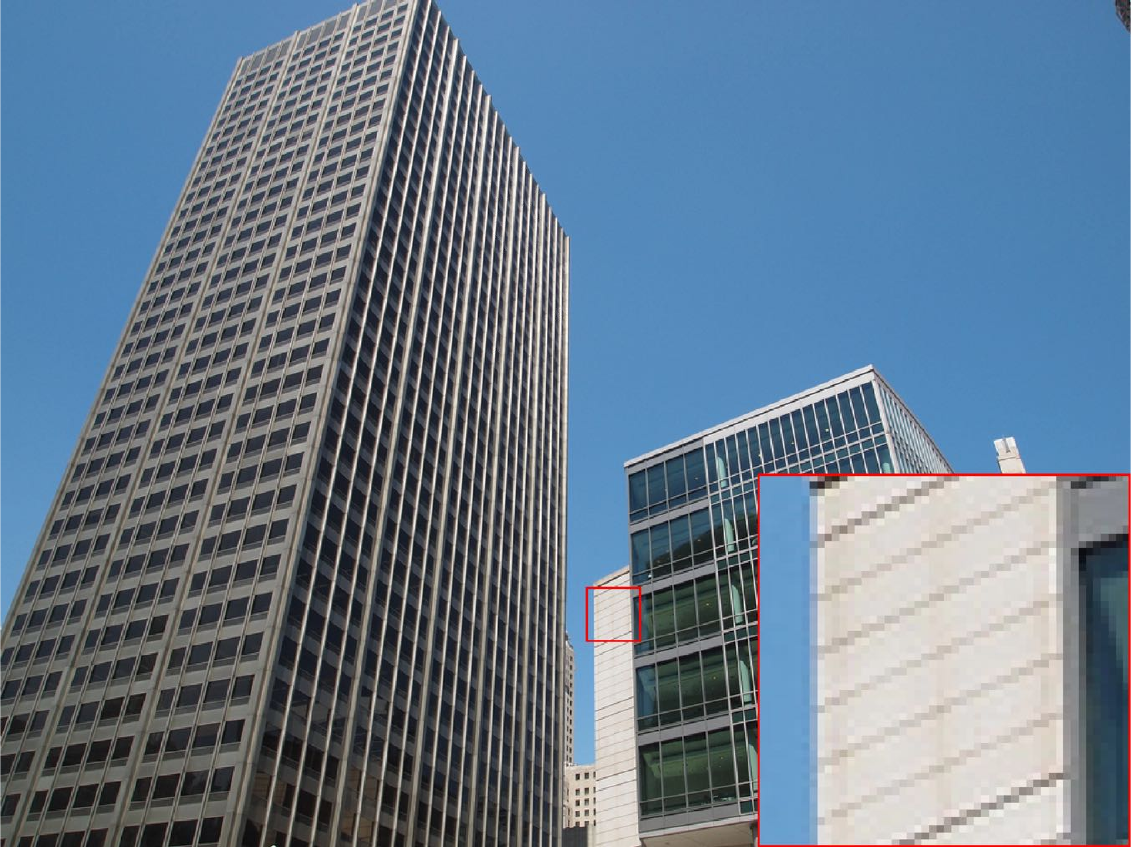}
\hspace{-0.45cm} &
\includegraphics[height=1.2in,width=1.7in,angle=0]{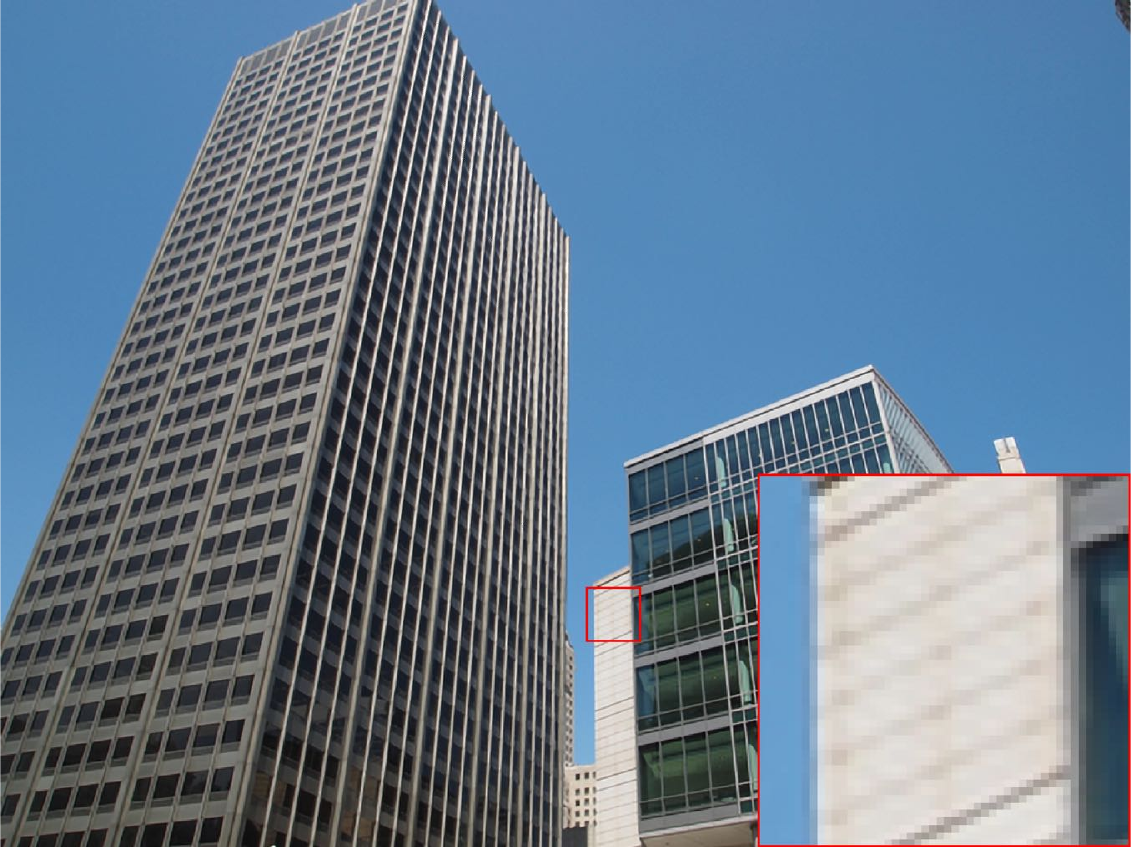} \hspace{-0.45cm} &
\includegraphics[height=1.2in,width=1.7in,angle=0]{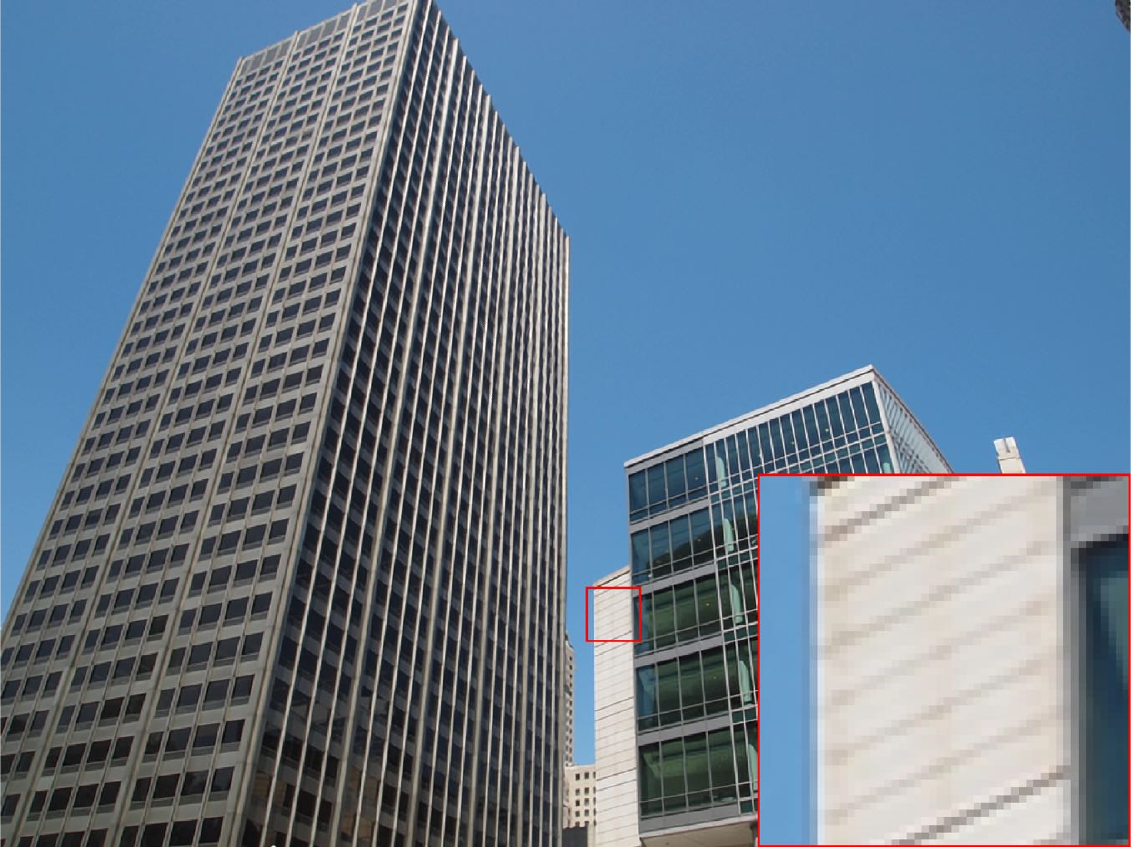}
\hspace{-0.45cm} &
\includegraphics[height=1.2in,width=1.7in,angle=0]{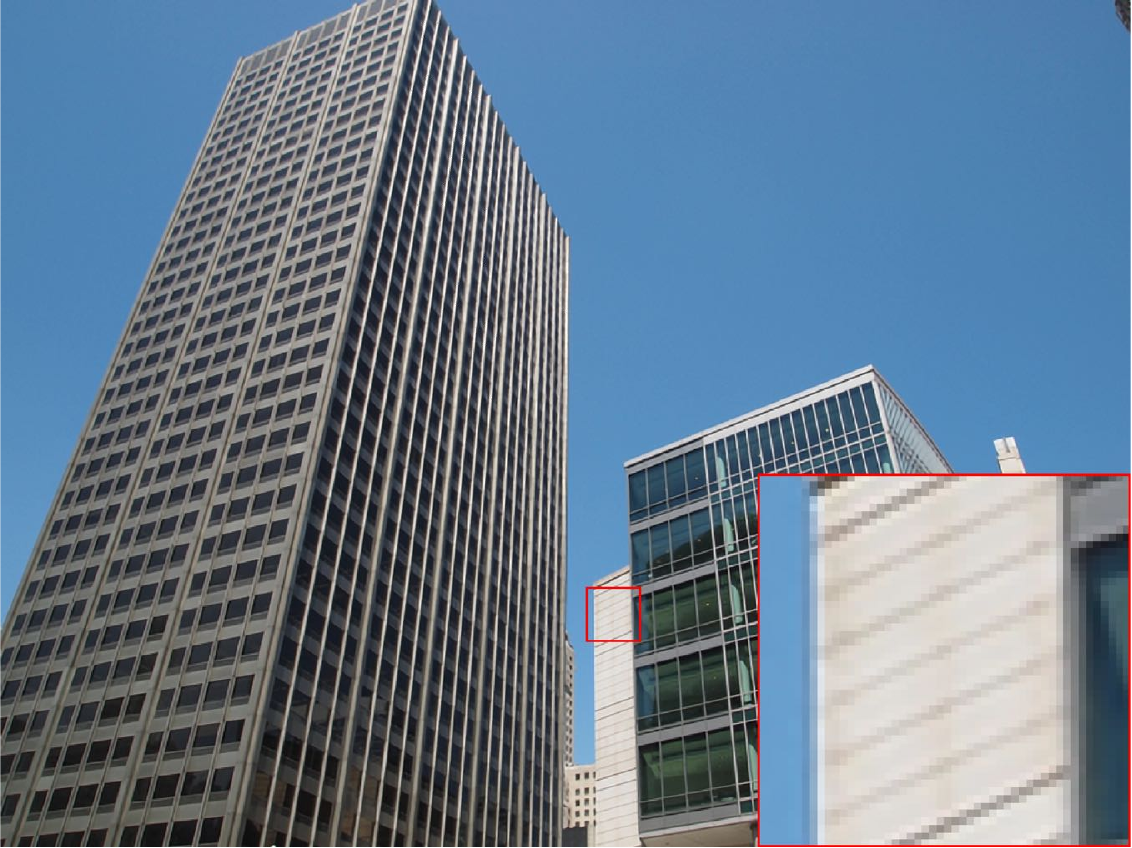}
\vspace{-0.07cm}
\\
{\sz{(PSNR/SSIM)}} & {\sz{(29.29/0.9459)}} & {\sz{(30.21/0.9569)}} & {\sz{(30.58/0.9593)}}\\
\includegraphics[height=1.2in,width=1.7in,angle=0]{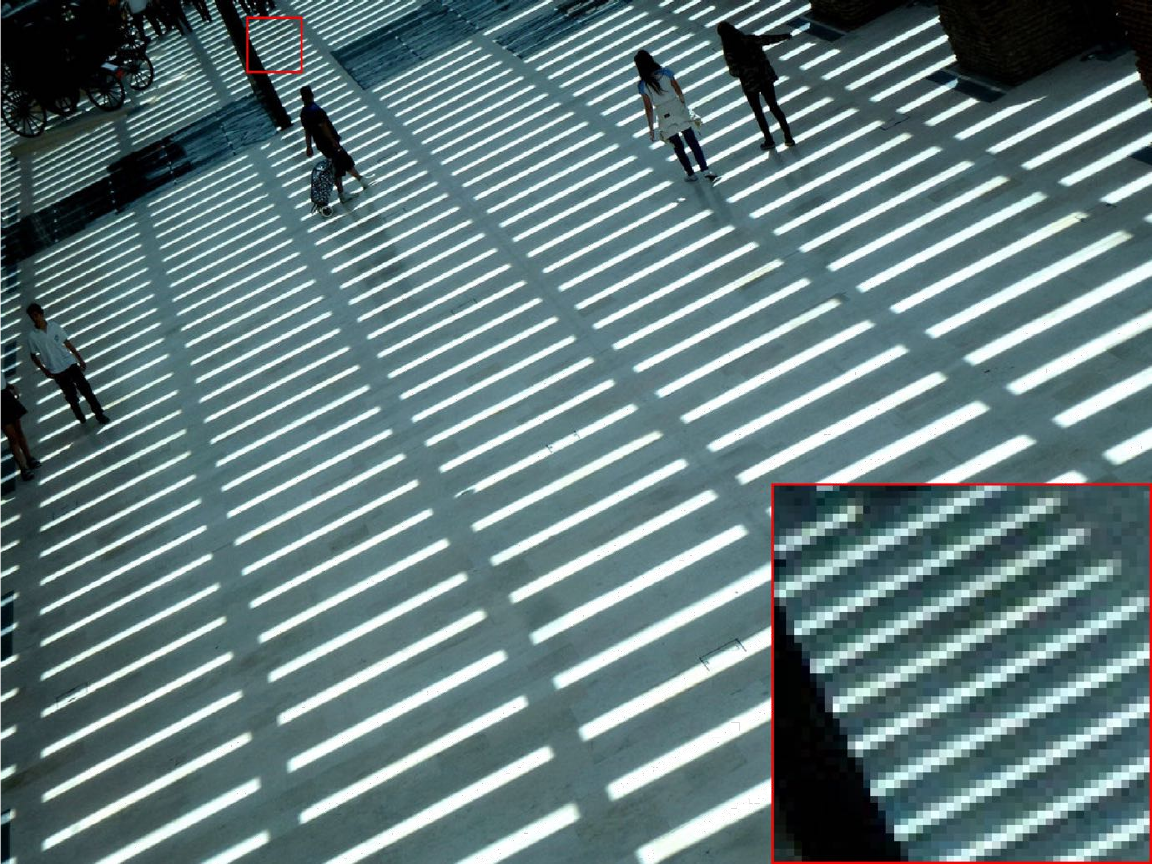}
\hspace{-0.45cm} &
\includegraphics[height=1.2in,width=1.7in,angle=0]{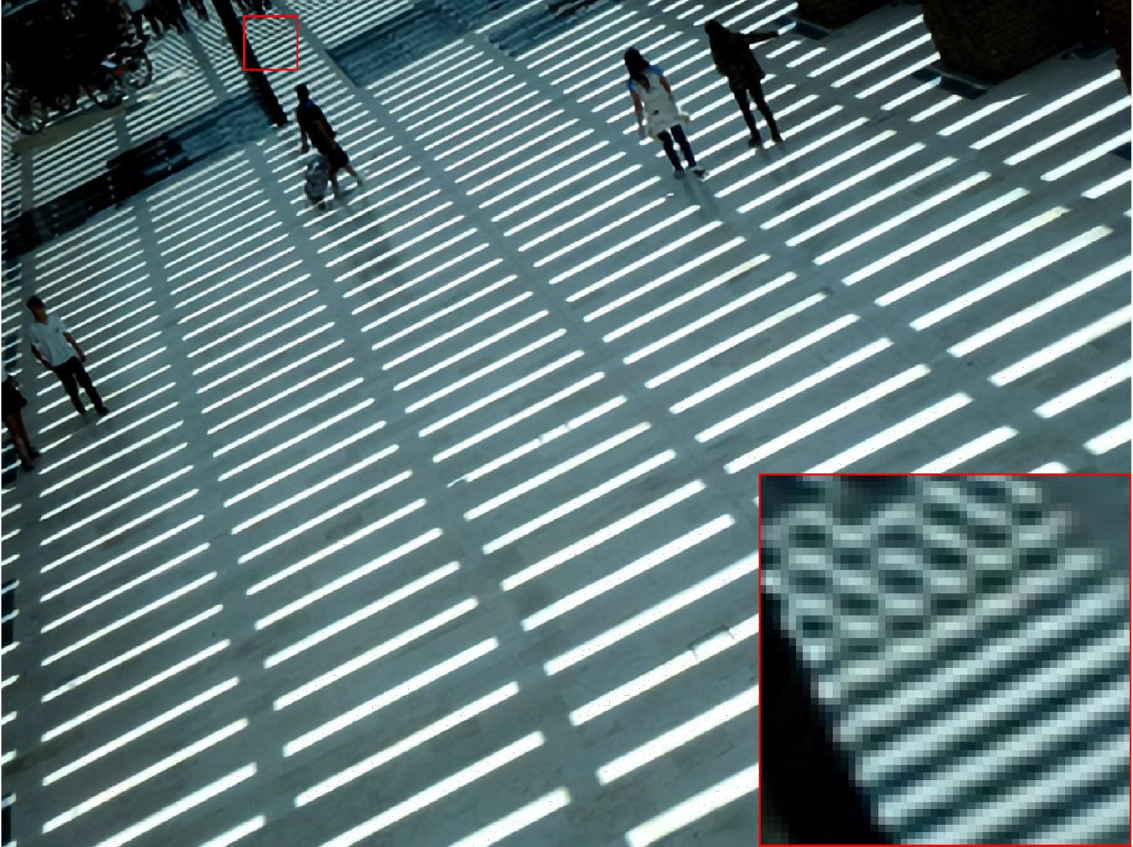} \hspace{-0.45cm} &
\includegraphics[height=1.2in,width=1.7in,angle=0]{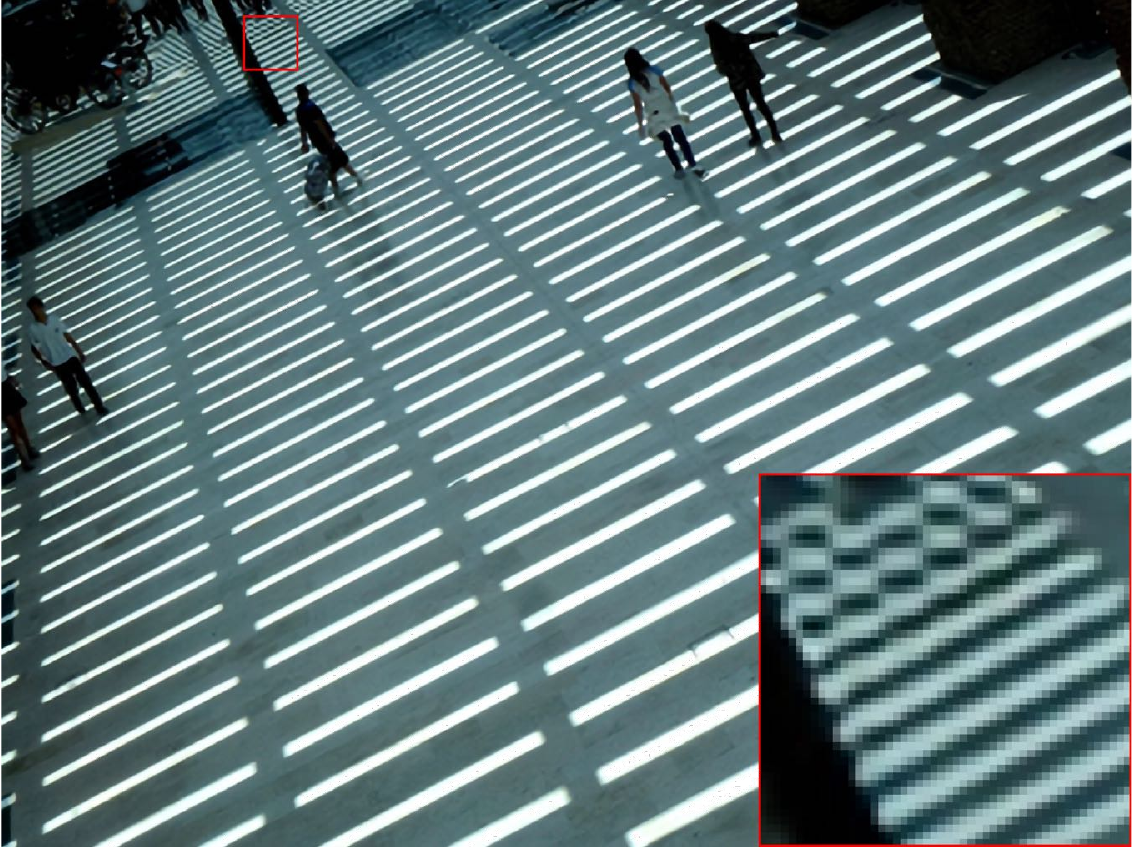}
\hspace{-0.45cm} &
\includegraphics[height=1.2in,width=1.7in,angle=0]{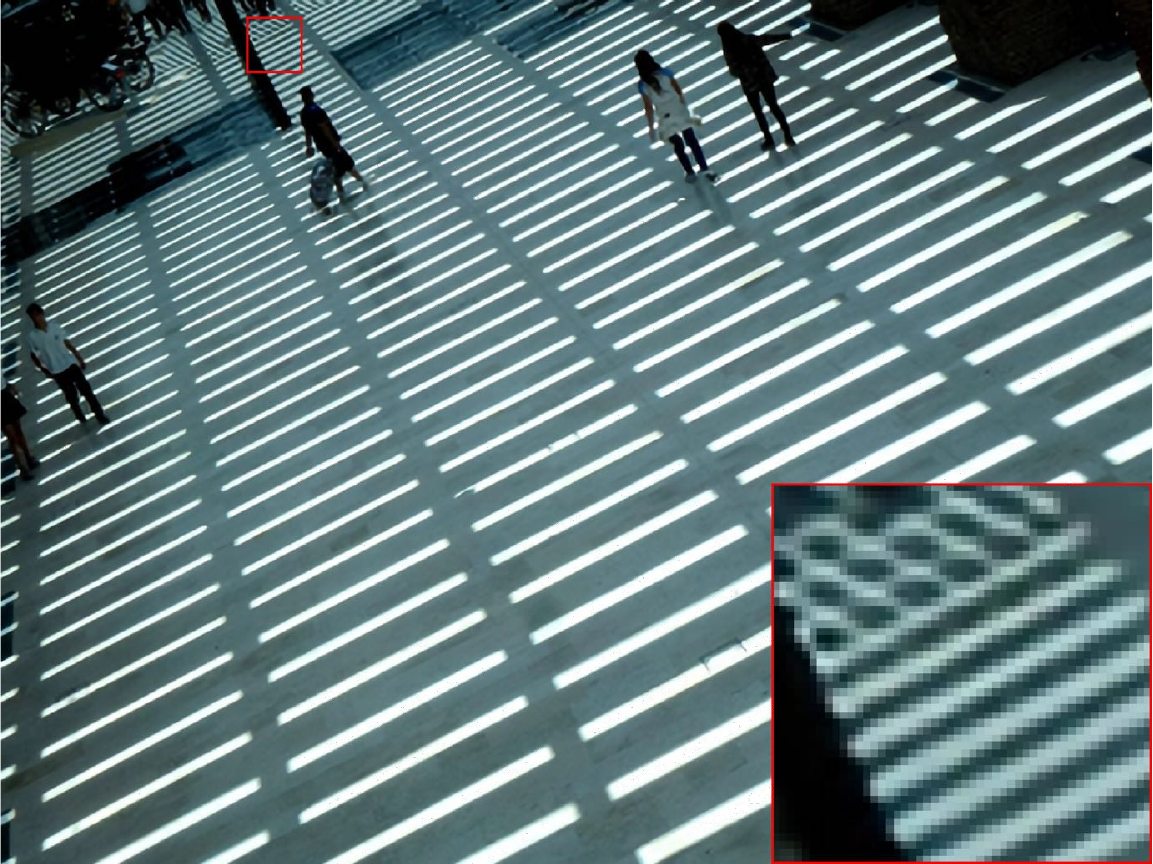}
\vspace{-0.07cm}
\\
{\sz{(PSNR/SSIM)}} & {\sz{(30.70/0.9465)}} & {\sz{(31.33/0.9505)}} & {\sz{(31.34/0.9517)}}\\
\end{tabu}
\end{center}
\caption{Visual super-resolution results of {\texttt{VDSR-f16}}, {\texttt{VDSR-f16+AIL}} and {\texttt{VDSR}}. First row: the super-resolution results for image '96' from Urban100 dataset when scaling factor is $2$. Second row: the super-resolution results for image '93' from Urban100 dataset when scaling factor is $3$.}
\label{fig:VDSR-f16}
\end{figure*}

\begin{figure*}[htp]
\setlength{\abovecaptionskip}{0pt}
\begin{center}
\begin{tabu} to 1\textwidth{ccccc}
Ground truth & VDSR-f22 & VDSR-f22+AIL & VDSR\\
\includegraphics[height=1.7in,width=1.7in,angle=0]{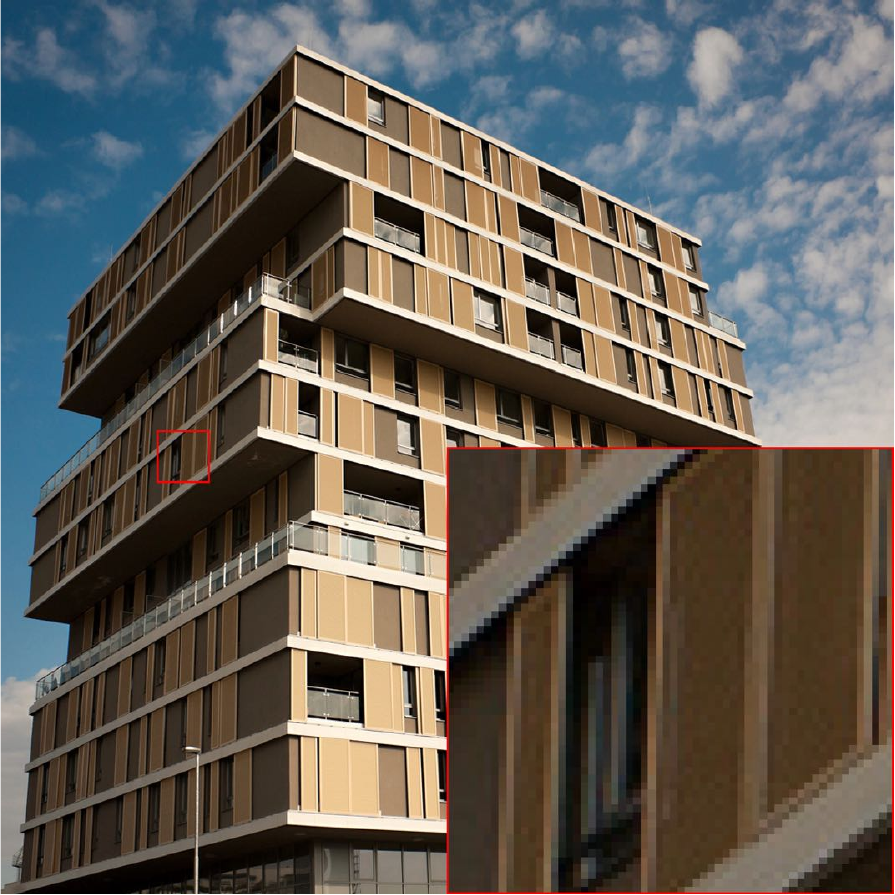}\hspace{-0.45cm} &
\includegraphics[height=1.7in,width=1.7in,angle=0]{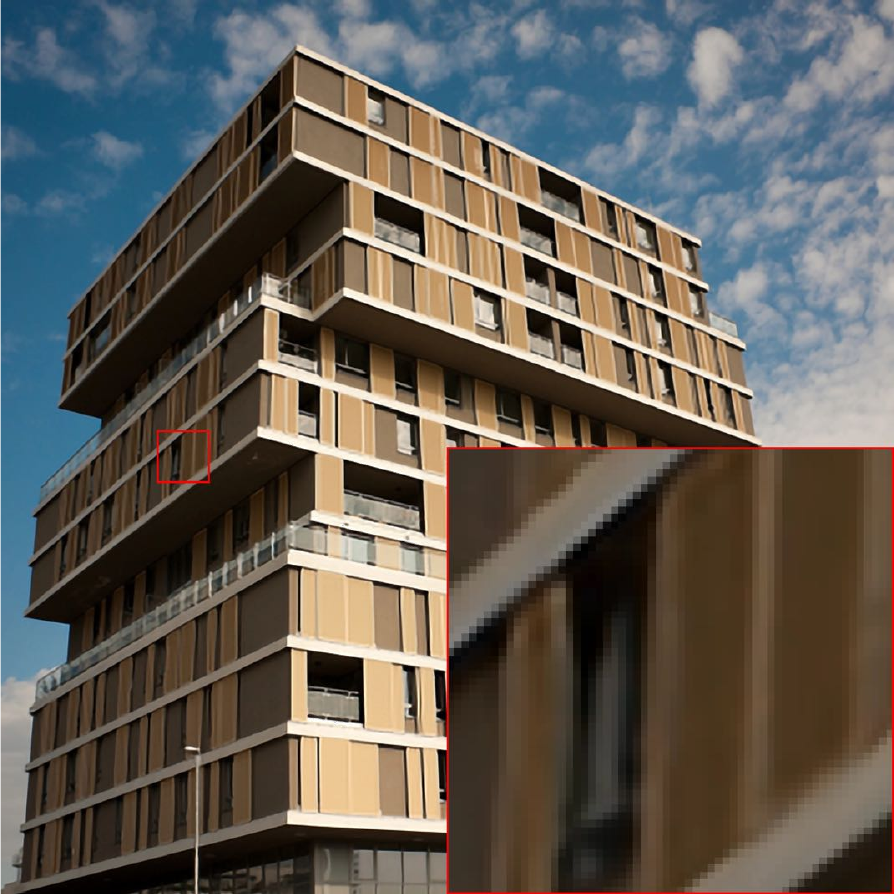} \hspace{-0.45cm} &
\includegraphics[height=1.7in,width=1.7in,angle=0]{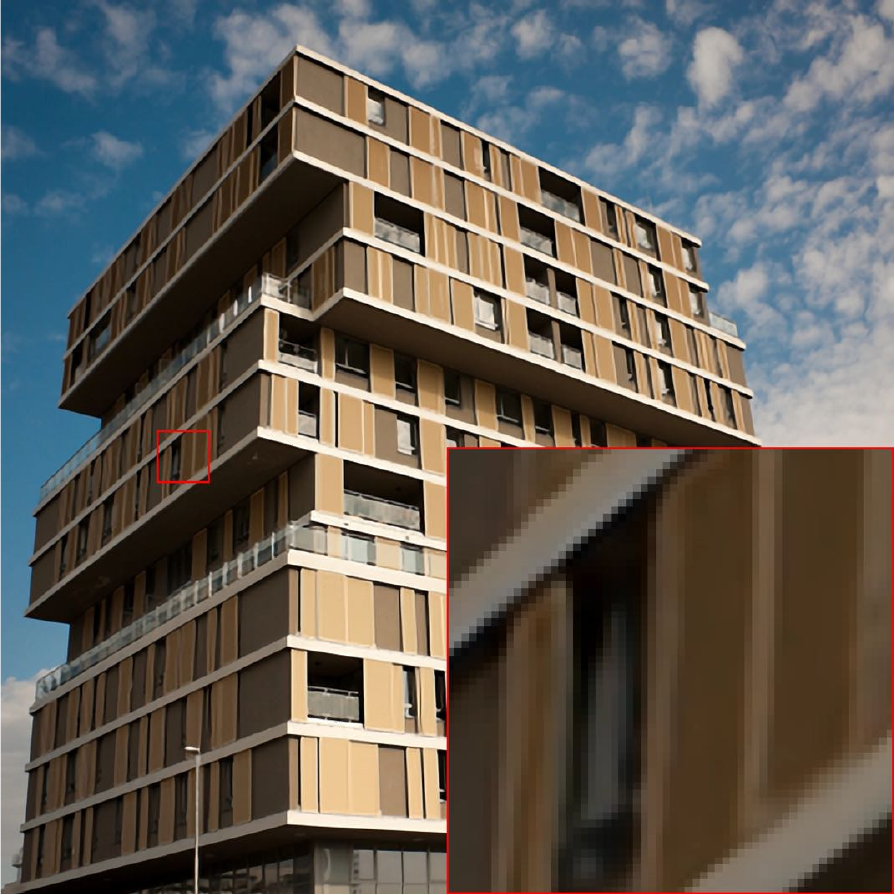}
\hspace{-0.45cm} &
\includegraphics[height=1.7in,width=1.7in,angle=0]{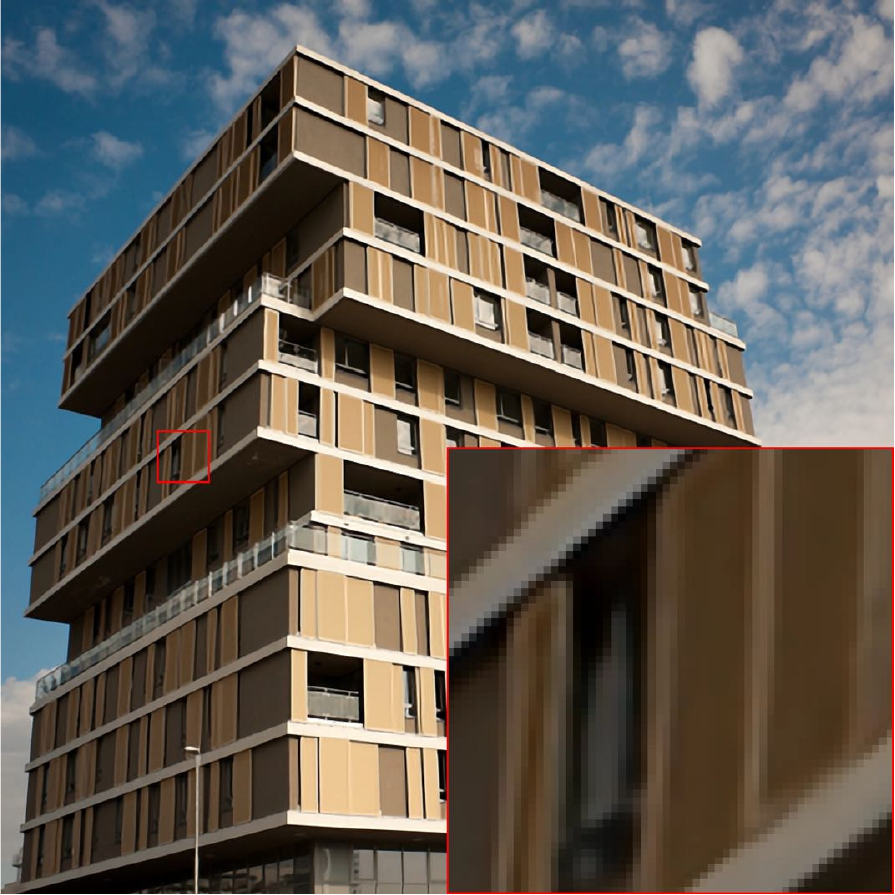}
\vspace{-0.07cm}
\\
{\sz{(PSNR/SSIM)}} & {\sz{(27.31/0.8833)}} & {\sz{(27.84/0.8950)}} & {\sz{(28.24/0.9029)}}\\
\includegraphics[height=1.7in,width=1.7in,angle=0]{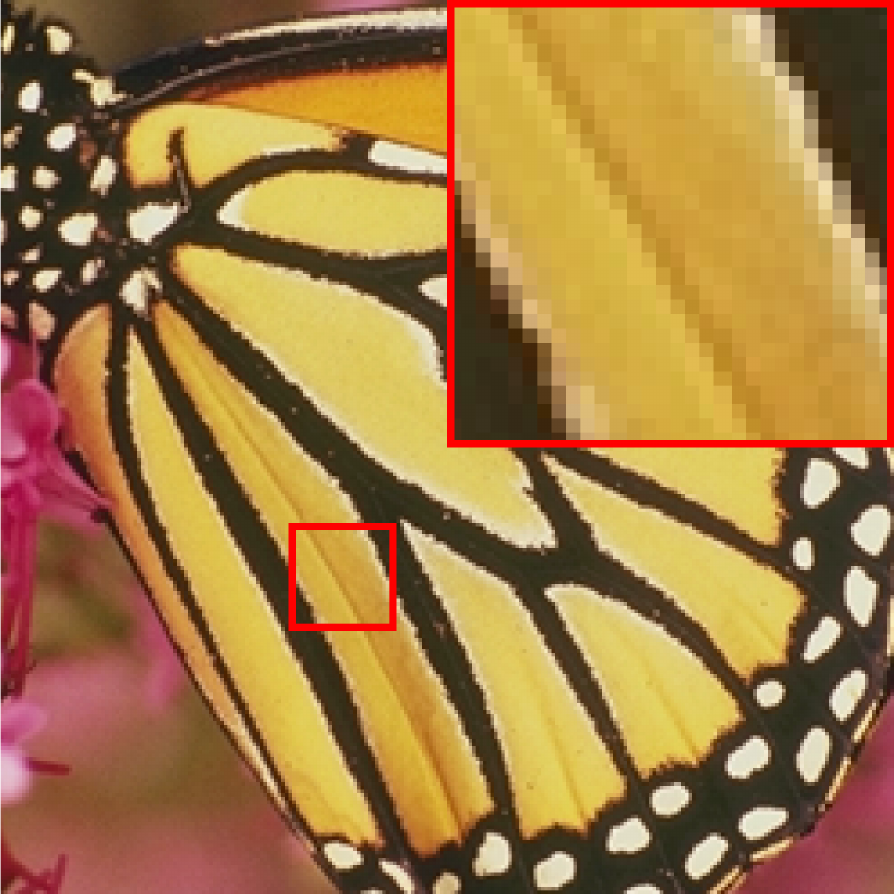}\hspace{-0.45cm} &
\includegraphics[height=1.7in,width=1.7in,angle=0]{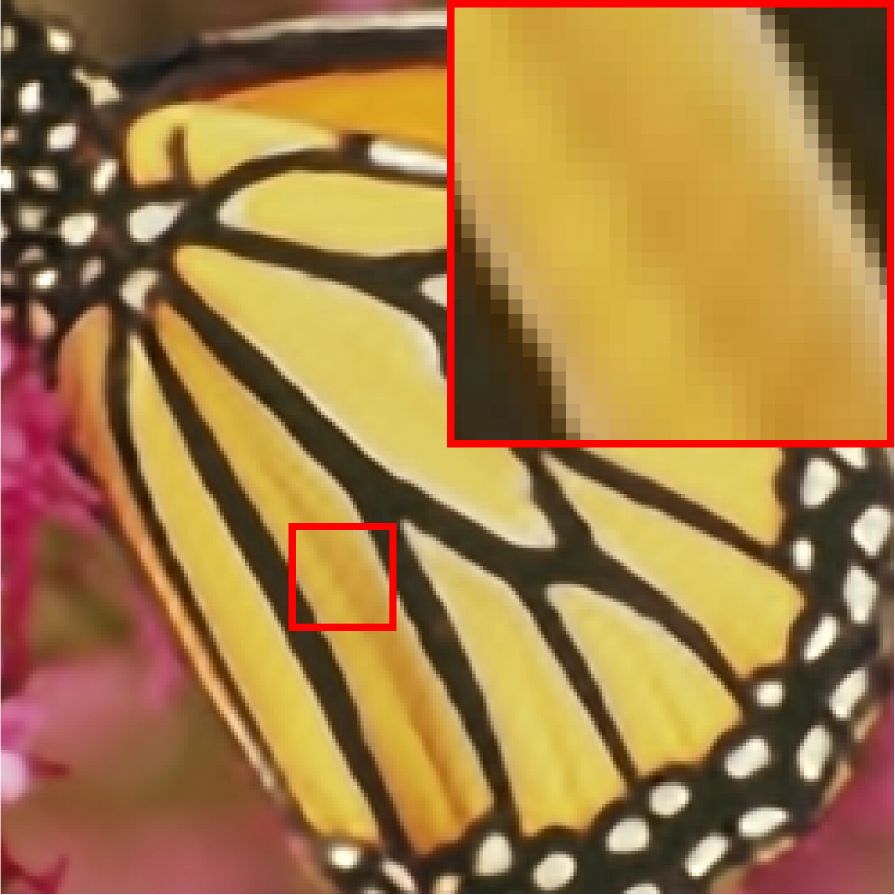} \hspace{-0.45cm} &
\includegraphics[height=1.7in,width=1.7in,angle=0]{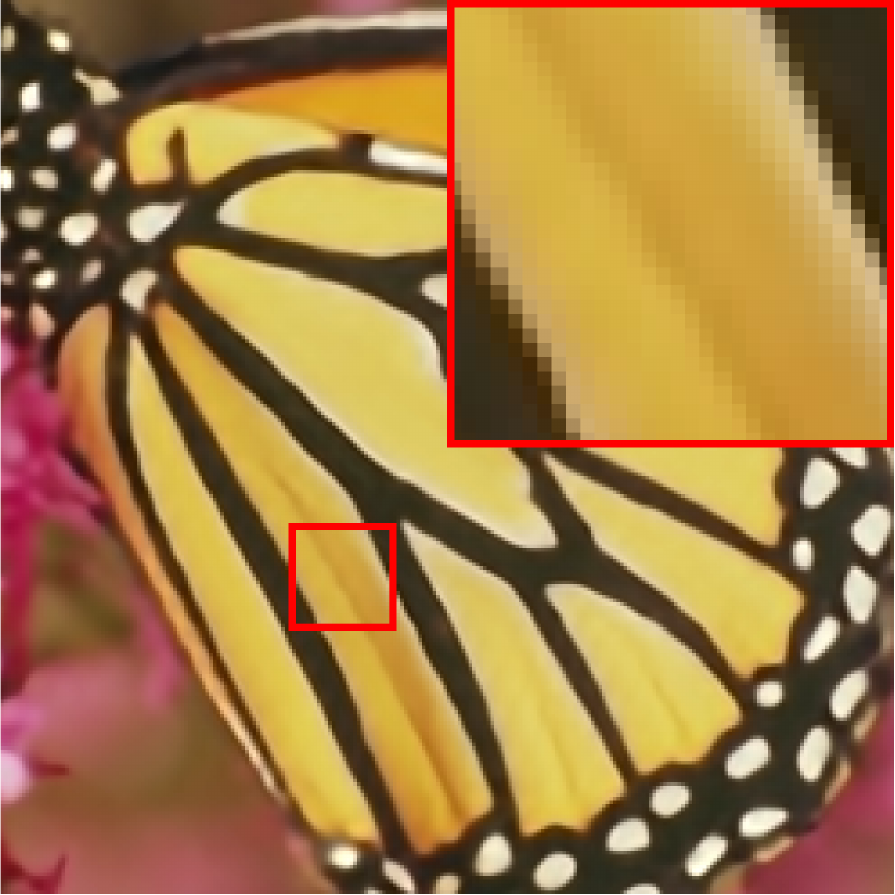}
\hspace{-0.45cm} &
\includegraphics[height=1.7in,width=1.7in,angle=0]{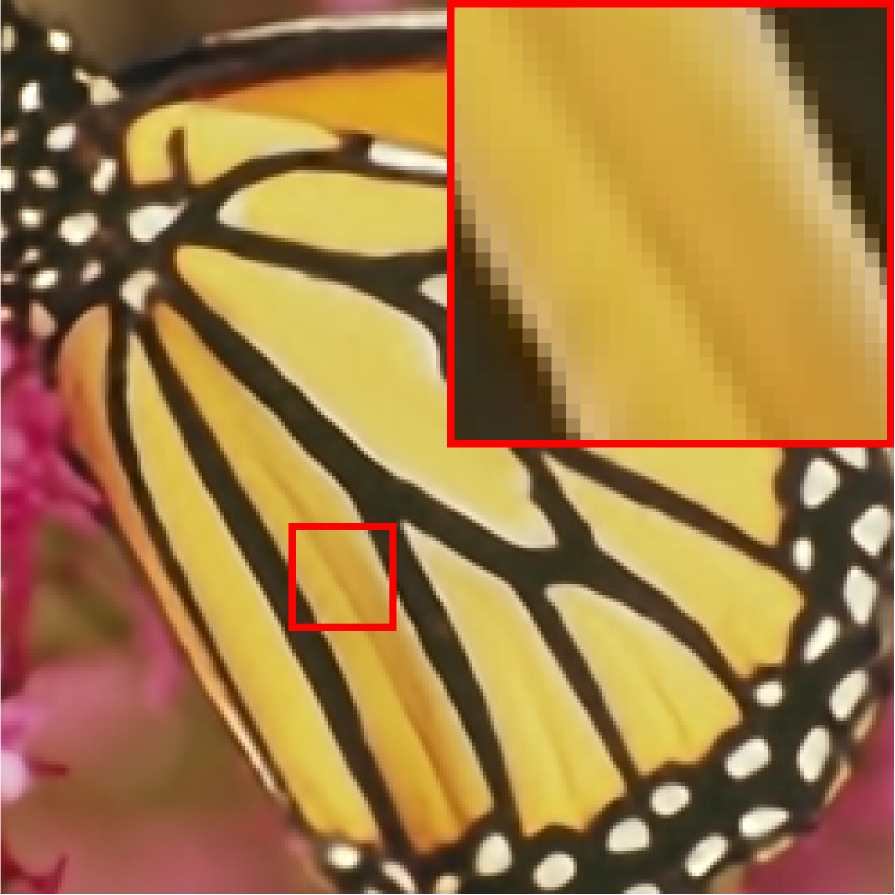}
\vspace{-0.07cm}
\\
{\sz{(PSNR/SSIM)}} & {\sz{(26.27/0.8886)}} & {\sz{(26.84/0.9019)}} & {\sz{(27.08/0.9061)}}\\
\end{tabu}
\end{center}
\caption{Visual super-resolution results of {\texttt{VDSR-f22}}, {\texttt{VDSR-f22+AIL }}and {\texttt{VDSR}}. First row: the super-resolution results for image '87' from Urban100 dataset when scaling factor is $3$. Second row: the super-resolution results for image 'butterfly' from Set5 dataset when scaling factor is $4$.}
\label{fig:VDSR-f22}
\end{figure*}

\begin{figure*}[htp]
\setlength{\abovecaptionskip}{0pt}
\begin{center}
\begin{tabu} to 1\textwidth{ccccc}
Ground truth & VDSR-f32 & VDSR-f32+AIL & VDSR\\
\includegraphics[height=1.1in,width=1.7in,angle=0]{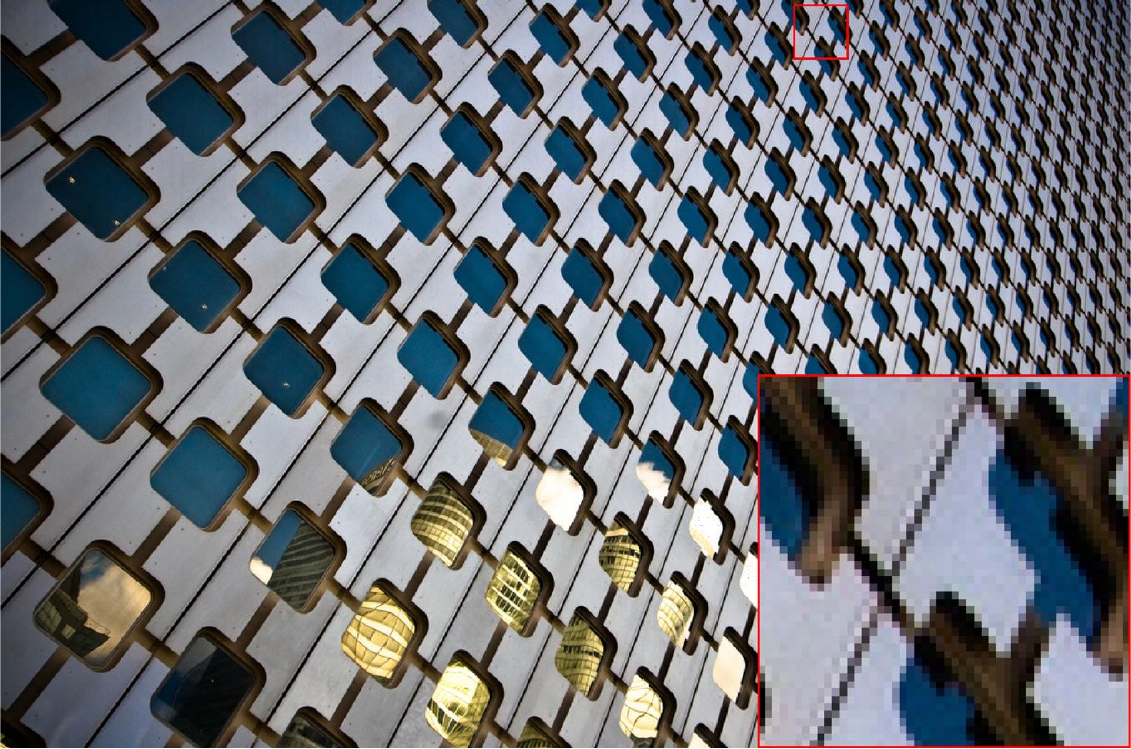}\hspace{-0.45cm} &
\includegraphics[height=1.1in,width=1.7in,angle=0]{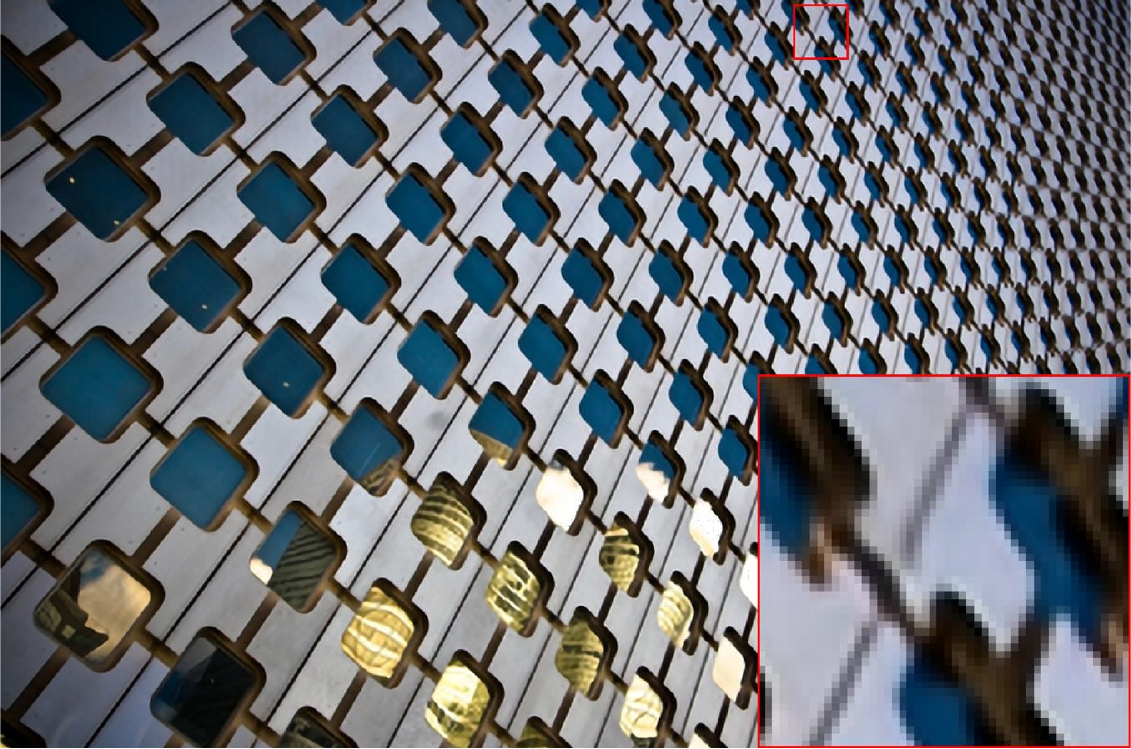} \hspace{-0.45cm} &
\includegraphics[height=1.1in,width=1.7in,angle=0]{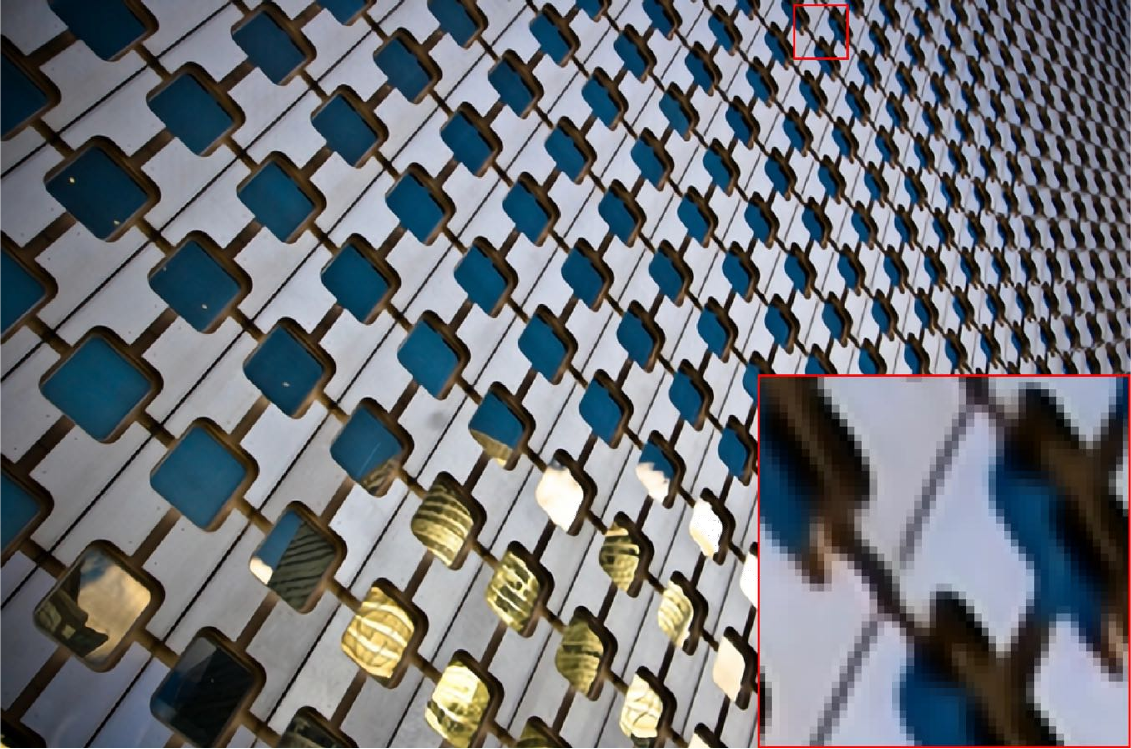}
\hspace{-0.45cm} &
\includegraphics[height=1.1in,width=1.7in,angle=0]{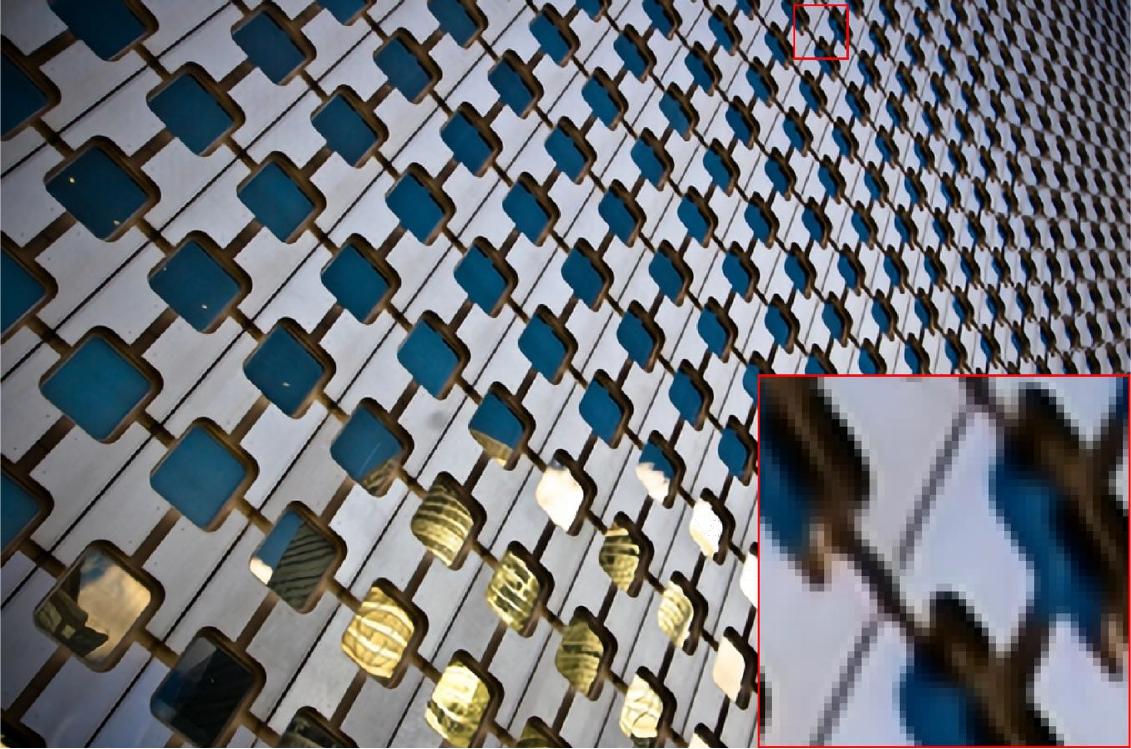}
\vspace{-0.07cm}
\\
{\sz{(PSNR/SSIM)}} & {\sz{(26.69/0.8925)}} & {\sz{(27.39/0.9058)}} & {\sz{(27.38/0.9050)}}\\
\includegraphics[height=1.2in,width=1.7in,angle=0]{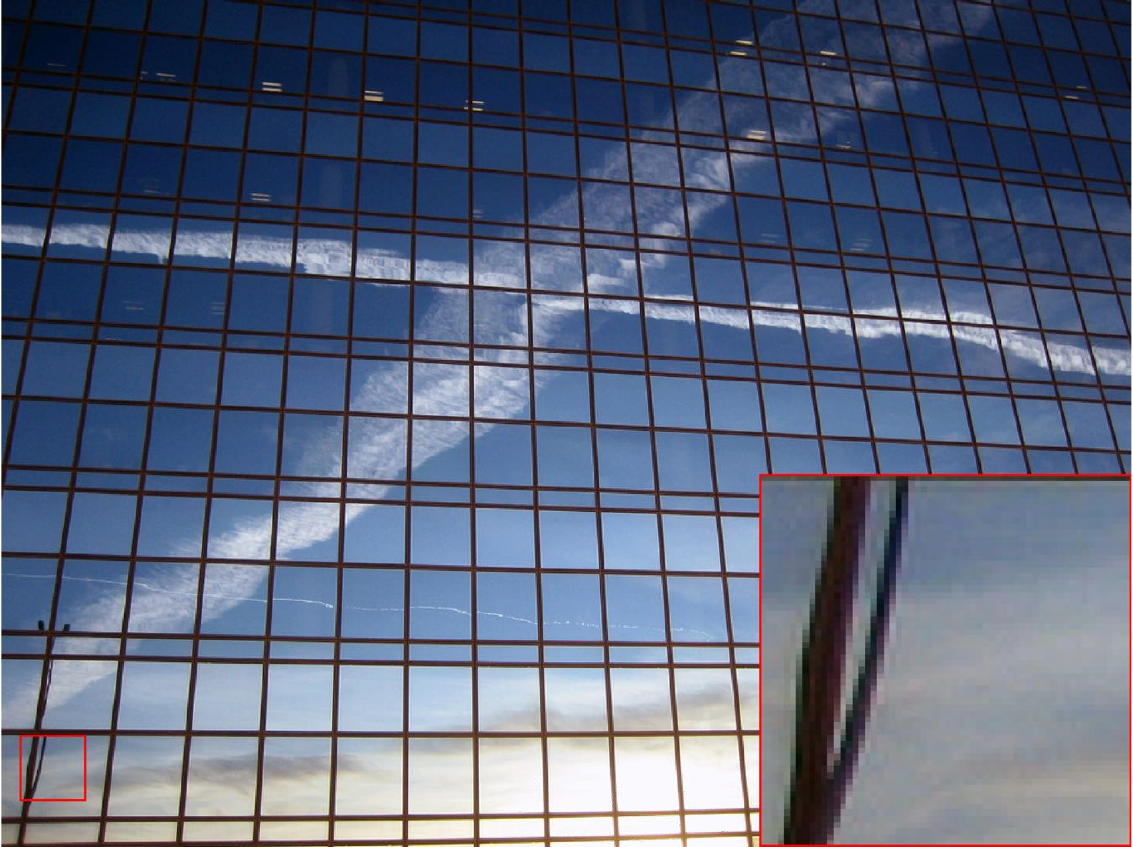}
\hspace{-0.45cm} &
\includegraphics[height=1.2in,width=1.7in,angle=0]{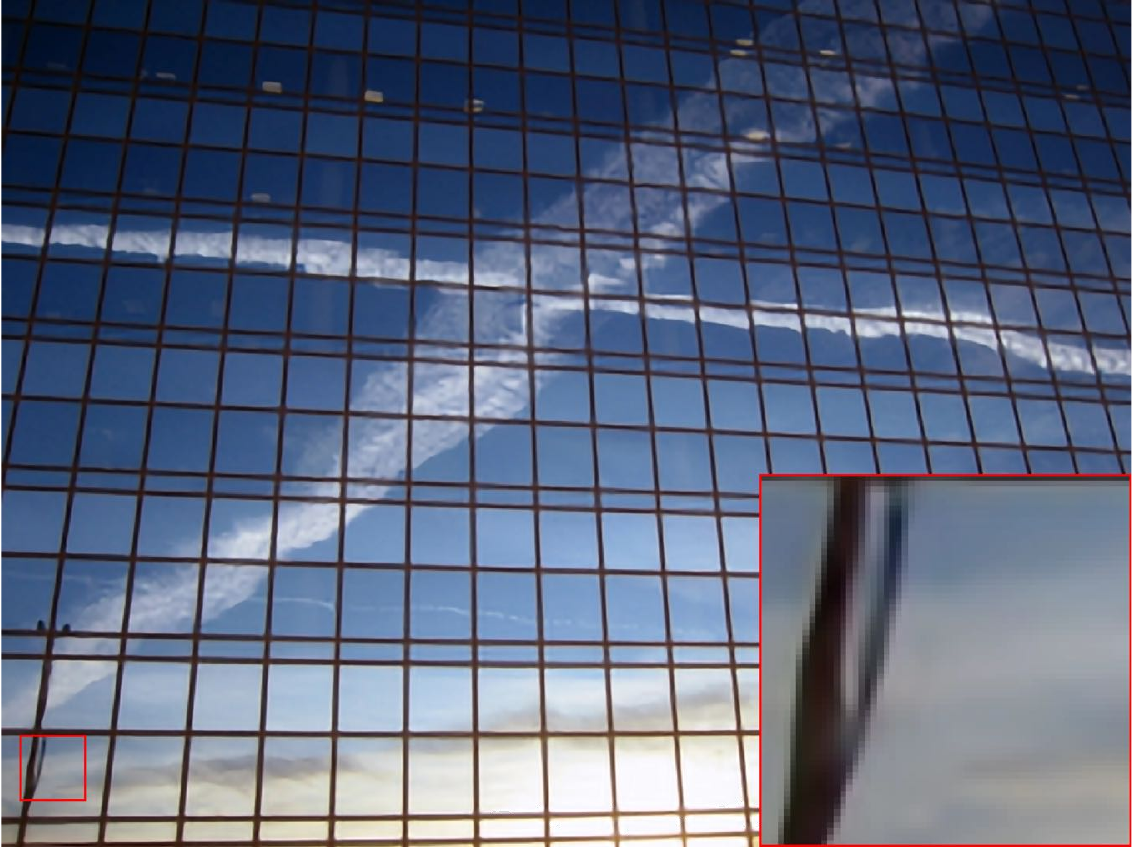} \hspace{-0.45cm} &
\includegraphics[height=1.2in,width=1.7in,angle=0]{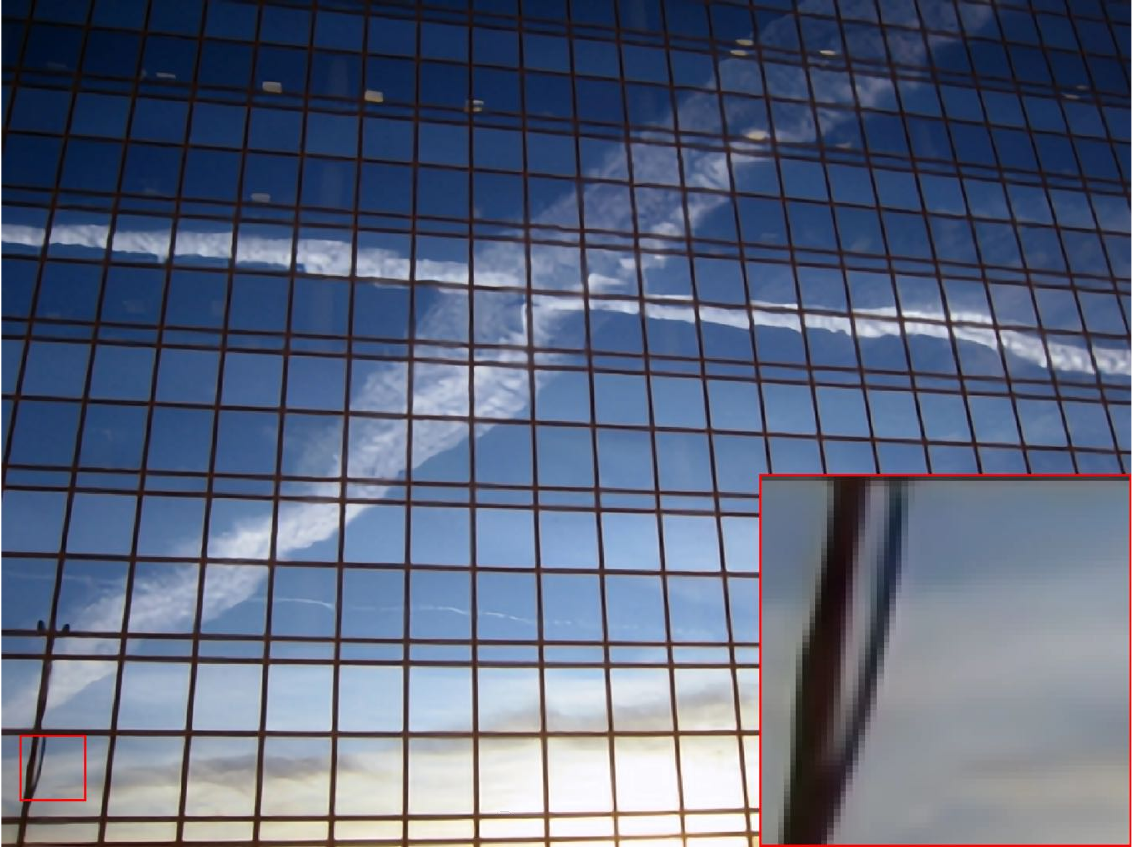}
\hspace{-0.45cm} &
\includegraphics[height=1.2in,width=1.7in,angle=0]{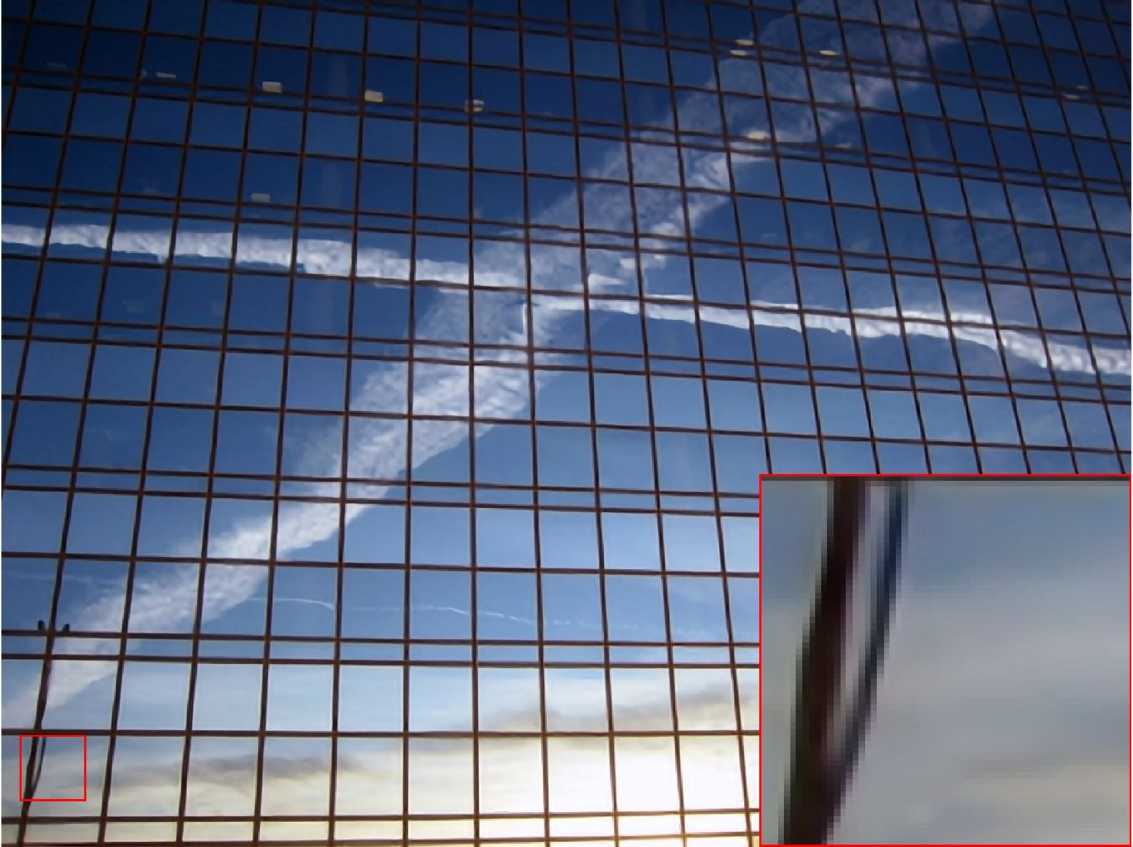}
\vspace{-0.07cm}
\\
{\sz{(PSNR/SSIM)}} & {\sz{(30.29/0.9152)}} & {\sz{(31.19/0.9305)}} & {\sz{(31.39/0.9321)}}\\
\end{tabu}
\end{center}
\caption{Visual super-resolution results of {\texttt{VDSR-f32}}, {\texttt{VDSR-f32+AIL}} and {\texttt{VDSR}}. First row: the super-resolution results for image 'ppt3' from Set14 dataset when scaling factor is $3$. Second row: the super-resolution results for image '55' from Urban100 dataset when scaling factor is $4$.}
\label{fig:VDSR-f32}
\end{figure*}

\subsection{Enhancing different scales of lightweight networks}\label{subsec:VDSR}
In this part, we employ the proposed learning scheme to enhance the capacity of different scales of lightweight networks for the given {\texttt{VDSR}} teacher network. Specifically, we implement three different scales of lightweight networks with $16$ (i.e., $\rho=0.75$), $22$ (i.e., $\rho=0.66$) and $32$ (i.e., $\rho=0.5$) features maps in each convolution layer. Similar as experiments above, we separately train each lightweight network with the traditional learning scheme in Eq.~\eqref{eq:eq1} and the proposed one in Algorithm~\ref{alg:bpl}. The resulted networks are termed with the same naming way as Section~\ref{subsec:ablation}. For example, two trained lightweight networks with $16$ feature maps are termed {\texttt{VDSR-f16}} and {\texttt{VDSR-f16+AIL}}, respectively. {\texttt{VDSR-f16}} denotes the baseline method.

Before discussing the performance of each network, we first analysis their amount of parameters as well as the computational complexity. Providing that the testing image is of size $256\times 256$, the parameters and theoretical computational complexity of these lightweight networks as well as the teacher network {\texttt{VDSR}} are given in Table~\ref{table:AIL_Complexity}. For example, the amount of parameters as well as the computational complexity of {\texttt{VDSR-f32}} and {\texttt{VDSR-f32+AIL}} are only $25\%$ of that for {\texttt{VDSR}}.

Under the same experimental settings, the quantitative results of all networks on four test datasets are provided in Table~\ref{table:AIL_f16}, Table~\ref{table:AIL_f22} and Table~\ref{table:AIL_f32}. It can be found that the proposed adaptive importance learning scheme enhances the performance of lightweight networks obviously. For example, in Table~\ref{table:AIL_f16}, when the scaling factor is $2$ on the Set5 dataset, the superiority of {\texttt{VDSR-f16+AIL}} over {\texttt{VDSR-f16}} in PSNR and SSIM is up to $0.28$db and $0.0013$, respectively. Moreover, the superiority of {\texttt{VDSR-f32+AIL}} is more obvious on the more challenging dataset. For example, when the scaling factor is $2$ on the Urban100 dataset, the superiority of {\texttt{VDSR-f16+AIL}} over {\texttt{VDSR-f16}} in PSNR and SSIM is even up to $0.47$db and $0.0064$, respectively. In addition, we find that the proposed learning scheme performs the best on scaling factor $2$ among three scaling factors. For example, as shown in Table~\ref{table:AIL_f22} and Table~\ref{table:AIL_f32}, {\texttt{VDSR-f22+AIL}} produces comparable results on four test datasets to that of {\texttt{VDSR}}, and {\texttt{VDSR-f32+AIL}} even outperforms {\texttt{VDSR}}, especially on the Urban100 dataset on which the superiority is up to $0.19$db in PSNR. The reason is intuitive. Compared with other two scaling factors, the SISR task on scaling factor $2$ is relatively easier and contains many pixels that cannot be well reconstructed by the baseline network (e.g., {\texttt{VDSR-f16}}) but may be well reconstructed when the capacity of the lightweight network is maximized. Thus, with the easy-to-complex learning paradigm, the proposed scheme is able to improve the performance more obviously. In contrast, the SISR task on other two scaling factors contains extensive complex pixels beyond the maximum capacity of the network, which cannot be well reconstructed even with the easy-to-hard learning paradigm. According to these results, we can conclude that the proposed adaptive importance learning scheme is able to enhance the performance of different scales of lightweight networks in SISR. More evidence in visual results can be found in Figure~\ref{fig:VDSR-f16},~\ref{fig:VDSR-f22} and~\ref{fig:VDSR-f32}. 

\begin{table*}\footnotesize%
\caption{Average PSNR/SSIM of {\texttt{DRRN-f25}}, {\texttt{DRRN-f25+AIL}} and {\texttt{DRRN}} on four test datasets. {\blue{$\uparrow$PSNR/SSIM}} and {{$\downarrow$PSNR/SSIM}} denote the performance increase and decrease over {\texttt{DRRN-f25}}, respectively.}
\renewcommand{\arraystretch}{1.2}
\begin{center}
\begin{tabular}{l|c|c|cc|c}
\hline
Dataset & scale & DRRN-f25 & \multicolumn{2}{c|}{DRRN-f25+AIL} & DRRN\\
\hline
\multirow{3}{*}{Set5} & $\times$2 & 37.02/0.9575 & 37.52/0.9593 & {\blue{$\uparrow$0.50}}/{\blue{$\uparrow$0.0018}} & 37.69/0.9602\\
& $\times$3 & 33.26/0.9181 & 33.64/0.9214 & {\blue{$\uparrow$0.37}}/{\blue{$\uparrow$0.0033}} & 34.02/0.9257\\
& $\times$4 & 30.92/0.8740 & 31.28/0.8832 & {\blue{$\uparrow$0.36}}/{\blue{$\uparrow$0.0092}} & 31.69/0.8899\\
\hline
\multirow{3}{*}{Set14} & $\times$2 & 32.80/0.9110 & 33.11/0.9135 & {\blue{$\uparrow$0.31}}/{\blue{$\uparrow$0.0025}} & 33.31/0.9152\\
& $\times$3 & 29.64/0.8290 & 29.83/0.8325 & {\blue{$\uparrow$0.19}}/{\blue{$\uparrow$0.0035}} & 30.05/0.8369\\
& $\times$4 & 27.82/0.7606 & 28.05/0.7682 & {\blue{$\uparrow$0.23}}/{\blue{$\uparrow$0.0076}} & 28.35/0.7752\\
\hline
\multirow{3}{*}{BSD100} & $\times$2 & 31.58/0.8926 & 31.84/0.8955 & {\blue{$\uparrow$0.26}}/{\blue{$\uparrow$0.0030}} & 32.04/0.8984\\
& $\times$3 & 28.60/0.7929 & 28.76/0.7960 & {\blue{$\uparrow$0.16}}/{\blue{$\uparrow$0.0030}} & 28.96/0.8015\\
& $\times$4 & 27.03/0.7167 & 27.21/0.7234 & {\blue{$\uparrow$0.18}}/{\blue{$\uparrow$0.0067}} & 27.42/0.7299\\
\hline
\multirow{3}{*}{Urban100} & $\times$2 & 29.95/0.9048 & 30.58/0.9123 & {\blue{$\uparrow$0.63}}/{\blue{$\uparrow$0.0075}} & 31.22/0.9195\\
& $\times$3 & 26.62/0.8140 & 26.96/0.8233 & {\blue{$\uparrow$0.34}}/{\blue{$\uparrow$0.0093}} & 27.57/0.8390\\
& $\times$4 & 24.72/0.7331 & 25.04/0.7474 & {\blue{$\uparrow$0.32}}/{\blue{$\uparrow$0.0143}} & 25.57/0.7668\\
\hline
\end{tabular}
\end{center}
\label{table:AIL_DRRNf25}
\end{table*}

\begin{figure*}[htp]
\setlength{\abovecaptionskip}{0pt}
\begin{center}
\begin{tabu} to 1\textwidth{ccccc}
\renewcommand{\arraystretch}{1.2}
Ground truth & DRRN-f25 & DRRN-f25+AIL & DRRN\\
\includegraphics[height=1.1in,width=1.7in,angle=0]{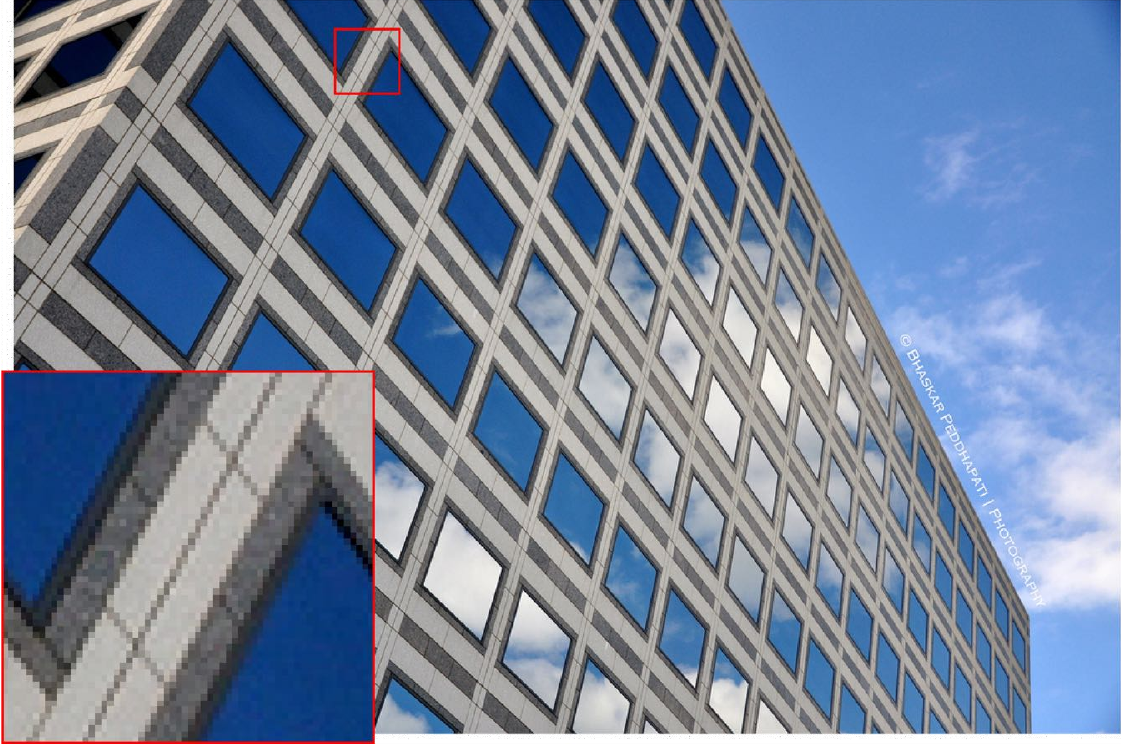}\hspace{-0.45cm} &
\includegraphics[height=1.1in,width=1.7in,angle=0]{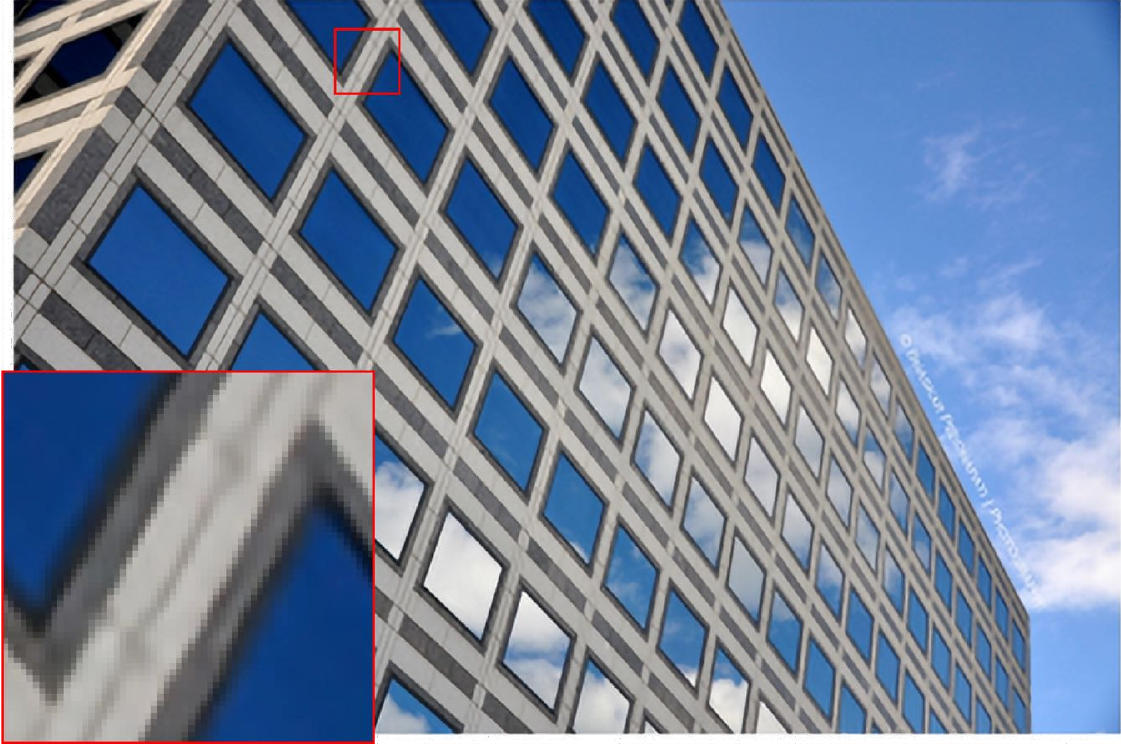} \hspace{-0.45cm} &
\includegraphics[height=1.1in,width=1.7in,angle=0]{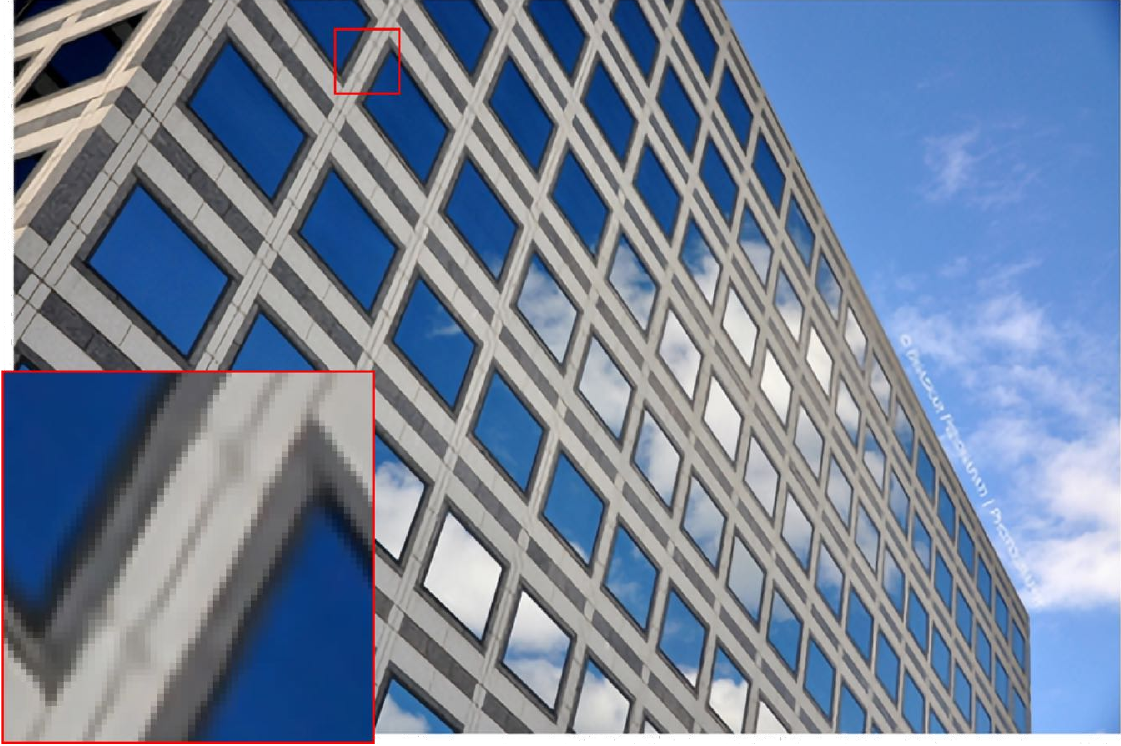}
\hspace{-0.45cm} &
\includegraphics[height=1.1in,width=1.7in,angle=0]{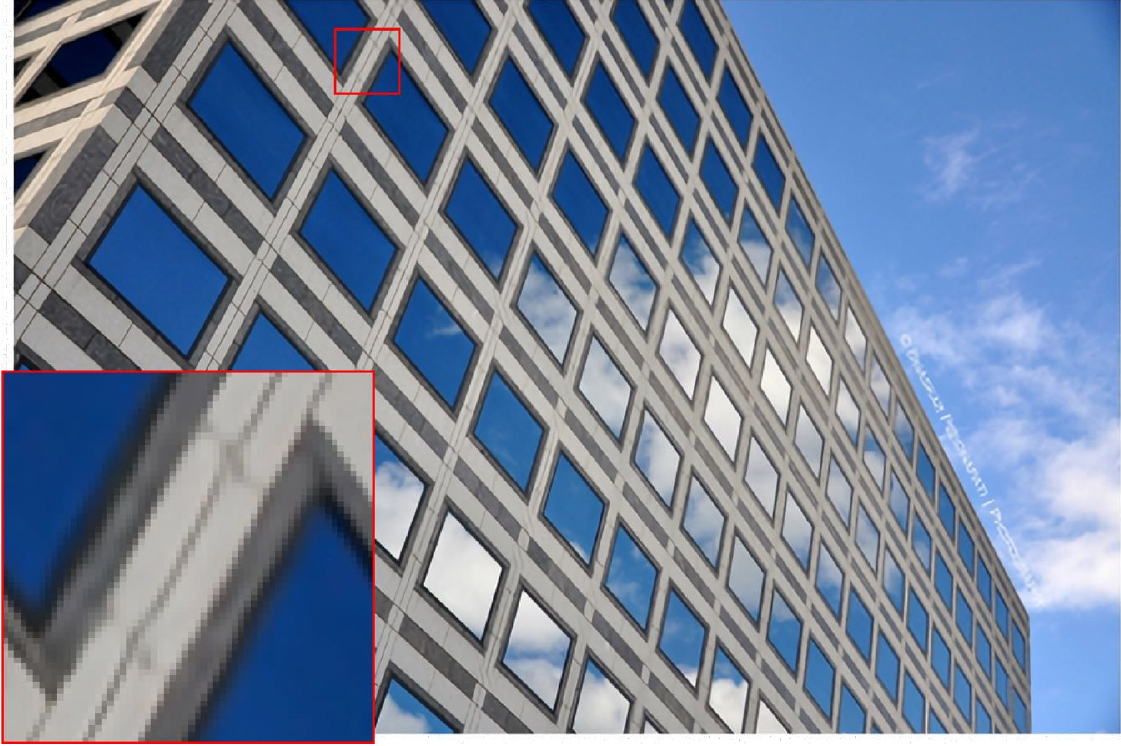}
\vspace{-0.07cm}
\\
{\sz{(PSNR/SSIM)}} & {\sz{(29.31/0.8816)}} & {\sz{(30.05/0.8958)}} & {\sz{(31.18/0.9136)}}\\
\includegraphics[height=1in,width=1.7in,angle=0]{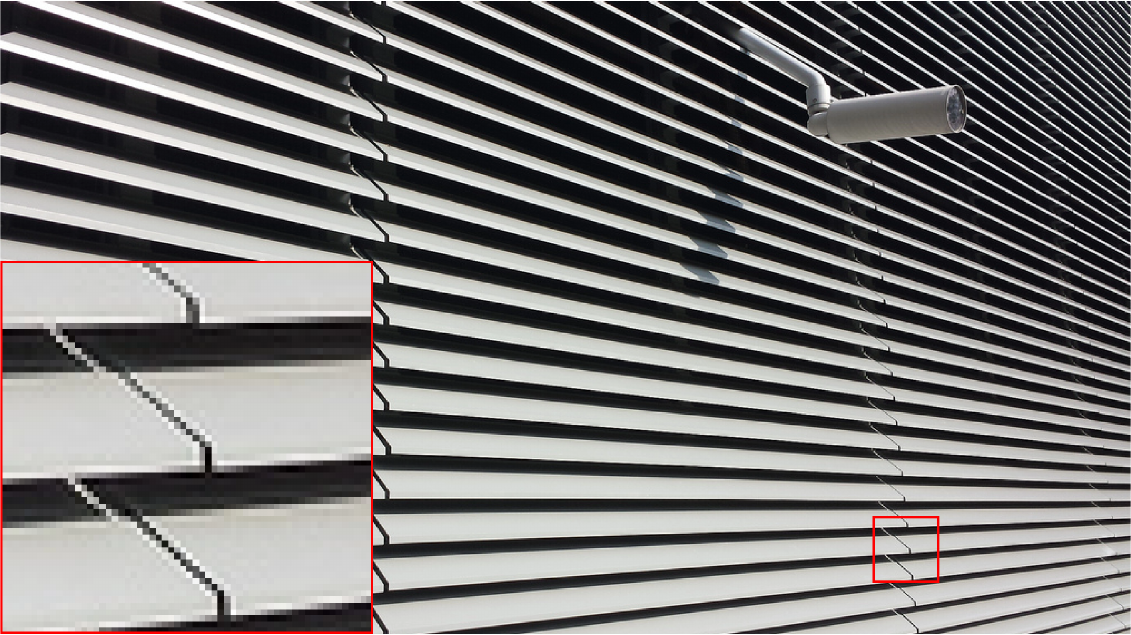}
\hspace{-0.45cm} &
\includegraphics[height=1in,width=1.7in,angle=0]{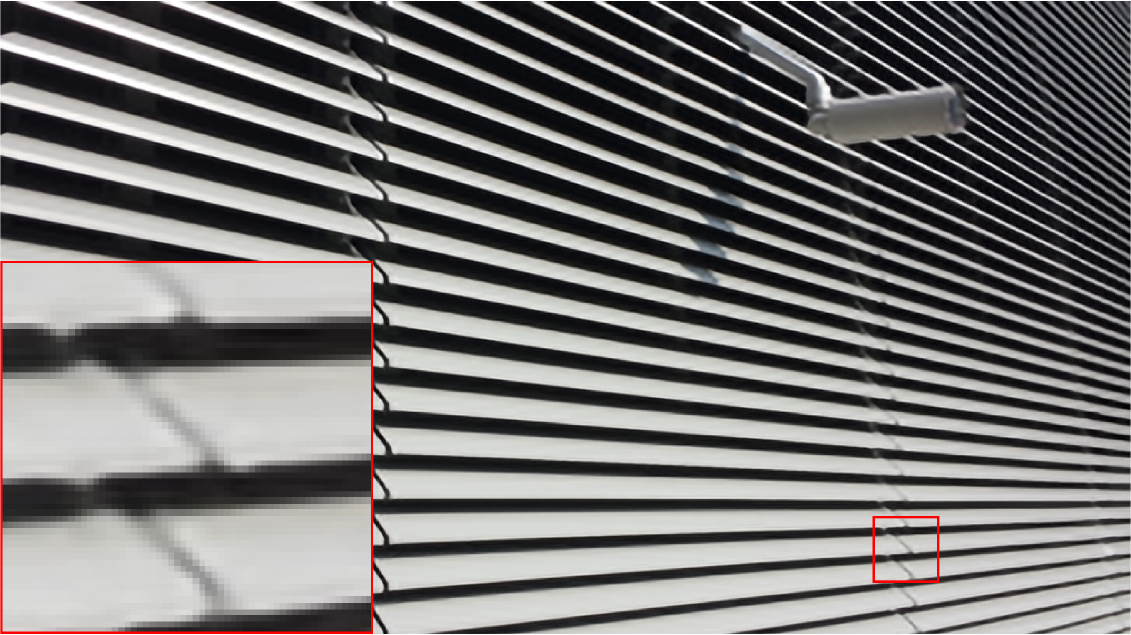} \hspace{-0.45cm} &
\includegraphics[height=1in,width=1.7in,angle=0]{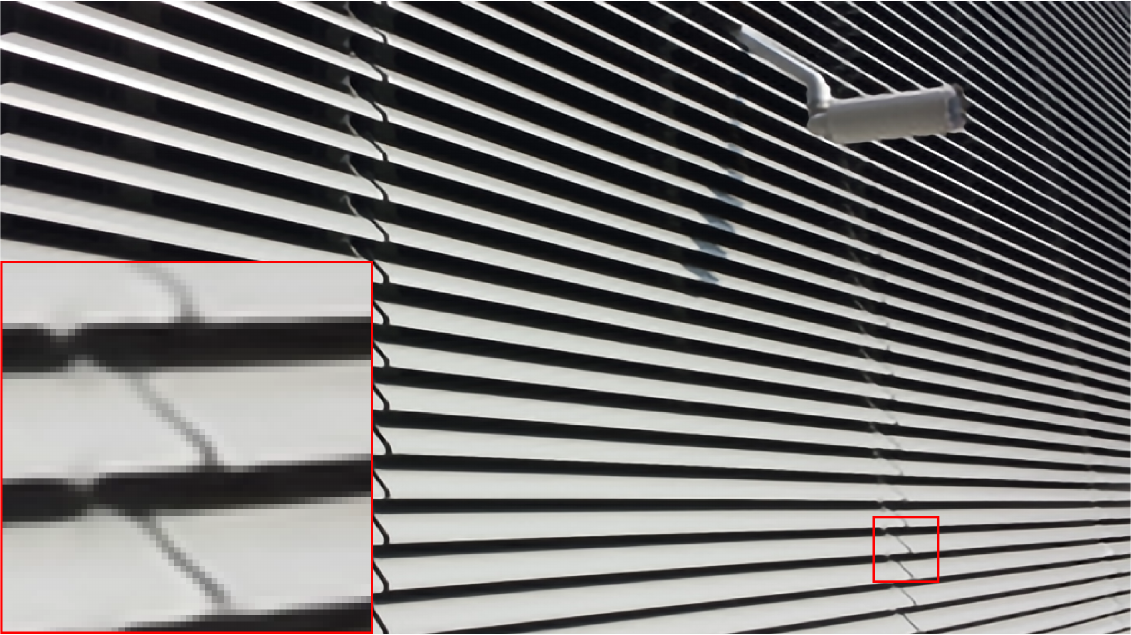}
\hspace{-0.45cm} &
\includegraphics[height=1in,width=1.7in,angle=0]{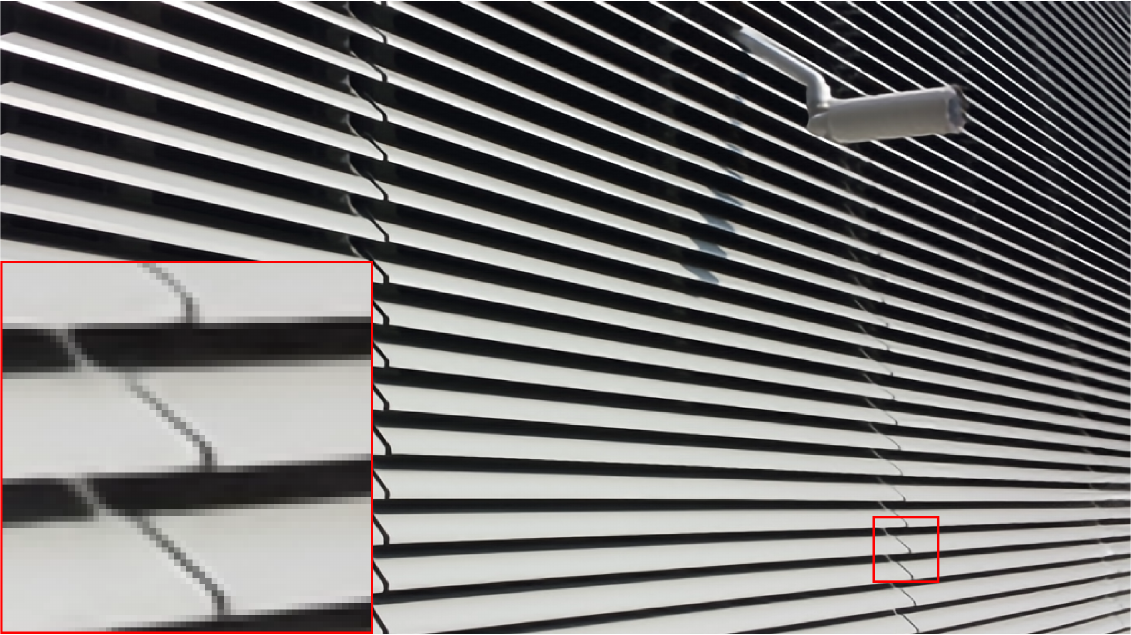}
\vspace{-0.07cm}
\\
{\sz{(PSNR/SSIM)}} & {\sz{(24.59/0.9023)}} & {\sz{(26.10/0.9308)}} & {\sz{(26.99/0.9414)}}\\
\end{tabu}
\end{center}
\caption{Visual super-resolution results of {\texttt{DRRN-f25}}, {\texttt{DRRN-f25+AIL}} and {\texttt{DRRN}}. First row: the super-resolution results for image '35' from Urban100 dataset when scaling factor is $3$. Second row: the super-resolution results for image '40' from Urban100 dataset when scaling factor is $4$.}
\label{fig:DRRN-f25}
\end{figure*}

\begin{table}\footnotesize%
\caption{Average PSNR/SSIM of {\texttt{FSRCNN}} and {\texttt{FSRCNN+AIL}} on four test datasets. {\blue{$\uparrow$PSNR/SSIM}} and {{$\downarrow$PSNR/SSIM}} denote the performance increase and decrease over {\texttt{FSRCNN}}, respectively.}
\renewcommand{\arraystretch}{1.2}
\begin{center}
\begin{tabular}{l|c|c|cc}
\hline
Dataset & scale & FSRCNN & \multicolumn{2}{c}{FSRCNN+AIL}\\
\hline
\multirow{3}{*}{Set5} & $\times$2 & 37.01/0.9570 & {\textbf{37.41/0.9587}} & {\blue{$\uparrow$0.40}}/{\blue{$\uparrow$0.0017}}\\
& $\times$3 & 33.01/0.9143 & {\textbf{33.33/0.9188}} & {\blue{$\uparrow$0.32}}/{\blue{$\uparrow$0.0045}}\\
& $\times$4 & 30.66/0.8699 & {\textbf{30.98/0.8776}} & {\blue{$\uparrow$0.31}}/{\blue{$\uparrow$0.0077}}\\
\hline
\multirow{3}{*}{Set14} & $\times$2 & 32.74/0.9103 & {\textbf{33.00/0.9123}} & {\blue{$\uparrow$0.26}}/{\blue{$\uparrow$0.0020}}\\
& $\times$3 & 29.56/0.8262 & {\textbf{29.73/0.8300}} & {\blue{$\uparrow$0.17}}/{\blue{$\uparrow$0.0038}}\\
& $\times$4 & 27.71/0.7583 & {\textbf{27.92/0.7643}} & {\blue{$\uparrow$0.20}}/{\blue{$\uparrow$0.0061}}\\
\hline
\multirow{3}{*}{BSD100} & $\times$2 & 31.55/0.8921 & {\textbf{31.72/0.8941}} & {\blue{$\uparrow$0.18}}/{\blue{$\uparrow$0.0020}}\\
& $\times$3 & 28.55/0.7904 & {\textbf{28.67/0.7938}} & {\blue{$\uparrow$0.12}}/{\blue{$\uparrow$0.0034}}\\
& $\times$4 & 27.02/0.7162 & {\textbf{27.13/0.7204}} & {\blue{$\uparrow$0.11}}/{\blue{$\uparrow$0.0042}}\\
\hline
\multirow{3}{*}{Urban100} & $\times$2 & 29.77/0.9010 & {\textbf{30.24/0.9076}} & {\blue{$\uparrow$0.47}}/{\blue{$\uparrow$0.0066}}\\
& $\times$3 & 26.43/0.8071 & {\textbf{26.70/0.8158}} & {\blue{$\uparrow$0.27}}/{\blue{$\uparrow$0.0088}}\\
& $\times$4 & 24.61/0.7279 & {\textbf{24.83/0.7393}} & {\blue{$\uparrow$0.23}}/{\blue{$\uparrow$0.0115}}\\
\hline
\end{tabular}
\end{center}
\label{table:AIL_FSRCNN}
\end{table}

\begin{figure*}[htp]
\setlength{\abovecaptionskip}{0pt}
\begin{center}
\begin{tabu} to 1\textwidth{ccccc}
Ground truth & FSRCNN & FSRCNN+AIL\\
\includegraphics[height=1.2in,width=1.7in,angle=0]{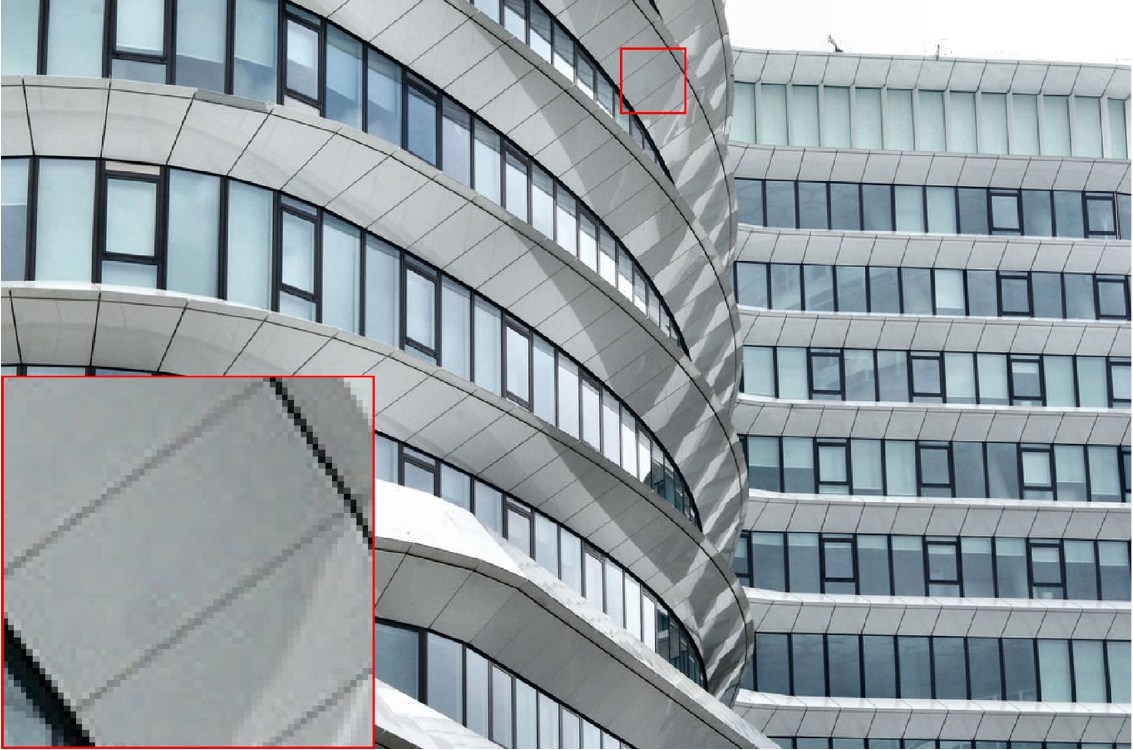}
\hspace{-0.45cm} &
\includegraphics[height=1.2in,width=1.7in,angle=0]{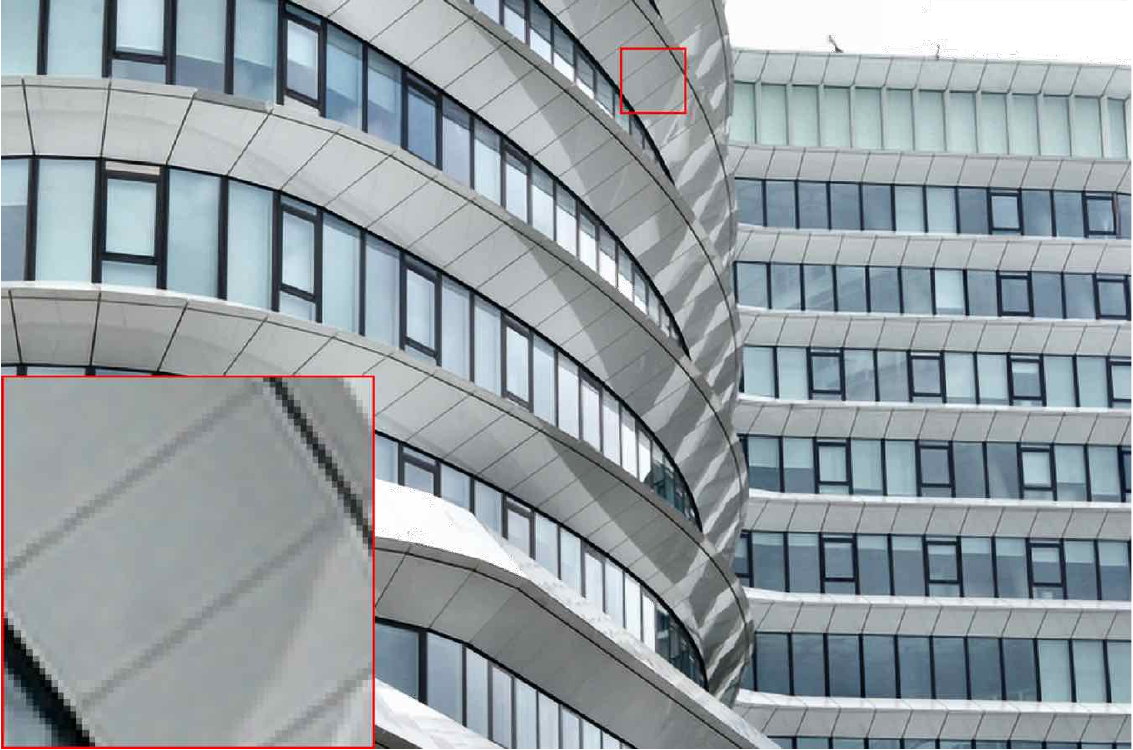} \hspace{-0.45cm} &
\includegraphics[height=1.2in,width=1.7in,angle=0]{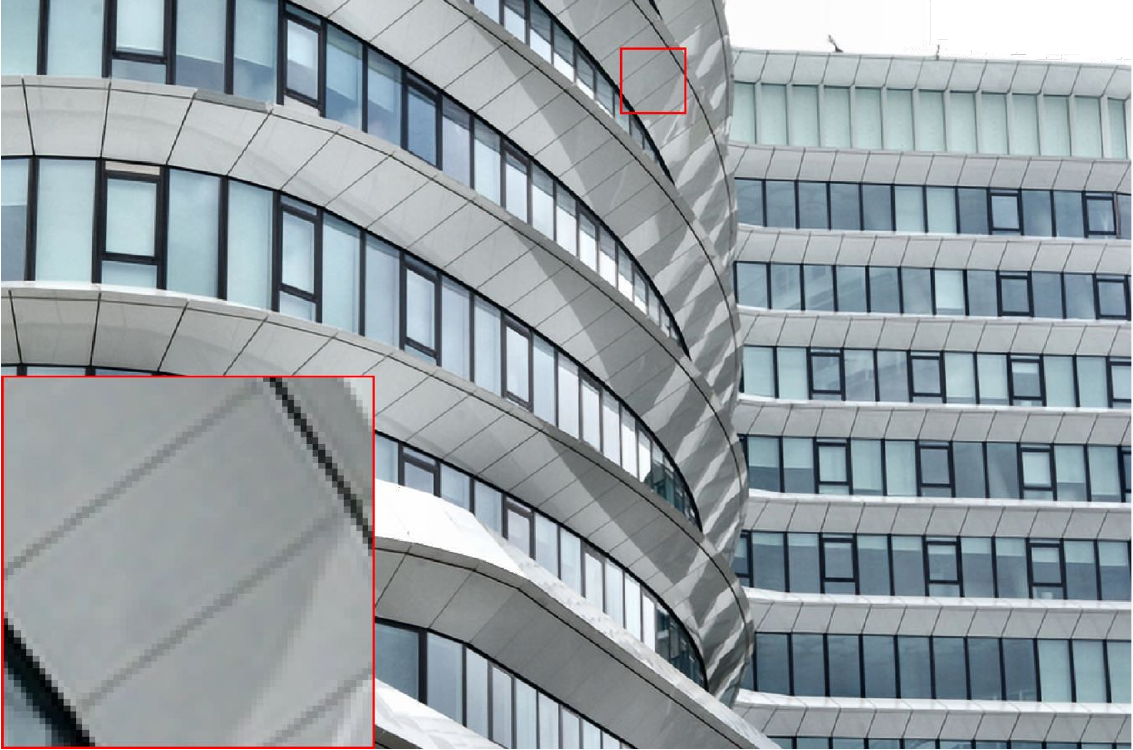}
\vspace{-0.07cm}
\\
{\sz{(PSNR/SSIM)}} & {\sz{(31.93/0.9522)}} & {\sz{(32.89/0.9615)}}%
\\
\includegraphics[height=2.1in,width=1.7in,angle=0]{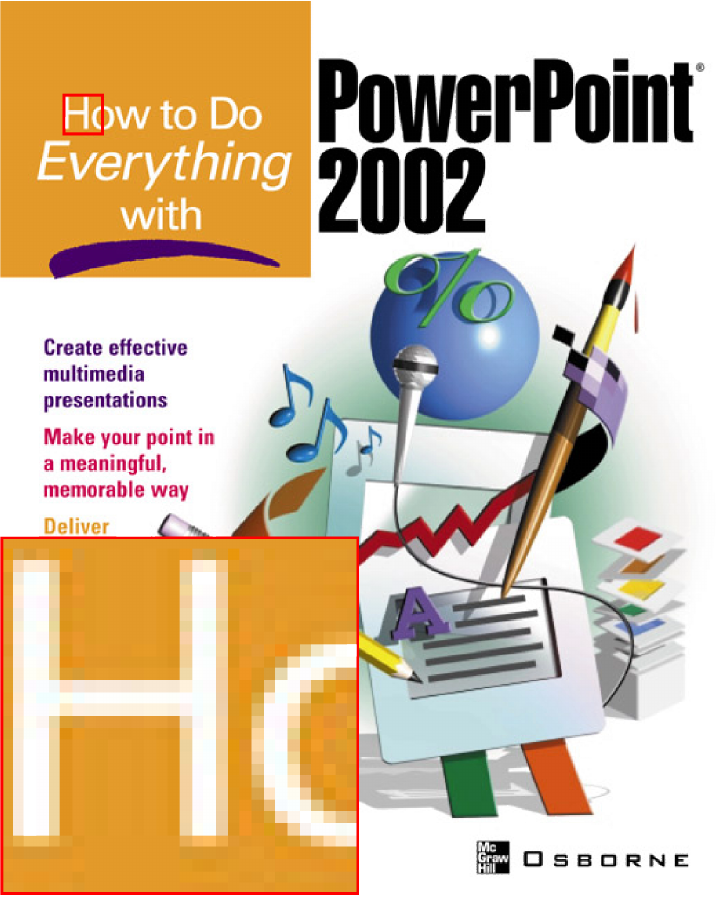}\hspace{-0.45cm} &
\includegraphics[height=2.1in,width=1.7in,angle=0]{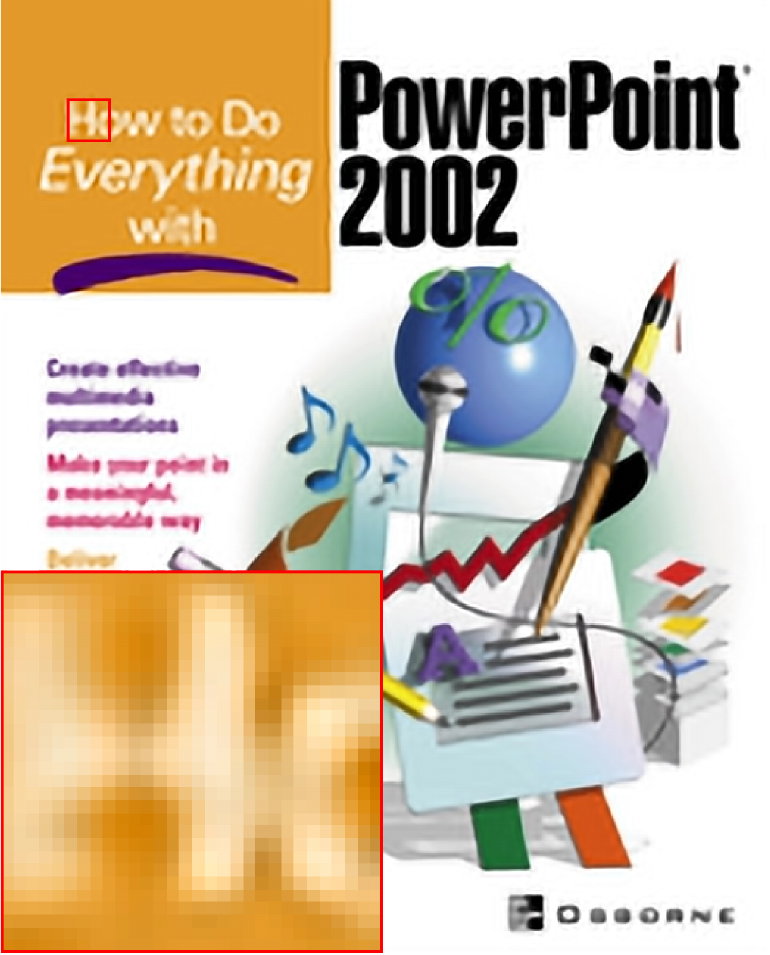} \hspace{-0.45cm} &
\includegraphics[height=2.1in,width=1.7in,angle=0]{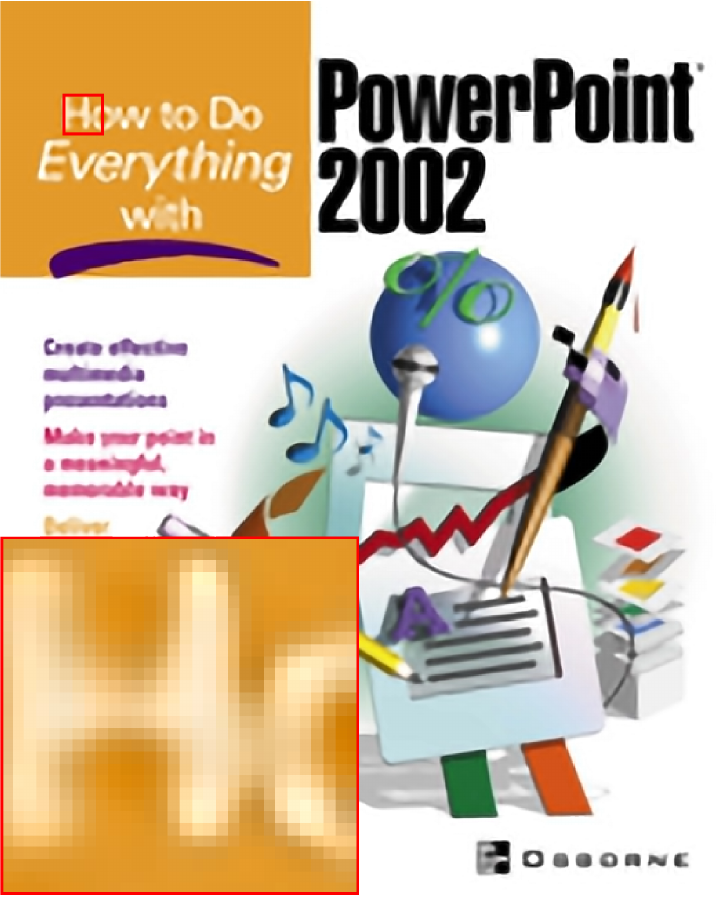}
\vspace{-0.07cm}
\\
{\sz{(PSNR/SSIM)}} & {\sz{(24.82/0.8957)}} & {\sz{(25.24/0.9175)}} %
\\
\end{tabu}
\end{center}
\caption{Visual super-resolution results of {\texttt{FSRCNN}} and {\texttt{FSRCNN+AIL}}. First row: the super-resolution results for image '52' from Urban100 dataset when scaling factor is $2$. Second row: the super-resolution results for image 'ppt3' from Set14 dataset when scaling factor is $4$.}
\label{fig:FSRCNN}
\end{figure*}

\subsection{Enhancing lightweight network with other architectures}
Due to not involving modifying the architecture of network, the proposed learning scheme can be directly applied to any lightweight DCNN based SISR methods. To demonstrate this point, we further evaluate the proposed learning scheme on another seminal network for SISR, {\texttt{DRRN}}~\citep{tai2017image}. Specifically, we implement a lightweight network with $25$ feature maps (i.e., $\rho=0.8$) in each convolution layer. The corresponding parameters as well computational complexity can be found in Table~\ref{table:AIL_Complexity}. Then, we train this lightweight network with the traditional learning scheme in Eq.~\eqref{eq:eq1} and the proposed adaptive importance learning scheme as Algorithm~\ref{alg:bpl}. In the proposed learning scheme, the pre-trained {\texttt{DRRN}} is utilized to initialized the importance. The obtained two networks are termed {\texttt{DRRN-f25}} and {\texttt{DRRN-f25+AIL}}, respectively. Similar as that in Section~\ref{subsec:VDSR}, the quantitative and visual results of these two networks are provided in Table~\ref{table:AIL_DRRNf25} and Figure~\ref{fig:DRRN-f25}. We can find that the propose learning scheme can obviously improve the performance of the corresponding lightweight network. For example, when the scaling factor is $2$ on the Urban100 dataset, {\texttt{DRRN-f25+AIL}} outperforms {\texttt{DRRN-f25}} in PSNR and SSIm by $0.51$db and $0.0059$, respectively. In Figure~\ref{fig:DRRN-f25}, {\texttt{DRRN-f25+ILT}} and {\texttt{DRR\-N-f25+AIL}} produces more sharp and clear results than that of {\texttt{DRRN-f25}}.

In previous experiments, we customize all lightweight networks by reducing the amount of filters in each convolution layer from a given teacher network. As mentioned in Section~\ref{sec:light}, there are some other choices~\cite{dong2016accelerating,shi2016real} that focus on investigating new architecture. To further demonstrate the effectiveness of the proposed learning scheme on those network with specialized lightweight architectures, we employ it to train the {\texttt{FSRCNN}}~\citep{dong2016accelerating} which exhibits a hourglass-shape structure. Similar as previous experiments, given the lightweight network, we train it separately with the traditional learning scheme as Eq.~\eqref{eq:eq1} and the proposed adaptive importance learning in Algorithm~\ref{alg:bpl}. The learned networks are termed {\texttt{FSRCNN}} and {\texttt{FSRCNN+AIL}}, respectively. For training {\texttt{FSRCNN+AIL}}, we adopt the pre-trained {\texttt{VDSR}} as the teacher network for importance initialization. The numerical results of these two networks on four test datasets are reported in Table~\ref{table:AIL_FSRCNN}. Since we adopt a larger training dataset, the performance of the {\texttt{FSRCNN}} is slightly higher that in~\cite{dong2016accelerating}. In Table~\ref{table:AIL_FSRCNN}, we can find that {\texttt{FSRCNN+AIL}} surpasses {\texttt{FSRCNN}} clearly in all cases. For example, when the scaling factor is $2$ on both Set5 and Urban100 datasets, {\texttt{FSRCNN+AIL}} improves the PSRN of {\texttt{FSRCNN}} at least by $0.4$db. More visual evidence can be found in Figure~\ref{fig:FSRCNN}. 

Therefore, we can conclude that the proposed adaptive importance learning scheme is a general SISR learning scheme and can be applied to any given lightweight network architectures for performance enhancement.

\section{Conclusion}
In this study, we present an easy-to-complex learning strategy, termed adaptive importance learning scheme, to enhance the fitting capacity of a given lightweight SISR network architecture. The propose learning scheme integrates network training and pixel-wise importance learning into a joint optimization framework, which can be well addressed in an alternative way. Through dynamically updating the importance of image pixels, the network starts with learning to reconstruct easy pixel at the beginning, and then are exposed to more and more complex pixels for training. By doing this, the fitting capacity can be gradually enhanced and ultimately maximized when the learning scheme converges. In addition, the learning scheme enables seamlessly assimilating the knowledge from a more powerful teacher network to initialize the importance of image pixels, which leads to better initial capacity of the network as well as the ultimate super-resolution performance. Extensive experimental results on four benchmark datasets demonstrate that the proposed learning strategy is able to enhance the super-resolution performance of a given lightweight network with different architectures or scales.

It is noteworthy that the proposed adaptive importance learning is general learning paradigm for enhancing the light\-weight regression networks. In the future, we will further exploit its potential benefits in other regression problems, e.g., image denoising, image deblurring and image inpainting etc.

{
\bibliographystyle{spbasic}\small
\bibliography{refs}
}

\end{document}